\documentclass{article}
\usepackage{iclr2027_conference,times}

\usepackage{amsmath,amsfonts,bm}

\def\eqref#1{equation~\ref{#1}}

\def\1{\bm{1}}

\DeclareMathAlphabet{\mathsfit}{\encodingdefault}{\sfdefault}{m}{sl}
\SetMathAlphabet{\mathsfit}{bold}{\encodingdefault}{\sfdefault}{bx}{n}

\newcommand{\E}{\mathbb{E}}

\newcommand{\R}{\mathbb{R}}

\usepackage{amsthm}
\usepackage{amssymb} 
\newtheorem{theorem}{Theorem}

\theoremstyle{definition}

\theoremstyle{remark}

\usepackage{graphicx}
\usepackage{subcaption}
\usepackage{fix-cm}
\usepackage{booktabs}
\usepackage{placeins}
\usepackage{tikz}
\usetikzlibrary{arrows.meta,calc,positioning}
\usepackage{hyperref}

\ifdefined\pdfrunninglinkoff
  \makeatletter
  \let\pdfannot@link@off@@\pdfrunninglinkoff
  \let\pdfannot@link@on@@\pdfrunninglinkon
  \makeatother
\fi
\usepackage{url}

\title{Quantile Head\\for Vision-Language-Action Models}
\author{%
  \hspace*{-\tabcolsep}
  \normalfont
  \begin{tabular*}{\textwidth}[t]{@{\extracolsep{\fill}}ccc@{}}
    \textbf{Xuan Wang} & \textbf{Yinan Wu} & \textbf{Haoran Duan} \\
    Department of Automation & Department of Automation & Department of Automation \\
    Tsinghua University & Tsinghua University & Tsinghua University \\
    \texttt{xwangrs@gmail.com} & \texttt{yinanwu@ieee.org} & \texttt{haoran.duan@ieee.org} \\[16pt]
    \multicolumn{3}{c}{\textbf{Jungong Han}\textsuperscript{*}} \\
    \multicolumn{3}{c}{Department of Automation} \\
    \multicolumn{3}{c}{Tsinghua University} \\
    \multicolumn{3}{c}{\texttt{jghan@tsinghua.edu.cn}}
  \end{tabular*}%
  \hspace*{-\tabcolsep}%
}
\iclrfinalcopy 

\expandafter\def\csname arxivOverlay@two_demo_data@left\endcsname{%
\draw[white,line width=1.7pt] (0.5000000,0.1162180) -- (0.5000000,0.8451803);
\draw[black!70,dash pattern=on 2pt off 1.5pt,line width=0.65pt] (0.5000000,0.1162180) -- (0.5000000,0.8451803);
\draw[white,line width=2.3pt] (0.5000000,0.4131433) -- (0.2500606,0.4131433);
\draw[demoLeft,line width=1.15pt,-{Stealth[length=3.5pt,width=3pt]}] (0.5000000,0.4131433) -- (0.2500606,0.4131433);
\draw[white,line width=2.2pt] (0.2500606,0.4131433) circle[radius=3.1pt];
\draw[demoLeft,line width=1.2pt] (0.2500606,0.4131433) circle[radius=3.1pt];
}
\expandafter\def\csname arxivOverlay@two_demo_data@right\endcsname{%
\draw[white,line width=1.7pt] (0.5000000,0.1174990) -- (0.5000000,0.8440282);
\draw[black!70,dash pattern=on 2pt off 1.5pt,line width=0.65pt] (0.5000000,0.1174990) -- (0.5000000,0.8440282);
\draw[white,line width=2.3pt] (0.5000000,0.4108167) -- (0.7028001,0.4108167);
\draw[demoRight,line width=1.15pt,-{Stealth[length=3.5pt,width=3pt]}] (0.5000000,0.4108167) -- (0.7028001,0.4108167);
\draw[white,line width=2.2pt] (0.7028001,0.4108167) circle[radius=3.1pt];
\draw[demoRight,line width=1.2pt] (0.7028001,0.4108167) circle[radius=3.1pt];
}
\expandafter\def\csname arxivOverlay@two_demo_obstacle_data@left\endcsname{%
\draw[white,line width=1.7pt] (0.5000000,0.1701071) -- (0.5000000,0.8219161);
\draw[black!70,dash pattern=on 2pt off 1.5pt,line width=0.65pt] (0.5000000,0.1701071) -- (0.5000000,0.8219161);
\draw[white,line width=1.8pt] (0.4397288,0.4357107) -- (0.5602712,0.4357107) -- (0.5602712,0.5241085) -- (0.4397288,0.5241085) -- cycle;
\draw[red!75!black,line width=0.8pt,dash pattern=on 2pt off 1pt] (0.4397288,0.4357107) -- (0.5602712,0.4357107) -- (0.5602712,0.5241085) -- (0.4397288,0.5241085) -- cycle;
\draw[white,line width=2.3pt] (0.5000000,0.4282343) -- (0.1833531,0.4282343);
\draw[demoLeft,line width=1.15pt,-{Stealth[length=3.5pt,width=3pt]}] (0.5000000,0.4282343) -- (0.1833531,0.4282343);
\draw[white,line width=2.2pt] (0.1833531,0.4282343) circle[radius=3.1pt];
\draw[demoLeft,line width=1.2pt] (0.1833531,0.4282343) circle[radius=3.1pt];
}
\expandafter\def\csname arxivOverlay@two_demo_obstacle_data@right\endcsname{%
\draw[white,line width=1.7pt] (0.5000000,0.1704170) -- (0.5000000,0.8216136);
\draw[black!70,dash pattern=on 2pt off 1.5pt,line width=0.65pt] (0.5000000,0.1704170) -- (0.5000000,0.8216136);
\draw[white,line width=1.8pt] (0.4397288,0.4357107) -- (0.5602712,0.4357107) -- (0.5602712,0.5241085) -- (0.4397288,0.5241085) -- cycle;
\draw[red!75!black,line width=0.8pt,dash pattern=on 2pt off 1pt] (0.4397288,0.4357107) -- (0.5602712,0.4357107) -- (0.5602712,0.5241085) -- (0.4397288,0.5241085) -- cycle;
\draw[white,line width=2.3pt] (0.5000000,0.4345127) -- (0.7699618,0.4345127);
\draw[demoRight,line width=1.15pt,-{Stealth[length=3.5pt,width=3pt]}] (0.5000000,0.4345127) -- (0.7699618,0.4345127);
\draw[white,line width=2.2pt] (0.7699618,0.4345127) circle[radius=3.1pt];
\draw[demoRight,line width=1.2pt] (0.7699618,0.4345127) circle[radius=3.1pt];
}

\begin{document}
\maketitle
\lhead{}
\renewcommand{\headrulewidth}{0pt}

\begin{abstract}
\looseness=-1
Vision-Language-Action (VLA) models integrate pretrained Vision-Language
Models (VLMs) with action heads for robot control. Common action heads
have distinct limitations: point regression provides only a point estimate
of the action distribution, while standard flow-matching samplers require
costly iterative sampling. To address these limitations, we unify
regression and flow matching under a shared objective and extend it to
derive a quantile objective. This quantile objective guides the design of
our Quantile Head, which predicts a median and positive gaps to form
ordered marginal action quantiles in one forward pass. These quantiles
support multiple sampling strategies without retraining and are jointly
supervised to train the default median policy. Our local analysis of this
joint supervision shows that, with calibrated nearby quantiles, fixed
gaps, and matched correction speed, direct median updates have lower
variance than under median-only supervision. Experiments show that this
jointly supervised median policy achieves the highest average success
rates among the compared methods on LIBERO, LIBERO-Plus, LIBERO-Pro, and
two real-robot tasks, together with the shortest mean episode time among
matched LIBERO baselines; code is available at
\url{https://github.com/xwangrs/Quantile-Head-for-VLA}.
\end{abstract}
\section{Introduction}
\label{sec:introduction}
Vision--Language--Action (VLA) models offer a promising approach to robot
control by mapping visual observations and language instructions to
actions. To learn this mapping, models such as RT-2
\citep{brohan2023rt}, OpenVLA \citep{kim2024openvla}, and $\pi_{0.5}$
\citep{pmlr-v305-black25a} adapt pretrained Vision-Language Models (VLMs) using robot
demonstrations. From these demonstrations, an action head learns to
translate multimodal representations into executable controls, making its
design central to how actions are represented and predicted.

Two common action head designs, regression and flow matching, offer
complementary capabilities.
Regression heads predict actions directly in one forward pass, providing
a simple interface for deterministic control \citep{kim2024openvla}.
Their point outputs, however, do not explicitly represent conditional
action distributions for sampling. Flow-matching heads model these
distributions by learning a velocity field that transforms noise into
action samples, as in $\pi_0$ \citep{black2024pi_0}.
Their standard sampling procedures typically require iterative
integration and repeated action expert evaluations, adding inference
cost \citep{kim2025fine}.

\begin{figure}[!t]
  \centering
  \begin{minipage}{0.8\linewidth}
\begingroup%
\definecolor{ihInk}{HTML}{303748}%
\definecolor{ihMuted}{HTML}{8795A8}%
\definecolor{ihPaleBlue}{HTML}{DCEAF7}%
\definecolor{ihBlue}{HTML}{5F83AB}%
\definecolor{ihPink}{HTML}{8F3038}%
\definecolor{ihPinkMid}{HTML}{B66B72}%
\definecolor{ihPinkPale}{HTML}{F2E1E3}%
\tikzset{ih picture/.style={x=1cm,y=1cm,line cap=round,line join=round,
  every node/.style={font=\rmfamily\bfseries\fontsize{9.5}{11}\selectfont,
    text=ihInk,inner sep=1pt,align=center}},
  ih small/.style={font=\rmfamily\bfseries\fontsize{8.5}{10}\selectfont},
  ih arrow/.style={-{Latex[length=2mm,width=1.55mm]},line width=1.35pt,draw=ihMuted},
  ih axis/.style={line width=1.1pt,draw=ihMuted!75},
  ih guide/.style={line width=.8pt,draw=ihPinkMid,dash pattern=on 2pt off 2.4pt}}%
\newcommand{\ihHead}[2]{%
  \path[draw=#1,fill=#2,line width=1.6pt]
    (1.62,3.36)
    .. controls (1.37,3.36) and (1.33,3.51) .. (1.34,3.70)
    .. controls (1.34,3.92) and (1.43,4.03) .. (1.65,4.03)
    -- (3.78,4.03)
    .. controls (4.04,4.03) and (4.09,3.91) .. (4.06,3.68)
    .. controls (4.06,3.46) and (3.98,3.36) .. (3.76,3.36)
    -- cycle;
  \node[text=#1] at (2.70,3.69) {Action head};
  \draw[ih arrow,draw=#1] (2.70,4.59) -- (2.70,4.10);
  \draw[ih arrow,draw=#1] (2.70,3.25) -- (2.70,2.78);
}%
%
\captionsetup[subfigure]{font={scriptsize,bf,color=ihInk},
  justification=centering,singlelinecheck=false,skip=2pt,hypcap=true}%
\begin{subfigure}[t]{.32142857\linewidth}
  \centering
  \resizebox{\linewidth}{!}{
\begin{tikzpicture}[ih picture]
  \path[use as bounding box] (0,0) rectangle (5.40,4.70);
  \ihHead{ihBlue}{ihPaleBlue}
  \draw[ih axis] (.65,1.28) -- (4.78,1.28);
  \draw[draw=ihBlue!65,line width=1pt,dash pattern=on 2pt off 2.5pt]
    (2.70,1.28) -- (2.70,1.99);
  \fill[ihPaleBlue] (2.70,2.00) circle (.23);
  \fill[ihBlue] (2.70,2.00) circle (.11);
  \node at (2.70,.49) {One pass, one value};
\end{tikzpicture}%
}%
  \caption{Regression}\label{fig:intro-regression}
\end{subfigure}\hfill%
\begin{subfigure}[t]{.32142857\linewidth}
  \centering
  \resizebox{\linewidth}{!}{
\begin{tikzpicture}[ih picture]
  \path[use as bounding box] (0,0) rectangle (5.40,4.70);
  \ihHead{ihBlue}{ihPaleBlue}
  \draw[draw=ihBlue,line width=1.0pt,
    -{Latex[length=1.45mm,width=1.05mm]}]
    (3.94,3.23) arc[start angle=-90,end angle=90,radius=.46];
  \path[fill=ihPaleBlue,draw=ihBlue,line width=1.25pt]
    (.27,1.28)
    .. controls (.59,1.30) and (.61,2.20) .. (.93,2.20)
    .. controls (1.24,2.20) and (1.26,1.30) .. (1.59,1.28)
    -- cycle;
  \path[fill=ihPinkPale,draw=ihPinkMid,line width=1.25pt]
    (2.08,1.28)
    .. controls (2.52,1.28) and (2.43,2.20) .. (2.70,2.20)
    .. controls (2.97,2.20) and (2.90,1.28) .. (3.33,1.28)
    -- cycle;
  \path[fill=ihPinkMid!30,draw=ihPink,line width=1.3pt]
    (3.88,1.28)
    .. controls (4.38,1.28) and (4.26,2.20) .. (4.47,2.20)
    .. controls (4.69,2.20) and (4.55,1.28) .. (5.10,1.28)
    -- cycle;
  \foreach \left/\right in {.27/1.59,2.08/3.33,3.88/5.10}{
    \draw[ih axis] (\left,1.28) -- (\right,1.28);
  }
  \draw[ih arrow] (1.54,1.89) .. controls (1.73,1.98) and (1.89,1.98) .. (2.04,1.89);
  \draw[ih arrow] (3.29,1.89) .. controls (3.48,1.98) and (3.64,1.98) .. (3.81,1.89);
  \filldraw[fill=ihPink,draw=white,line width=1pt] (4.65,1.28) circle (.075);
  \node[ih small,text=ihMuted] at (.93,.93) {Noise};
  \node[ih small,text=ihPink] at (4.49,.93) {Sample};
  \node at (2.70,.49) {Iterative sampling};
\end{tikzpicture}%
}%
  \caption{Flow matching}\label{fig:intro-flow}
\end{subfigure}\hfill%
\begin{subfigure}[t]{.32142857\linewidth}
  \centering
  \captionsetup{font={scriptsize,bf,color=ihPink}}%
  \resizebox{\linewidth}{!}{
\begin{tikzpicture}[ih picture]
  \path[use as bounding box] (0,0) rectangle (5.40,4.70);
  \ihHead{ihPink}{ihPinkPale}
  \draw[ih axis] (.72,1.28) -- (4.84,1.28);
  \draw[ih axis] (.72,1.28) -- (.72,2.53);
  \path[fill=ihPinkPale!65]
    (.76,1.28) -- plot[domain=-3:3,samples=81,smooth]
    ({.76+.66*(\x+3)},{1.28+1.19/(1+exp(-1.4*\x))})
    -- (4.72,1.28) -- cycle;
  \draw[draw=ihPink,line width=1.8pt]
    plot[domain=-3:3,samples=81,smooth]
    ({.76+.66*(\x+3)},{1.28+1.19/(1+exp(-1.4*\x))});
  \foreach \prob in {.1,.5,.9}{
    \pgfmathsetmacro{\qx}{.76+.66*(ln(\prob/(1-\prob))/1.4+3)}
    \pgfmathsetmacro{\qy}{1.28+1.19*\prob}
    \draw[ih guide] (\qx,1.28) -- (\qx,\qy);
    \filldraw[fill=ihPinkMid,draw=white,line width=.65pt] (\qx,\qy) circle (.063);
    \fill[ihPinkMid] (\qx,1.28) circle (.06);
  }
  \filldraw[fill=ihPink,draw=white,line width=.65pt] (2.74,1.875) circle (.083);
  \fill[ihPink] (2.74,1.28) circle (.078);
  \node[text=ihPink] at (2.70,.49) {One pass, ordered quantiles};
\end{tikzpicture}%
}%
  \caption{Quantile head (ours)}\label{fig:intro-quantile}
\end{subfigure}%
\endgroup%
  \end{minipage}
  \caption{\textbf{Action head comparison.} Regression predicts a point,
  iterative flow matching reuses the same head at updated action states
  and integration times, and our head predicts ordered marginal quantiles
  in one pass for direct median control or optional sampling.}
  \label{fig:introduction-heads}
\end{figure}
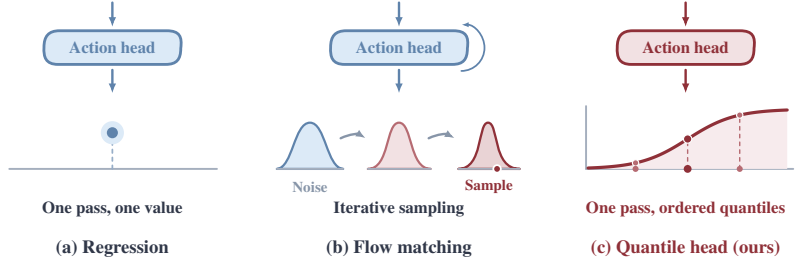

These complementary capabilities raise a natural question:
\emph{Can we combine the single-pass prediction of regression with the
distributional modeling of flow matching?}
To this end, we can extend a single action prediction to a set of quantiles,
where the median represents the action center and lower and upper quantiles
describe its variability. Predicting these quantiles together allows our
Quantile Head (Figure~\ref{fig:introduction-heads}) to directly characterize
the conditional distribution of each action coordinate. To formalize this
intuition, we formulate a distributional learning objective and derive its
equivalent quantile regression form, leading to Theorem~\ref{thm:quantile-objective}.

Building on this quantile objective, we further analyze how auxiliary
quantiles can improve the training signal for the median.
Specifically, we study the noise in direct median updates under joint
quantile supervision, leading to Theorem~\ref{thm:quantile-generalization}.
Under the stated conditions, calibrated quantiles near the median reduce
update variance relative to median-only supervision, with gaps held fixed
and local correction speed matched. This result motivates the joint
supervision of the median and auxiliary quantiles in our Quantile Head.

We integrate Quantile Head with a pretrained VLM for action learning.
To limit the number of trainable parameters, we freeze the VLM
and introduce learnable prompts that extract task context from its fixed
features. Conditioned on this context, the quantile head jointly predicts
a median and gaps between adjacent quantiles on either side for each
action coordinate. The gaps are constrained to be positive and
cumulatively subtracted from or added to the median, yielding ordered
lower and upper quantiles. The quantiles for the entire action chunk are
generated together in a single forward pass, capturing the center and
variability of each action coordinate. To mitigate overfitting to
demonstrations, we use a masked pinball loss that randomly withholds
supervision for selected action coordinates, sharing the mask across all
quantiles of each coordinate.

\begin{samepage}
Our contributions cover theory, quantile action head design, and
experimental evaluation:
\begin{enumerate}
  \item \textbf{Theory.} Starting from a unified supervision framework
  for regression and flow matching, we derive a quantile objective for
  action distribution modeling and show that joint quantile supervision
  reduces noise in direct median updates under stated conditions.

  \item \textbf{Quantile Head.} We propose a quantile action head
  that predicts ordered quantiles in a single forward pass. A masked
  pinball loss jointly supervises these outputs to model each action
  coordinate's conditional distribution.

  \item \textbf{Experimental Evaluation.} We achieve state of the art
  performance on LIBERO and state of the art generalization on LIBERO-Plus
  and LIBERO-Pro, together with the best performance among the compared
  methods on two physical robot tasks.
\end{enumerate}
\end{samepage}
\section{Related Work}
\label{sec:related-work}

\textbf{Regression-style action heads.}\quad
Regression heads map observation-conditioned features directly to
continuous actions using losses such as $L_1$ or $L_2$. For an
unrestricted predictor, squared loss targets the conditional mean,
while absolute loss targets the conditional median
\citep{koenker1978regression}. OpenVLA-OFT combines continuous $L_1$
regression with parallel action-chunk decoding, demonstrating the
effectiveness of direct action prediction in VLA models
\citep{kim2025fine}. Such heads produce actions in a single forward
pass and provide a simple interface for control. However, a point
prediction for each action coordinate does not explicitly represent
its conditional distribution, motivating supervision beyond a single
central estimate.

\textbf{Flow-matching-style action heads.}\quad
Diffusion and flow-matching heads transform noise into samples from
conditional action distributions \citep{zhang2026cofactvla}. Diffusion Policy learns action
sequences through denoising \citep{chi2025diffusion}, while
$\pi_0$ \citep{black2024pi_0} and $\pi_{0.5}$
\citep{pmlr-v305-black25a} use flow-matching velocity supervision to
connect noise and continuous action chunks. Their standard samplers use
successive denoising or numerical integration steps, requiring repeated
model evaluations. Recent flow-based VLAs also support one-step generation.
SnapFlow uses progressive self-distillation of a pretrained flow policy,
adding model evaluations to construct training targets
\citep{luan2026snapflow}. \emph{Let It Be Simple} \citep{chen2026simple}
uses high-noise training to enable one-step action generation without
distillation. Its LIBERO-Long ablation shows that the success-rate advantage
over a uniformly trained ten-step baseline diminishes as the action
horizon increases.

\textbf{Quantile-based learning.}\quad
\citet{richter2019learning} learn continuous-action policies through
advantage-weighted quantile regression.
Learning across quantile levels can improve estimation of a target
quantile under suitable conditions \citep{narayan2024expected}.
In probabilistic forecasting, OrderFusion constructs ordered outputs
from a median anchor and nonnegative gaps \citep{yu2026orderfusion}.
Within VLA systems, ReconVLA uses conformally calibrated action-error
quantiles to rank candidates from a pretrained policy
\citep{chen2026reconvla}. We jointly supervise ordered action quantiles
for VLA control, with median decoding by default and optional quantile
sampling.

\textbf{Novelty.}\quad
1. Quantile Head extends point regression to jointly supervised
conditional action quantiles, combining single-pass prediction with an
explicit representation of action marginals. It predicts action values
directly rather than learning a flow-matching velocity field, and the
same outputs support median control and optional sampling.
2. Our focus is how auxiliary action quantiles contribute to learning
the median used for VLA control. A local analysis identifies conditions
under which this joint supervision reduces noise in direct median
updates, while matched ablations demonstrate improved closed-loop
performance even when inference uses only the median.
\section{Method}
\label{sec:method}

\begin{figure}[!t]
  \centering
\includegraphics[width=0.8\linewidth]{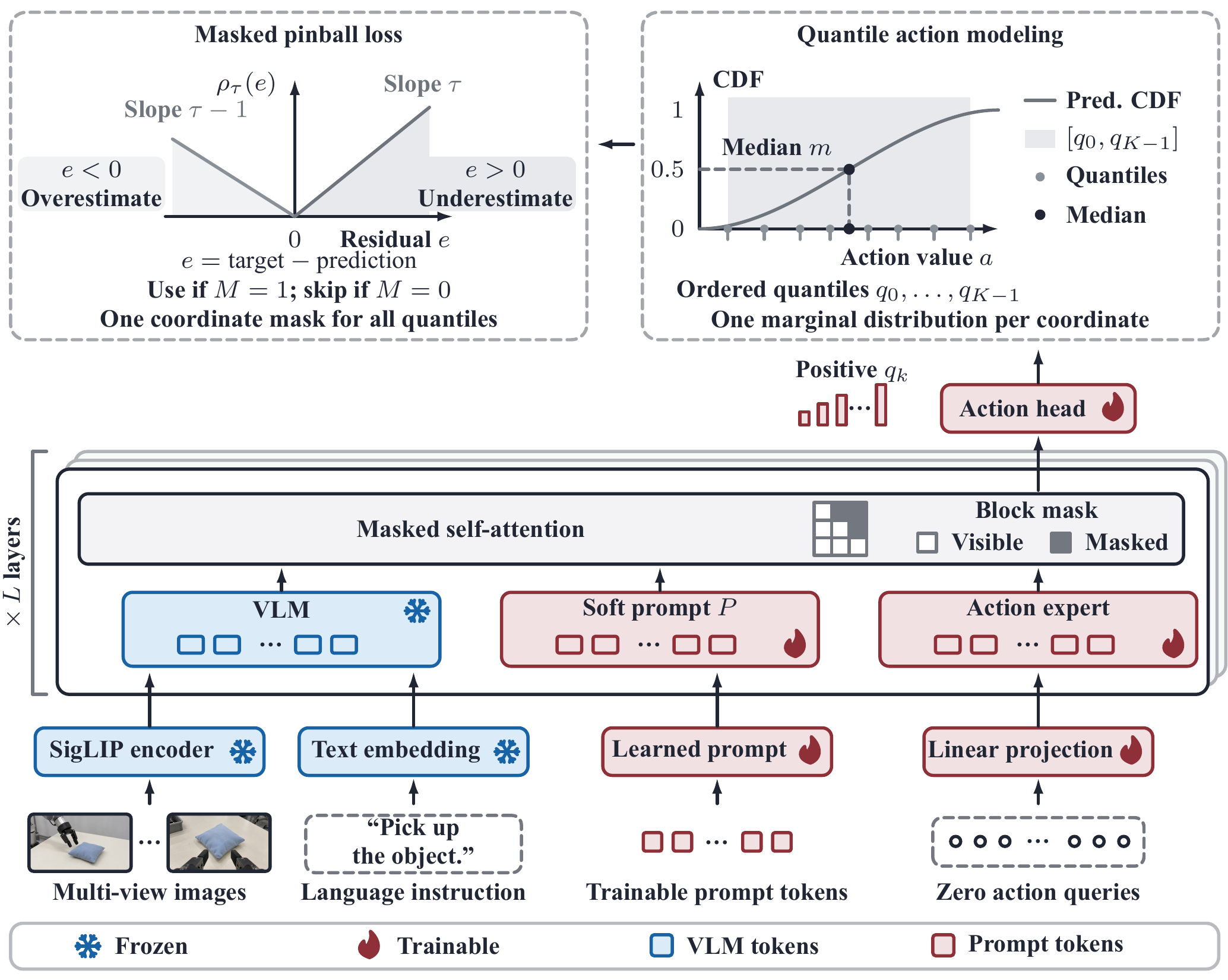}\space\space
  \caption{\textbf{Architecture of the Quantile Head.}
  Our model combines a frozen VLM, learnable prompts, and a Quantile Head. The VLM encodes images and
  instructions. Learnable prompts precede action tokens; masked
  self-attention lets prompts read VLM features and action tokens read
  both VLM and prompt features. The action head maps the resulting
  action features to a median and positive gaps, then accumulates these
  gaps around the median to form ordered marginal quantiles. We supervise these
  quantiles with masked pinball loss. Inference uses the median by default
  or samples from the predicted marginal quantiles.}
  \label{fig:pi05-model}
\end{figure}

\subsection{Theoretical Analysis}
\label{sec:theoretical-analysis}

\textbf{A unified objective for action heads.}\quad
Let $X$ denote visual observations, language instructions, and the robot's
proprioceptive state, and let $A\in\mathbb{R}^{H\times D}$ denote a
demonstrated action chunk with $H$ future steps and $D$ action dimensions.
For each coordinate $Y=A_{h,d}$, $L_2$ regression predicts $Y$ from $X$.
Conditional flow matching (CFM) predicts a velocity from $X$, an
interpolated action, and time, then integrates the velocity field to
generate actions. We unify their squared losses on action and velocity
targets through a reference path. Draw $B\sim\nu(\cdot\mid X)$ with $B\perp Y\mid X$ and an
independent time $\alpha\sim\pi$ on $[0,1]$, and set
$A_\alpha=(1-\alpha)B+\alpha Y$. The shared objective is
\begin{equation}
\mathcal L_{\nu,\pi}(\theta)
=\E\left|f_\theta(X,A_\alpha,\alpha)-(Y-B)\right|^2.
\label{eq:quantile-mother-objective}
\end{equation}
Setting $B=\alpha=0$ recovers squared action regression; a Gaussian
reference and sampled times give CFM velocity supervision
\citep{lipman2023flow}. Appendix~\ref{sec:proof-quantile-objective} derives both cases.

\textbf{Deriving quantile supervision.}\quad
Under \eqref{eq:quantile-mother-objective}, unrestricted regression learns
the mean action, whereas CFM learns a mean velocity field whose integration
generates actions. Both squared losses have output gradients that grow
with residuals (Appendix~\ref{sec:proof-quantile-objective}). To learn action
distributions directly, we extend squared supervision to the threshold
labels $Z_z=\mathbf1\{Y\le z\}$. Their conditional means
$F_X(z)=\Pr(Y\le z\mid X)$, over all $z\in\R$, determine the action
distribution. The following result connects predicting these probabilities
to quantile supervision.

\begin{samepage}
\begin{theorem}[proof in Appendix~\ref{sec:proof-quantile-objective}]
\label{thm:quantile-objective}
Let $q_\theta$ be measurable and nondecreasing in its quantile level,
and let $G_\theta(\cdot\mid X)$ be the conditional cumulative distribution
function (CDF) of $q_\theta(X,U)$,
where $U\sim\operatorname{Unif}(0,1)$ is independent of $(X,Y)$.
If both the true action $Y$
and the predicted action $q_\theta(X,U)$ have finite first absolute
moments, then
\begin{equation}
\underbrace{\E_{X,Y}\int_{\R}
 \bigl(G_\theta(z\mid X)-\mathbf1\{Y\le z\}\bigr)^2\,\mathrm dz}
 _{\mathcal L_{\mathrm{CDF}}(\theta)}
=2\E_{X,Y}\int_0^1
 \rho_\tau\bigl(Y-q_\theta(X,\tau)\bigr)\,\mathrm d\tau.
\label{eq:quantile-integrated-objective}
\end{equation}
Here $\rho_\tau$ is the pinball loss. If the true conditional distribution
is representable, the population-optimal quantile
function equals the true conditional quantile function almost everywhere.
\end{theorem}
\end{samepage}

Theorem~\ref{thm:quantile-objective} shows that integrated pinball loss is
equivalent to CDF supervision up to a factor of two. Conditional action
distributions can therefore be learned through quantile supervision,
without constructing threshold labels. Our head jointly supervises a
finite set of quantiles to represent each action coordinate's conditional
distribution.

\textbf{Local noise reduction in median updates.}\quad
Theorem~\ref{thm:quantile-objective} provides a distributional basis for
quantile supervision. Since our default policy executes the median, we
further examine whether auxiliary quantile supervision can reduce noise
in direct median updates.

\begin{samepage}
\begin{theorem}[Noise reduction in median updates]
\label{thm:quantile-generalization}
Under the assumptions in Appendix~\ref{sec:proof-center-generalization},
consider calibrated quantiles at distinct levels symmetric about and
sufficiently close to $1/2$. With fixed gaps and matched local mean
correction rates, the updates at the true median satisfy
\begin{equation}
\operatorname{Var}(\Delta m_Q)<\operatorname{Var}(\Delta m_M).
\label{eq:center-update-noise}
\end{equation}
Here $\Delta m_Q$ and $\Delta m_M$ are the joint-quantile and median-only
random center updates, respectively.
\end{theorem}
\end{samepage}

Theorem~\ref{thm:quantile-generalization} shows that auxiliary quantiles
can reduce direct median-update variance at matched local correction
rates under the stated conditions. This motivates jointly supervising
the median and auxiliary quantiles in our Quantile Head. We evaluate the
resulting policy empirically with jointly learned gaps and a wider
quantile grid.

\subsection{Quantile Head}
\label{sec:action-features}
\textbf{Architecture overview.}\quad
As shown in Figure~\ref{fig:pi05-model}, our design combines prompt adaptation
of a frozen VLM with an ordered quantile head. We freeze the VLM to preserve
its pretrained representations and limit the parameters learned from
action demonstrations. We insert learnable prompts that read the frozen
VLM features and provide task-specific context to the action expert.
The expert produces action-token features for the quantile head, which
constructs ordered predictions from a median and positive
gaps. Joint pinball supervision follows the distributional interpretation
in Theorem~\ref{thm:quantile-objective}. Theorem~\ref{thm:quantile-generalization}
isolates a local center-update mechanism under calibrated, fixed gaps.
Positive gaps enforce noncrossing marginal outputs for quantile decoding;
the effect of joint supervision on median control is assessed in the
matched experiments.

\raggedbottom
\begin{figure}[!t]
  \centering
  \includegraphics[width=\linewidth]{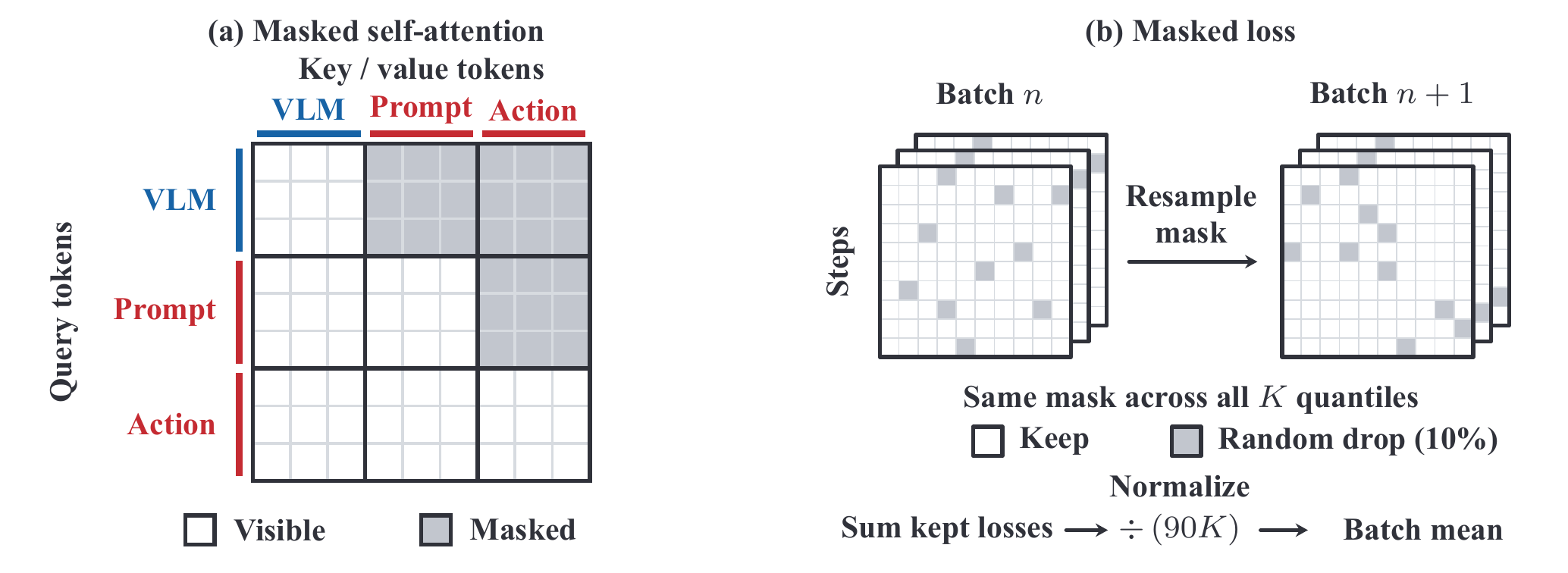}
  \caption{\textbf{Attention and loss masks (schematic).}
  (a) Block self-attention controls information flow among VLM, prompt,
  and action tokens.
  (b) Coordinate masks are independently resampled for each example in
  every training batch and shared across all quantiles.}
  \label{fig:model-details}
\end{figure}

\textbf{Learnable prompt insertion.}\quad
We prepend $N_p$ learnable prompt tokens
$P\in\mathbb{R}^{N_p\times d_e}$ to the action-token sequence at the input
of the action expert, where $d_e$ is the expert's hidden width.
Prompt embeddings are shared across observations, while their hidden
states evolve with action tokens through the expert layers.
At each layer, the block attention mask in
Figure~\ref{fig:model-details}(a) allows prompt queries to attend to VLM
and prompt tokens, while action queries attend to VLM, prompt, and action
tokens. VLM queries attend only to VLM tokens, and attention is
bidirectional within each group. The resulting prompt features provide
observation-conditioned context for action prediction. We use $N_p=16$
and jointly train $P$, the action expert, and output heads with the
VLM frozen.

\textbf{Ordered quantile head.}\quad
We feed the final $H$ action-token features from the action expert into
two linear projections shared across future steps. For each action
coordinate, one projection predicts a median $m$, and the other predicts
raw gaps $r_i^-,r_i^+$ on its two sides. Let $K=2c+1$ and index the
quantiles from $0$ to $K-1$, suppressing the step and coordinate indices.
We map the raw gaps to positive values and accumulate them around $m$:
\begingroup
\predisplaypenalty=0
\begin{equation}
  g_i^\pm = \frac{\delta_0}{\log 2}\,
    \operatorname{softplus}(r_i^\pm),\qquad q_c=m,
  \label{eq:quantile-gaps}
\end{equation}
\begin{equation}
  q_{c-j} = m-\sum_{i=1}^{j}g_i^-,\qquad
  q_{c+j}=m+\sum_{i=1}^{j}g_i^+,\quad j=1,\ldots,c.
  \label{eq:ordered-quantiles}
\end{equation}
\endgroup
\flushbottom
Positive gaps prevent quantile crossing; separately learned left and
right gaps allow asymmetric spacing. We use $K=21$ levels
$\tau_k=0.025+0.95k/20$, with $\tau_c=0.5$.
The scale $\delta_0=0.02$ sets the gap at zero raw output; the gaps can
shrink or expand during training. All $K$ predictions receive pinball
supervision with a shared label mask for each action coordinate.
For median decoding, we collect $q_c=m$ across steps and coordinates
and unnormalize the action chunk.

\textbf{Action expert inputs.}\quad
The observation prefix combines SigLIP image embeddings with text embeddings
of the task and discretized robot state. The state is normalized before
being discretized into 256 bins and included in the text prompt.
The action tokens are obtained by linearly projecting $H$ all-zero action
queries. Both training and
inference fix the time condition to zero; its embedding and MLP still
condition the expert's adaptive normalization.

\subsection{Masked pinball loss}
\label{sec:label-mask}
We apply pinball supervision from Theorem~\ref{thm:quantile-objective}
to each action coordinate's $K$ ordered quantiles. For sample $b$, let $a_{bhd}$ be the normalized
action target at future step $h$ and dimension $d$, and let $q_{bhdk}$
be its predicted quantile at level $\tau_k$. With residual
$e_{bhdk}=a_{bhd}-q_{bhdk}$, the pinball loss is
\begin{equation}
  \rho_\tau(e)=\max\{\tau e,(\tau-1)e\}.
  \label{eq:pinball-loss}
\end{equation}
Underprediction is penalized with weight $\tau$, and overprediction
with weight $1-\tau$.

\begin{table}[!t]
  \centering
  \caption{Success rates (\%) on LIBERO and LIBERO-Pro.
  Best in \textbf{bold}; second-best \underline{underlined}.
  Superscripts indicate result sources:
  \textsuperscript{a}\citet{kim2025fine};
  \textsuperscript{b}\citet{wu2026continuous}.}
  \label{tab:libero-pro-combined}
  \vspace{4pt}
  \captionsetup[subtable]{font=small,labelfont=bf,textfont=bf,
    justification=raggedright,singlelinecheck=false,skip=4pt}
  \scalebox{0.8}{%
    \begin{minipage}{1.25\linewidth}
      \centering
      \begin{subtable}[t]{0.485\linewidth}
        \caption{LIBERO}
        \label{tab:libero-comparison}
        \centering
        \small
        \renewcommand{\arraystretch}{1.08}
        \setlength{\tabcolsep}{0.8pt}
  \begin{tabular*}{\linewidth}{@{\extracolsep{\fill}}lccccc@{}}
    \toprule
    Method & Spatial & Object & Goal & Long & Avg. \\
    \midrule
    Diffusion Policy\textsuperscript{a} {\fontsize{8.5}{10}\selectfont\citep{chi2025diffusion}} & 78.3 & 92.5 & 68.3 & 50.5 & 72.4 \\
    OpenVLA\textsuperscript{a} {\fontsize{8.5}{10}\selectfont\citep{kim2024openvla}} & 84.7 & 88.4 & 79.2 & 53.7 & 76.5 \\
    $\pi_0$-FAST\textsuperscript{a} {\fontsize{8.5}{10}\selectfont\citep{pertsch2025fast}} & 96.4 & 96.8 & 88.6 & 60.2 & 85.5 \\
    $\pi_0$\textsuperscript{a} {\fontsize{8.5}{10}\selectfont\citep{black2024pi_0}} & 96.8 & 98.8 & 95.8 & 85.2 & 94.2 \\
    OpenVLA-OFT {\fontsize{8.5}{10}\selectfont\citep{kim2025fine}} & 97.6 & 98.4 & 97.9 & 94.5 & 97.1 \\
    $\pi_{0.5}$ {\fontsize{8.5}{10}\selectfont\citep{pmlr-v305-black25a}} & 98.8 & 98.2 & 98.0 & 92.4 & 96.85 \\
    VLANeXt {\fontsize{8.5}{10}\selectfont\citep{wu2026vlanext}} & \underline{99.0} & 99.2 & 96.6 & 94.8 & 97.4 \\
    SimVLA {\fontsize{8.5}{10}\selectfont\citep{luo2026simvla}} & \textbf{99.6} & \underline{99.8} & \underline{98.6} & 96.4 & 98.6 \\
    WLA-0 {\fontsize{8.5}{10}\selectfont\citep{yang2026world}} & \underline{99.0} & \textbf{100.0} & 97.8 & 97.6 & 98.6 \\
    InternVLA-A1.5 {\fontsize{8.5}{10}\selectfont\citep{ma2026internvla}} & 98.6 & \underline{99.8} & \underline{98.6} & \underline{98.4} & \underline{98.9} \\
    \midrule
    Quantile Head & \underline{99.0} & \textbf{100.0} & \textbf{99.6} & \textbf{98.6} & \textbf{99.3} \\
    \bottomrule
  \end{tabular*}
      \end{subtable}%
      \hfill
      \begin{subtable}[t]{0.485\linewidth}
        \caption{LIBERO-Pro}
        \label{tab:libero-pro-comparison}
        \centering
        \small
        \renewcommand{\arraystretch}{1.08}
        \setlength{\tabcolsep}{0.8pt}
  \begin{tabular*}{\linewidth}{@{\extracolsep{\fill}}lccccc@{}}
    \toprule
    Method & Object & Pos. & Sem. & Task & Avg. \\
    \midrule
    OpenVLA {\fontsize{8.5}{10}\selectfont\citep{kim2024openvla}} & \underline{93.0} & 0.0 & 97.3 & 0.0 & 47.6 \\
    $\pi_0$ {\fontsize{8.5}{10}\selectfont\citep{black2024pi_0}} & 90.5 & 0.0 & 90.5 & 0.0 & 45.3 \\
    $\pi_{0.5}$ {\fontsize{8.5}{10}\selectfont\citep{pmlr-v305-black25a}} & \textbf{96.0} & \underline{20.8} & 95.8 & 0.8 & \underline{53.3} \\
    MolmoAct {\fontsize{8.5}{10}\selectfont\citep{lee2025molmoact}} & 76.0 & 1.5 & 85.8 & 1.5 & 41.2 \\
    X-VLA {\fontsize{8.5}{10}\selectfont\citep{zheng2026x}} & 79.0 & 0.8 & 96.8 & 6.8 & 45.8 \\
    OpenVLA-OFT\textsuperscript{b} {\fontsize{8.5}{10}\selectfont\citep{kim2025fine}} & 27.6 & 5.8 & 59.0 & 0.6 & 23.3 \\
    X-VLA (CR eval.)\textsuperscript{b} {\fontsize{8.5}{10}\selectfont\citep{zheng2026x}} & 78.4 & 1.6 & 82.9 & 16.4 & 44.8 \\
    VLA-Adapter\textsuperscript{b} {\fontsize{8.5}{10}\selectfont\citep{wang2026vla}} & 73.8 & 0.0 & 90.8 & \underline{19.8} & 46.1 \\
    SimVLA {\fontsize{8.5}{10}\selectfont\citep{luo2026simvla}} & 81.5 & 8.3 & \textbf{98.8} & 6.0 & 48.6 \\
    SUREFlow {\fontsize{8.5}{10}\selectfont\citep{islam2026sureflow}} & 68.5 & 0.0 & 87.8 & 0.0 & 39.1 \\
    \midrule
    \textbf{Quantile Head} & 85.8 & \textbf{31.3} & \underline{97.8} & \textbf{26.0} & \textbf{60.2} \\
    \bottomrule
  \end{tabular*}
      \end{subtable}
    \end{minipage}%
  }
\end{table}

During training, we randomly withhold supervision for a subset of valid
action coordinates. Let $V_{bhd}\in\{0,1\}$ indicate a non-padded coordinate
and $N_b=\sum_{h,d}V_{bhd}$. For each sample, we uniformly select
$\min(\lceil rN_b\rceil,\max(N_b-1,0))$ valid coordinates without
replacement and exclude their labels from the loss, using $r=0.1$.
The mask is redrawn on each training forward pass and retains at least
one label whenever $N_b>0$. Define $M_{bhd}=1$ for valid, retained
coordinates and $0$ otherwise. As illustrated in
Figure~\ref{fig:model-details}(b), each $M_{bhd}$ is shared across all
$K$ quantiles, so the entire set of losses for a coordinate is retained
or dropped together. This mask acts on the supervision terms.

We average over the retained coordinates and all $K$ quantile levels
within each sample, then over the $B$ samples in the batch:
\begin{equation}
  \mathcal L=\frac{1}{B}\sum_{b=1}^{B}
    \frac{\sum_{h,d,k}M_{bhd}\,
      \rho_{\tau_k}(e_{bhdk})}
      {K\max\!\left(1,\sum_{h,d}M_{bhd}\right)}.
  \label{eq:masked-pinball}
\end{equation}
Action dimensions introduced for model padding are removed before loss
computation. Evaluation losses use all valid coordinates, with random
dropping disabled and padded action steps excluded.

\section{Experiments}
\label{sec:experiments}

\subsection{Experimental Setting}
\label{sec:experimental-setup}

\textbf{Benchmarks and Tasks.}\quad
We evaluate our method on three simulation benchmarks and two real-robot tasks.
The simulation benchmarks are LIBERO \citep{liu2023libero},
LIBERO-Plus \citep{fei2026libero}, and LIBERO-Pro
\citep{zhou2025libero}.
We use four LIBERO suites: Spatial, Object, Goal, and Long (LIBERO-10),
with ten tasks each.
LIBERO-Plus perturbs these 40 tasks to create 10,030 instances, covering
camera viewpoints, robot initial states, language instructions,
lighting, backgrounds, sensor noise, and object layouts.
LIBERO-Pro also evaluates robustness to changes in objects, positions,
instructions, task goals, and environments.
We further evaluate real-world performance on two manipulation tasks,
with results in Table~\ref{tab:real-robot-results} and the platform and tasks
described in Appendix~\ref{sec:appendix-real-robot}.
Evaluation protocols and implementation details are provided in
Appendix~\ref{sec:appendix-simulation-results}.
Baselines cited from other work retain their reported protocols.

\subsection{Task Performance with Default Median Decoding}
\label{sec:benchmark-comparison}

\textbf{Overall Performance.}\quad
With median decoding, Quantile Head achieves 99.3\% average success across
the four LIBERO suites, 0.4 percentage points above
the strongest compared baseline (Table~\ref{tab:libero-pro-combined}).
On LIBERO-Plus, it reaches 87.1\% zero-shot and 89.1\% fine-tuned success,
exceeding the respective strongest compared baselines by 1.4 and 5.0 percentage points
(Table~\ref{tab:libero-plus-comparison}).
On LIBERO-Pro, it achieves 60.2\% without further training,
exceeding $\pi_{0.5}$ at 53.3\% by 6.9 percentage points.
These mean success rates support the generalization of the full median
policy across the evaluated benchmarks and perturbations.
Matched head ablations examine the benefit of joint quantile supervision
in Section~\ref{sec:training-ablations}.

\textbf{Real-robot Performance.}\quad
With median decoding, Quantile Head achieves 75.0\% average success
across the two real-robot tasks, compared with 58.5\% for $\pi_{0.5}$,
a 16.5 percentage point gain
(Table~\ref{tab:real-robot-results}).
The $\pi_{0.5}$ baseline is fully fine-tuned; both methods use the same
demonstration data, optimization settings, and evaluation conditions
(Appendix~\ref{sec:appendix-real-robot}).
Success increases from 31\% to 52\% for apple placement
(+21 percentage points) and from 86\% to 98\% for cuboid removal
(+12 percentage points). These results show higher success on both
evaluated tasks using deterministic median decoding, although apple
placement remains challenging at 52\% success.

\begin{table}[!t]
  \centering
  \caption{Success rates (\%) on LIBERO-Plus.
  Within each setting, best in \textbf{bold}; second-best \underline{underlined}.
  Superscript sources:
  \textsuperscript{a}\citet{fei2026libero};
  \textsuperscript{b}\citet{zhong2026acot};
  \textsuperscript{d}\citet{shi2026memoryvla++}.
  Unmarked baselines use their original papers.}
  \label{tab:libero-plus-comparison}
  \vspace{4pt}
  \scalebox{0.75}{%
    \begin{minipage}{\linewidth}
      \centering
  \footnotesize
  \renewcommand{\arraystretch}{1.08}
  \setlength{\tabcolsep}{3pt}
  \newlength{\liberoplusmethodwidth}
  \settowidth{\liberoplusmethodwidth}{GE-Act~2.0 \citep{liu2026ge}}
  \begin{tabular*}{\linewidth}{@{\extracolsep{\fill}}lcccccccc@{}}
    \multicolumn{9}{@{}l}{\textbf{(a) Zero-shot transfer}} \\
    \toprule
    \makebox[\liberoplusmethodwidth][l]{Method} & Camera & Robot & Lang. & Light & Bkg. & Noise & Layout & Total \\
    \midrule
    OpenVLA\textsuperscript{a} \citep{kim2024openvla} & 0.8 & 3.5 & 23.0 & 8.1 & 34.8 & 15.2 & 28.5 & 15.6 \\
    $\pi_0$-FAST\textsuperscript{a} \citep{pertsch2025fast} & 65.1 & 21.6 & 61.0 & 73.2 & 73.2 & 74.4 & 68.8 & 61.6 \\
    $\pi_0$\textsuperscript{a} \citep{black2024pi_0} & 13.8 & 6.0 & 58.8 & 85.0 & 81.4 & 79.0 & 68.9 & 53.6 \\
    OpenVLA-OFT\textsuperscript{a} \citep{kim2025fine} & 56.4 & 31.9 & 79.5 & 88.7 & 93.3 & 75.8 & 74.2 & 69.6 \\
    $\pi_{0.5}$\textsuperscript{b} \citep{pmlr-v305-black25a} & 75.8 & 79.4 & 83.3 & 95.5 & \underline{95.0} & 89.6 & 87.0 & \underline{85.7} \\
    VLANeXt \citep{wu2026vlanext} & \underline{90.4} & 65.7 & 81.8 & 95.9 & 82.5 & 94.1 & 80.8 & 83.9 \\
    GAM \citep{han2026geometric} & 83.1 & 70.0 & \underline{84.8} & \textbf{97.2} & 94.3 & \underline{95.3} & 79.1 & 85.5 \\
    GE-Act~2.0 \citep{liu2026ge} & \textbf{94.1} & 50.7 & 81.4 & 94.0 & 60.5 & \textbf{95.5} & 83.1 & 80.4 \\
    ACoT-VLA\textsuperscript{b} \citep{zhong2026acot} & 68.9 & \underline{80.3} & 84.1 & 95.6 & 93.1 & 81.5 & \textbf{88.3} & 83.6 \\
    \midrule
    Quantile Head & 71.6 & \textbf{80.7} & \textbf{87.6} & \underline{96.1} & \textbf{97.2} & 94.8 & \underline{87.8} & \textbf{87.1} \\
    \bottomrule
  \end{tabular*}
  \vspace{7pt}
  \begin{tabular*}{\linewidth}{@{\extracolsep{\fill}}lcccccccc@{}}
    \multicolumn{9}{@{}l}{\textbf{(b) Supervised fine-tuning}} \\
    \toprule
    \makebox[\liberoplusmethodwidth][l]{Method} & Camera & Robot & Lang. & Light & Bkg. & Noise & Layout & Total \\
    \midrule
    $\pi_0$\textsuperscript{b} \citep{black2024pi_0} & 79.6 & 21.1 & 72.5 & 84.7 & 86.2 & 68.3 & 69.4 & 67.4 \\
    $\pi_{0.5}$\textsuperscript{b} \citep{pmlr-v305-black25a} & 70.3 & 41.7 & 81.1 & \underline{97.3} & 94.6 & 71.8 & \underline{84.9} & 75.7 \\
    OpenVLA-OFT+ \citep{fei2026libero} & \underline{92.8} & 30.3 & \underline{85.8} & 94.9 & 93.9 & 89.3 & 77.6 & 79.6 \\
    MemoryVLA\textsuperscript{d} \citep{shi2026memoryvla} & 91.4 & 48.6 & 79.4 & 95.2 & 95.3 & 94.0 & 75.7 & 81.9 \\
    MemoryVLA++ \citep{shi2026memoryvla++} & \textbf{96.8} & 49.7 & 71.0 & 96.6 & \underline{97.0} & \textbf{96.0} & 78.6 & 82.7 \\
    ACoT-VLA\textsuperscript{b} \citep{zhong2026acot} & 91.2 & \underline{62.5} & 80.3 & 95.1 & 91.5 & 88.3 & \underline{84.9} & \underline{84.1} \\
    \midrule
    Quantile Head (SFT) & 88.6 & \textbf{77.0} & \textbf{86.2} & \textbf{97.4} & \textbf{97.3} & \underline{94.1} & \textbf{87.3} & \textbf{89.1} \\
    \bottomrule
  \end{tabular*}
    \end{minipage}%
  }
\end{table}

\begin{table}[!t]
  \centering
  \caption{Real-robot success rates (\%).
  The $\pi_{0.5}$ baseline is fully fine-tuned.
  $\Delta$ gives the gain over $\pi_{0.5}$ in percentage points.}
  \label{tab:real-robot-results}
  \vspace{4pt}
  \scalebox{0.75}{%
    \begin{minipage}{\linewidth}
      \centering
      \small
      \renewcommand{\arraystretch}{1.08}
      \setlength{\tabcolsep}{4pt}
      \begin{tabular*}{\linewidth}{@{\extracolsep{\fill}}p{0.54\linewidth}ccc@{}}
        \toprule
        Real-robot Task & $\pi_{0.5}$ & Quantile Head & $\Delta$ \\
        \midrule
        Place apple on yellow plate & 31 & 52 & +21 \\
        Remove cuboid from blue plate & 86 & 98 & +12 \\
        \midrule
        Average & 58.5 & 75.0 & +16.5 \\
        \bottomrule
      \end{tabular*}
    \end{minipage}%
  }
\end{table}

\subsection{Ablation Studies}
\label{sec:ablation-analysis}
\label{sec:training-ablations}

\textbf{Head and Component Ablations.}\quad
Under matched training and evaluation conditions, median-decoded Quantile
Head achieves 99.3\% average success across the four LIBERO suites,
compared with 96.0\%, 97.3\%, and 96.3\% for $L_2$, $L_1$, and sampled
Flow Matching, respectively
(Table~\ref{tab:libero-component-ablation}(a)).
Mean episode time, including failures, decreases from 8.13\,s for $L_1$
to 5.88\,s, a 27.7\% reduction.

In Table~\ref{tab:libero-component-ablation}(b), Median-Only reaches
97.5\% mean success, compared with 99.3\% for the full head.
Detached Median Anchor reaches
98.7\%: noncentral losses still update shared features and gaps, but their
direct gradient to the median projection is blocked. These mean rates are
consistent with potential benefits from both shared supervision and the
direct gradient path; task success does not measure error to the true
conditional median. Symmetric Quantiles reaches 97.9\%, below the full
head with independently learned left and right offsets.
Noise Query and VLM Query reach 93.8\% and 92.8\%, respectively, showing
that all-zero action queries outperform both alternatives.

In (c), joint quantile supervision yields higher mean evaluation success
across all four prompt and label-mask configurations: 83.5\% to 91.8\%,
95.5\% to 98.8\%, 86.0\% to 97.0\%, and 97.5\% to 99.3\%.
These matched comparisons suggest potential generalization benefits
from joint quantile supervision for the median policy, even when
inference uses only the median without quantile sampling.

\begin{table}[!t]
  \centering
  \caption{Ablations on the four LIBERO suites: success rates (\%) and episode times (s).
  Episode times in (a) average over all evaluation episodes.
  See Appendix~\ref{sec:appendix-simulation-results} for ablation settings.}
  \label{tab:libero-component-ablation}
  \vspace{4pt}
  \scalebox{0.8}{%
    \begin{minipage}{\linewidth}
      \centering
      \small
      \renewcommand{\arraystretch}{1.08}
      \setlength{\tabcolsep}{3pt}
      \begin{tabular*}{\linewidth}{@{\extracolsep{\fill}}ccc|@{\extracolsep{0pt}}ccccc@{}}
        \toprule
        \multicolumn{3}{@{}l|}{Setting} & \makebox[0.08\linewidth][c]{Spatial$\uparrow$} & \makebox[0.08\linewidth][c]{Object$\uparrow$} & \makebox[0.08\linewidth][c]{Goal$\uparrow$} & \makebox[0.08\linewidth][c]{Long$\uparrow$} & \makebox[0.08\linewidth][c]{Avg.$\uparrow$} \\
        \midrule
        \multicolumn{2}{@{}l}{\textbf{(a) Supervised Fine-Tuning}} & {\footnotesize Avg. Episode Time (s)$\downarrow$} & & & & & \\
        \multicolumn{2}{@{}l}{$L_2$} & 8.73 & \makebox[0.08\linewidth][c]{97.0} & \makebox[0.08\linewidth][c]{98.0} & \makebox[0.08\linewidth][c]{96.0} & \makebox[0.08\linewidth][c]{93.0} & \makebox[0.08\linewidth][c]{96.0} \\
        \multicolumn{2}{@{}l}{$L_1$} & 8.13 & \makebox[0.08\linewidth][c]{99.0} & \makebox[0.08\linewidth][c]{99.0} & \makebox[0.08\linewidth][c]{96.0} & \makebox[0.08\linewidth][c]{95.0} & \makebox[0.08\linewidth][c]{97.3} \\
        \multicolumn{2}{@{}l}{Flow Matching} & 12.57 & \makebox[0.08\linewidth][c]{97.0} & \makebox[0.08\linewidth][c]{98.0} & \makebox[0.08\linewidth][c]{96.0} & \makebox[0.08\linewidth][c]{94.0} & \makebox[0.08\linewidth][c]{96.3} \\
        \multicolumn{2}{@{}l}{Quantile Head} & 5.88 & \makebox[0.08\linewidth][c]{99.0} & \makebox[0.08\linewidth][c]{100.0} & \makebox[0.08\linewidth][c]{99.6} & \makebox[0.08\linewidth][c]{98.6} & \makebox[0.08\linewidth][c]{99.3} \\
        \midrule
        \multicolumn{8}{@{}l}{\textbf{(b) Prompt + Label Mask}} \\
        \multicolumn{3}{@{}l|}{Quantile Head} & \makebox[0.08\linewidth][c]{99.0} & \makebox[0.08\linewidth][c]{100.0} & \makebox[0.08\linewidth][c]{99.6} & \makebox[0.08\linewidth][c]{98.6} & \makebox[0.08\linewidth][c]{99.3} \\
        \multicolumn{3}{@{}l|}{Median-Only} & \makebox[0.08\linewidth][c]{98.0} & \makebox[0.08\linewidth][c]{99.0} & \makebox[0.08\linewidth][c]{97.0} & \makebox[0.08\linewidth][c]{96.0} & \makebox[0.08\linewidth][c]{97.5} \\
        \multicolumn{3}{@{}l|}{Detached Median Anchor} & \makebox[0.08\linewidth][c]{99.0} & \makebox[0.08\linewidth][c]{99.0} & \makebox[0.08\linewidth][c]{99.0} & \makebox[0.08\linewidth][c]{97.8} & \makebox[0.08\linewidth][c]{98.7} \\
        \multicolumn{3}{@{}l|}{Symmetric Quantiles} & \makebox[0.08\linewidth][c]{99.0} & \makebox[0.08\linewidth][c]{98.0} & \makebox[0.08\linewidth][c]{98.0} & \makebox[0.08\linewidth][c]{96.6} & \makebox[0.08\linewidth][c]{97.9} \\
        \multicolumn{3}{@{}l|}{Noise Query} & \makebox[0.08\linewidth][c]{94.0} & \makebox[0.08\linewidth][c]{95.0} & \makebox[0.08\linewidth][c]{93.0} & \makebox[0.08\linewidth][c]{93.0} & \makebox[0.08\linewidth][c]{93.8} \\
        \multicolumn{3}{@{}l|}{VLM Query} & \makebox[0.08\linewidth][c]{94.0} & \makebox[0.08\linewidth][c]{96.0} & \makebox[0.08\linewidth][c]{91.0} & \makebox[0.08\linewidth][c]{90.0} & \makebox[0.08\linewidth][c]{92.8} \\
        \midrule
        \multicolumn{8}{@{}l}{\textbf{(c) Component Ablations (Frozen VLM)}} \\
        Prompts & Pinball Loss & Label Mask & & & & & \\
        \midrule
         &  &  & \makebox[0.08\linewidth][c]{83.0} & \makebox[0.08\linewidth][c]{86.0} & \makebox[0.08\linewidth][c]{80.0} & \makebox[0.08\linewidth][c]{85.0} & \makebox[0.08\linewidth][c]{83.5} \\
        $\checkmark$ &  &  & \makebox[0.08\linewidth][c]{97.0} & \makebox[0.08\linewidth][c]{97.0} & \makebox[0.08\linewidth][c]{95.0} & \makebox[0.08\linewidth][c]{93.0} & \makebox[0.08\linewidth][c]{95.5} \\
         & $\checkmark$ &  & \makebox[0.08\linewidth][c]{90.0} & \makebox[0.08\linewidth][c]{94.0} & \makebox[0.08\linewidth][c]{94.0} & \makebox[0.08\linewidth][c]{89.0} & \makebox[0.08\linewidth][c]{91.8} \\
         &  & $\checkmark$ & \makebox[0.08\linewidth][c]{85.0} & \makebox[0.08\linewidth][c]{88.0} & \makebox[0.08\linewidth][c]{85.0} & \makebox[0.08\linewidth][c]{86.0} & \makebox[0.08\linewidth][c]{86.0} \\
        $\checkmark$ &  & $\checkmark$ & \makebox[0.08\linewidth][c]{98.0} & \makebox[0.08\linewidth][c]{99.0} & \makebox[0.08\linewidth][c]{97.0} & \makebox[0.08\linewidth][c]{96.0} & \makebox[0.08\linewidth][c]{97.5} \\
        $\checkmark$ & $\checkmark$ &  & \makebox[0.08\linewidth][c]{99.0} & \makebox[0.08\linewidth][c]{100.0} & \makebox[0.08\linewidth][c]{98.0} & \makebox[0.08\linewidth][c]{98.0} & \makebox[0.08\linewidth][c]{98.8} \\
         & $\checkmark$ & $\checkmark$ & \makebox[0.08\linewidth][c]{97.0} & \makebox[0.08\linewidth][c]{99.0} & \makebox[0.08\linewidth][c]{99.0} & \makebox[0.08\linewidth][c]{93.0} & \makebox[0.08\linewidth][c]{97.0} \\
        $\checkmark$ & $\checkmark$ & $\checkmark$ & \makebox[0.08\linewidth][c]{99.0} & \makebox[0.08\linewidth][c]{100.0} & \makebox[0.08\linewidth][c]{99.6} & \makebox[0.08\linewidth][c]{98.6} & \makebox[0.08\linewidth][c]{99.3} \\
        \bottomrule
      \end{tabular*}
    \end{minipage}%
  }
\end{table}

\subsection{Exploring Alternative Sampling Strategies}
\label{sec:sampling-ablations}
\label{sec:action-analysis}
\label{sec:multimodal-action-decoding}

\textbf{Sampling Analysis.}\quad
In the LIBERO-Long sampling evaluation, with cross-chunk correlation
$\rho=0.7$, narrowing the sampling window
from $[0.1,0.9]$ to $[0.4,0.6]$ increases LIBERO-Long success from
93.0\% to 99.6\%, exceeding median decoding (98.6\%) by 1.0 percentage
point (Figure~\ref{fig:inference-sampling-ablation}).
These results illustrate flexibility at inference
and motivate future work on selecting sampling strategies. Median decoding
requires no sampling choices and remains our default.
Narrow windows keep sampled quantile levels near 0.5, and temporal
correlation couples these levels across chunks; median decoding directly
uses the learned center.

\begin{figure}[!t]
  \centering
  \scalebox{0.8}{%
    \begin{minipage}{\linewidth}
      \centering
\begingroup
\captionsetup[subfigure]{font=footnotesize,labelfont=bf,textfont=bf,
  justification=centering,singlelinecheck=false,skip=3pt}
{\footnotesize
  \mbox{\textcolor[HTML]{8172B2}{\rule{9pt}{5pt}}\,$[0.1,0.9]$}\hspace{12pt}%
  \mbox{\textcolor[HTML]{4C8AB4}{\rule{9pt}{5pt}}\,$[0.2,0.8]$}\hspace{12pt}%
  \mbox{\textcolor[HTML]{D89952}{\rule{9pt}{5pt}}\,$[0.3,0.7]$}\hspace{12pt}%
  \mbox{\textcolor[HTML]{479D85}{\rule{9pt}{5pt}}\,$[0.4,0.6]$}\par}
\vspace{3pt}
\subfloat[Action decoding.\label{fig:sampling-action-decoding}]{%
  \includegraphics[width=.415\linewidth]{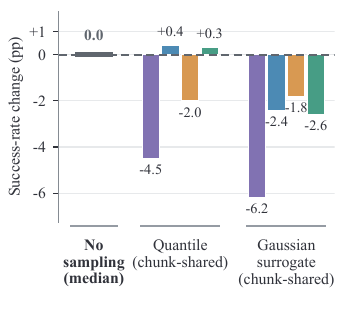}%
}\hfill
\subfloat[Temporal sampling structure.\label{fig:sampling-temporal-structure}]{%
  \includegraphics[width=.56\linewidth]{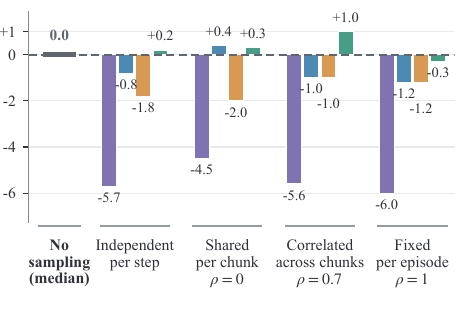}%
}
\endgroup
    \end{minipage}%
  }
  \caption{\textbf{Inference-time sampling on LIBERO-Long.}
  Colors denote quantile windows; bars show success-rate changes from
  median decoding (98.6\%; dashed lines) in percentage points.
  Quantile (a) and Shared per chunk (b) use the same setting ($\rho=0$).
  See Appendix~\ref{sec:appendix-simulation-results}.}
  \label{fig:inference-sampling-ablation}
  \vspace{-8pt}
\end{figure}

\textbf{Action Distribution Analysis.}\quad
Figure~\ref{fig:quantile-mean-range} shows predicted marginal ranges in one
LIBERO-Long rollout: narrower at selected transport moments and wider
during grasping and placing.
This pattern suggests a possible role for local action adjustments around
object contact, but does not establish distributional accuracy.
Per-dimension ranges appear in Appendix
Figure~\ref{fig:quantile-ranges-all-dimensions}.

\begin{figure}[!t]
  \centering
  \includegraphics[width=0.8\linewidth,trim=0 8 0 12,clip]{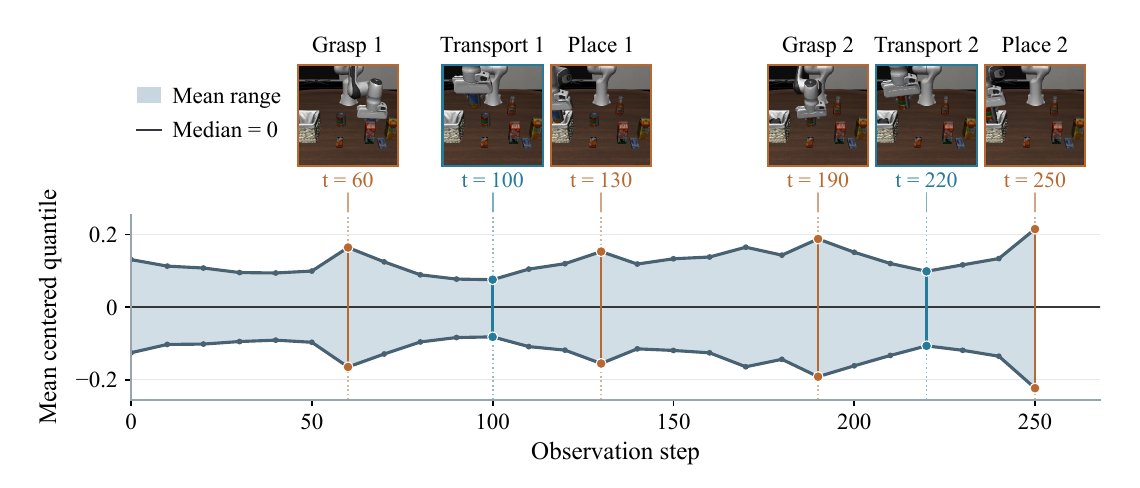}
  \caption{\textbf{Mean predicted action spread during manipulation.}
  Median-centered $Q_0$--$Q_1$ ranges averaged over six normalized action
  dimensions (\mbox{$h=0$}); endpoints use linear tail extrapolation.}
  \label{fig:quantile-mean-range}
\end{figure}

\textbf{Sampling in Two Controlled Scenes.}\quad
We train each head on two demonstrations per scene, with identical initial
observations and instructions but opposite detours
(Figure~\ref{fig:four-head-dx-distribution}). Without an obstacle,
median-decoded Quantile Head and both regression heads ($L_1$, $L_2$)
reach 100\% success; uniform and density-weighted quantile sampling
reach 92\% and 95\%. With an obstacle, these sampling rules reach 6\% and
13\%, while median decoding, both regression heads, and Flow Matching
reach 0\%. Sampling benefits are thus task-dependent, and obstacle-scene
success remains low. Density weighting requires no retraining.

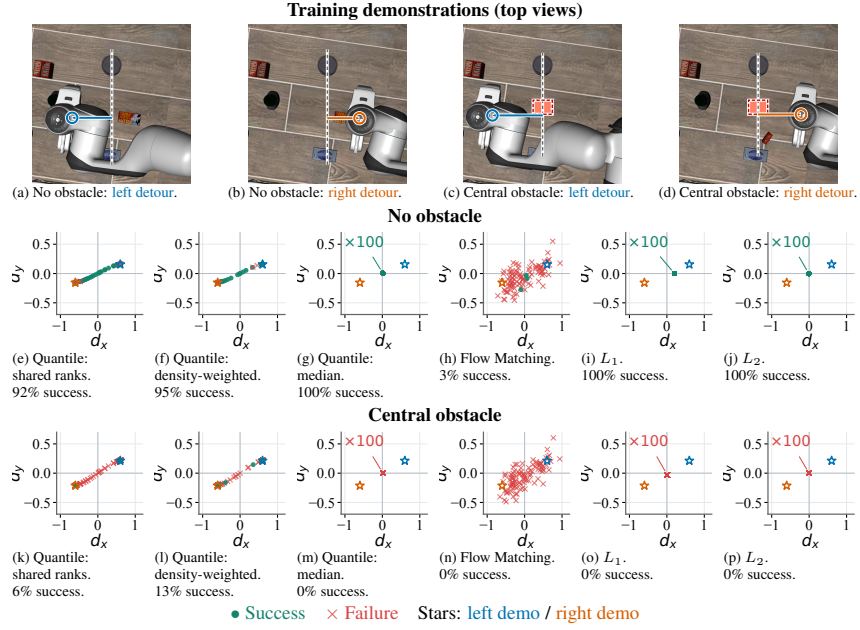
\begin{figure}[!t]
  \centering
\begingroup
\definecolor{mainOutcomeSuccess}{HTML}{13846B}
\definecolor{mainOutcomeFailure}{HTML}{D94443}
\definecolor{mainDemoLeft}{HTML}{0072B2}
\definecolor{mainDemoRight}{HTML}{D55E00}
\colorlet{demoLeft}{mainDemoLeft}
\colorlet{demoRight}{mainDemoRight}
\scalebox{0.8}{%
\begin{minipage}{\linewidth}
\captionsetup[subfigure]{font=scriptsize,justification=raggedright,
  singlelinecheck=false,skip=1pt}
\centering
{\small\bfseries Training demonstrations (top views)\par}
\par\vspace{1pt}
  \begin{subfigure}[t]{0.235\linewidth}
    \centering
    \resizebox{0.8\linewidth}{!}{%
    \begin{tikzpicture}
      \node[anchor=south west,inner sep=0] (view) at (0,0)
        {\includegraphics[width=\linewidth]{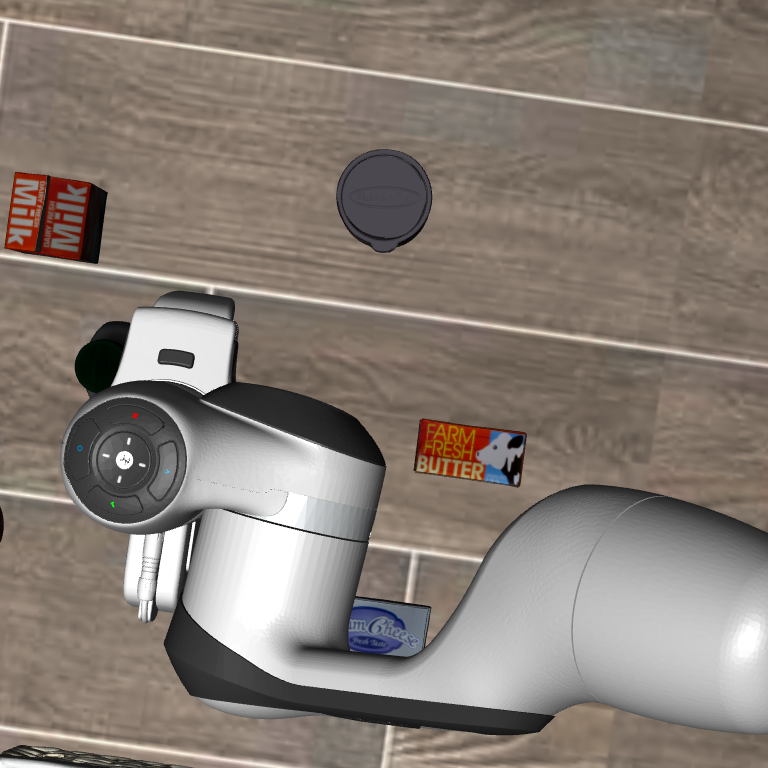}};
      \begin{scope}[x={(view.south east)},y={(view.north west)}]
        \clip (0,0) rectangle (1,1);
        \csname arxivOverlay@two_demo_data@left\endcsname
      \end{scope}
    \end{tikzpicture}%
    }
    \caption{No obstacle: \textcolor{mainDemoLeft}{left detour}.}
  \end{subfigure}%
\hfill
  \begin{subfigure}[t]{0.235\linewidth}
    \centering
    \resizebox{0.8\linewidth}{!}{%
    \begin{tikzpicture}
      \node[anchor=south west,inner sep=0] (view) at (0,0)
        {\includegraphics[width=\linewidth]{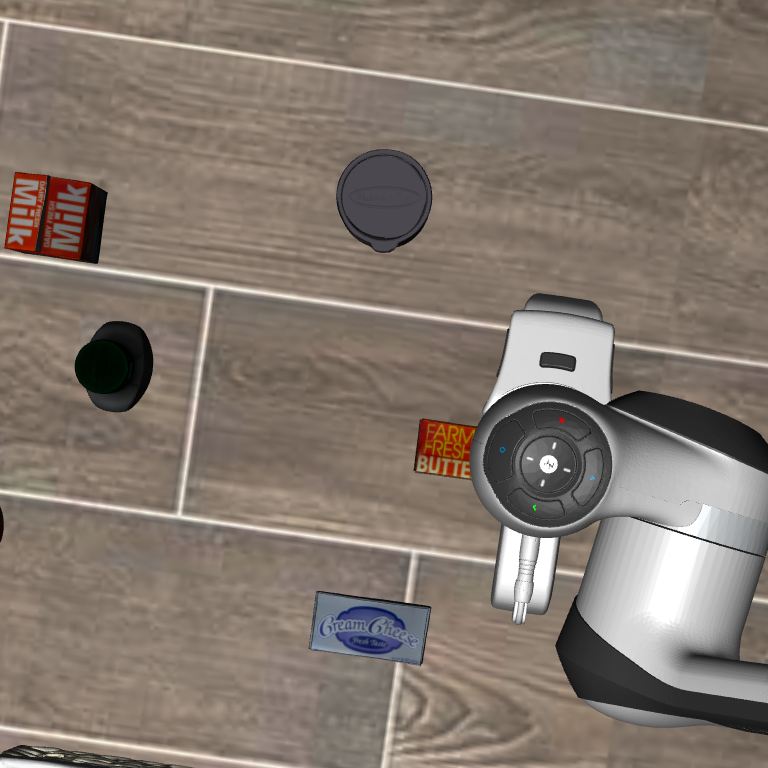}};
      \begin{scope}[x={(view.south east)},y={(view.north west)}]
        \clip (0,0) rectangle (1,1);
        \csname arxivOverlay@two_demo_data@right\endcsname
      \end{scope}
    \end{tikzpicture}%
    }
    \caption{No obstacle: \textcolor{mainDemoRight}{right detour}.}
  \end{subfigure}%
\hfill
  \begin{subfigure}[t]{0.235\linewidth}
    \centering
    \resizebox{0.8\linewidth}{!}{%
    \begin{tikzpicture}
      \node[anchor=south west,inner sep=0] (view) at (0,0)
        {\includegraphics[width=\linewidth]{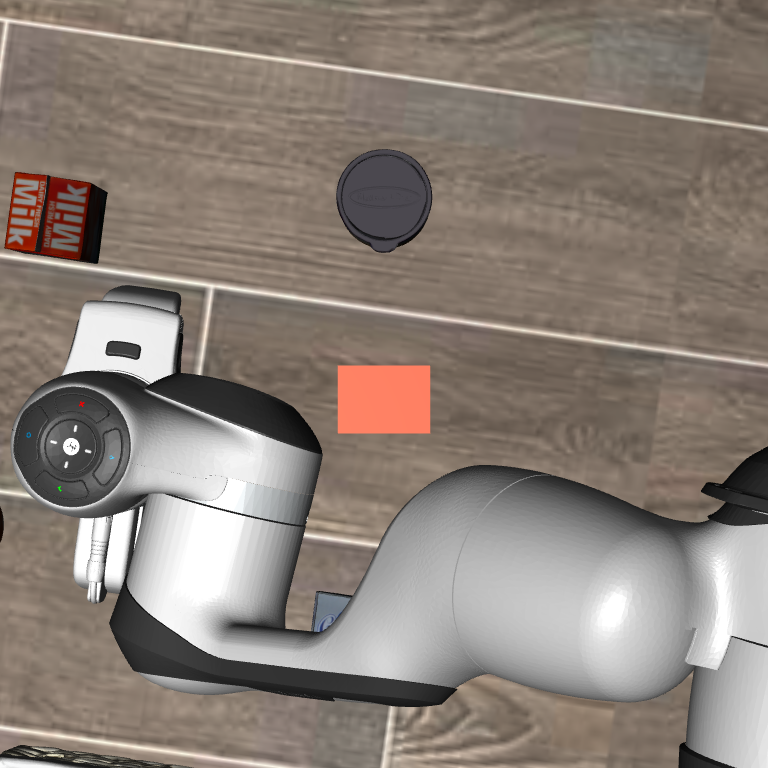}};
      \begin{scope}[x={(view.south east)},y={(view.north west)}]
        \clip (0,0) rectangle (1,1);
        \csname arxivOverlay@two_demo_obstacle_data@left\endcsname
      \end{scope}
    \end{tikzpicture}%
    }
    \caption{Central obstacle: \textcolor{mainDemoLeft}{left detour}.}
  \end{subfigure}%
\hfill
  \begin{subfigure}[t]{0.235\linewidth}
    \centering
    \resizebox{0.8\linewidth}{!}{%
    \begin{tikzpicture}
      \node[anchor=south west,inner sep=0] (view) at (0,0)
        {\includegraphics[width=\linewidth]{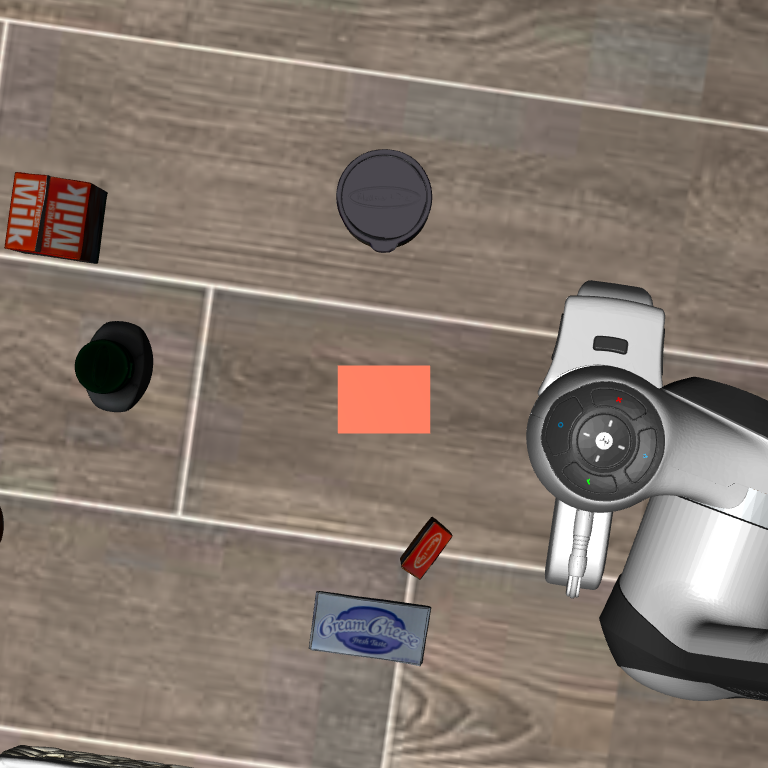}};
      \begin{scope}[x={(view.south east)},y={(view.north west)}]
        \clip (0,0) rectangle (1,1);
        \csname arxivOverlay@two_demo_obstacle_data@right\endcsname
      \end{scope}
    \end{tikzpicture}%
    }
    \caption{Central obstacle: \textcolor{mainDemoRight}{right detour}.}
  \end{subfigure}%

\par\vspace{2pt}
{\small\bfseries No obstacle\par}
\par\vspace{1pt}
  \begin{subfigure}[t]{0.158\linewidth}
    \centering
    \includegraphics[width=\linewidth]{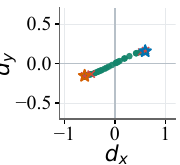}%
    \caption{Quantile:\newline shared ranks.\newline 92\% success.}
  \end{subfigure}%
\hfill
  \begin{subfigure}[t]{0.158\linewidth}
    \centering
    \includegraphics[width=\linewidth]{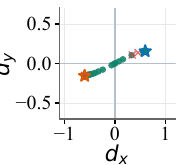}%
    \caption{Quantile:\newline density-weighted.\newline 95\% success.}
  \end{subfigure}%
\hfill
  \begin{subfigure}[t]{0.158\linewidth}
    \centering
    \includegraphics[width=\linewidth]{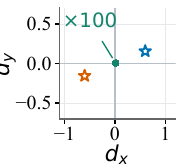}%
    \caption{Quantile:\newline median.\newline 100\% success.}
  \end{subfigure}%
\hfill
  \begin{subfigure}[t]{0.158\linewidth}
    \centering
    \includegraphics[width=\linewidth]{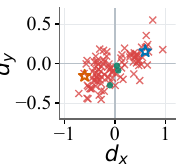}%
    \caption{Flow Matching.\newline 3\% success.}
  \end{subfigure}%
\hfill
  \begin{subfigure}[t]{0.158\linewidth}
    \centering
    \includegraphics[width=\linewidth]{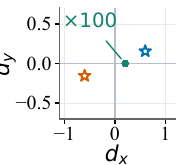}%
    \caption{$L_1$.\newline 100\% success.}
  \end{subfigure}%
\hfill
  \begin{subfigure}[t]{0.158\linewidth}
    \centering
    \includegraphics[width=\linewidth]{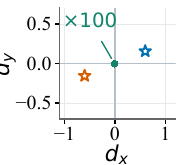}%
    \caption{$L_2$.\newline 100\% success.}
  \end{subfigure}%

\par\vspace{2pt}
{\small\bfseries Central obstacle\par}
\par\vspace{1pt}
  \begin{subfigure}[t]{0.158\linewidth}
    \centering
    \includegraphics[width=\linewidth]{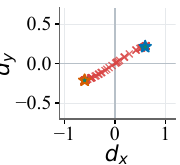}%
    \caption{Quantile:\newline shared ranks.\newline 6\% success.}
  \end{subfigure}%
\hfill
  \begin{subfigure}[t]{0.158\linewidth}
    \centering
    \includegraphics[width=\linewidth]{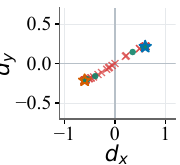}%
    \caption{Quantile:\newline density-weighted.\newline 13\% success.}
  \end{subfigure}%
\hfill
  \begin{subfigure}[t]{0.158\linewidth}
    \centering
    \includegraphics[width=\linewidth]{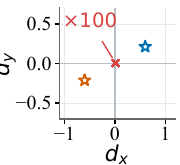}%
    \caption{Quantile:\newline median.\newline 0\% success.}
  \end{subfigure}%
\hfill
  \begin{subfigure}[t]{0.158\linewidth}
    \centering
    \includegraphics[width=\linewidth]{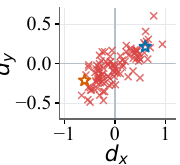}%
    \caption{Flow Matching.\newline 0\% success.}
  \end{subfigure}%
\hfill
  \begin{subfigure}[t]{0.158\linewidth}
    \centering
    \includegraphics[width=\linewidth]{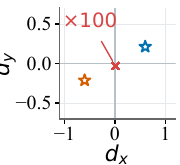}%
    \caption{$L_1$.\newline 0\% success.}
  \end{subfigure}%
\hfill
  \begin{subfigure}[t]{0.158\linewidth}
    \centering
    \includegraphics[width=\linewidth]{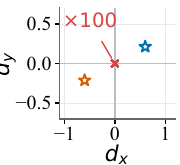}%
    \caption{$L_2$.\newline 0\% success.}
  \end{subfigure}%

\par\vspace{1pt}
{\footnotesize
\textcolor{mainOutcomeSuccess}{$\bullet$ Success}\quad
\textcolor{mainOutcomeFailure}{$\times$ Failure}\quad
Stars: \textcolor{mainDemoLeft}{left demo}\;/\;\textcolor{mainDemoRight}{right demo}\par}
\end{minipage}%
}
\caption{\textbf{Training demonstrations and action head outputs.}
Top views show demonstration states; points show initial $(d_x,d_y)$ commands
from one training run, colored by outcome. Subcaptions report success rates;
$\times100$ marks overlapping points.
Details: Appendices~\ref{sec:two-demo-distributions}--\ref{sec:two-demo-obstacle}.}
\label{fig:four-head-dx-distribution}
\endgroup
\end{figure}

\par\vfill
\begin{samepage}
\section{Conclusion}
\label{sec:conclusion}
We presented Quantile Head, a theory-inspired method that jointly supervises
ordered marginal action quantiles for single-pass median control. Our local
analysis shows lower median-update variance than median-only supervision
with nearby calibrated quantiles, fixed gaps, smooth densities, and matched
correction rates. Joint supervision
improves mean success across four matched LIBERO configurations. The median
policy reaches 99.3\% on LIBERO and 75.0\% across two real-robot tasks.
Sampling yields gains and losses; systematic strategy selection remains
future work.
\par
\end{samepage}
\clearpage
\subsection*{AI use statement}
We used generative AI tools to assist with language polishing, including
refining sentence structure and improving the clarity and readability of
the manuscript.



\clearpage
\appendix
\numberwithin{equation}{section}
\section{Quantile objectives and local center updates}
\label{sec:appendix}
\label{sec:qr-theory}


\subsection{Proof of Theorem~\ref{thm:quantile-objective}}
\label{sec:proof-quantile-objective}
We first give the background for the change from numerical targets to
threshold events in Section~\ref{sec:theoretical-analysis}. The formal proof
then establishes the two claims of Theorem~\ref{thm:quantile-objective} in
order: the loss identity and the population optimum. We explain each
condition where it is used, then assess its implications for the finite
quantile head. All predictors below are measurable.

\textbf{Recovering squared regression.}\quad
Let $X$ be the observation context and $Y$ a scalar action coordinate.
In \eqref{eq:quantile-mother-objective}, choose
$\nu(\cdot\mid X)=\delta_0$ and $\pi=\delta_0$, where $\delta_0$ is a
point mass at zero. Then $B=0$, $\alpha=0$, $A_0=0$, and $Y-B=Y$.
Writing $g_\theta(X)=f_\theta(X,0,0)$ gives
\begin{equation}
\mathcal L_{\delta_0,\delta_0}(\theta)
=\E_{X,Y}|g_\theta(X)-Y|^2
=:\mathcal L_{\mathrm{reg}}(\theta).
\label{eq:quantile-regression-objective}
\end{equation}
Thus the shared objective reduces exactly to squared action regression.

\textbf{Recovering conditional flow matching.}\quad
Choose $\nu(\cdot\mid X)=N(0,1)$ and let $\pi$ be the CFM training-time
distribution. Draw $B\perp Y\mid X$ and
$\alpha\perp(X,Y,B)$. The straight path and its target velocity are
\[
A_\alpha=(1-\alpha)B+\alpha Y,\qquad
\frac{\mathrm d A_\alpha}{\mathrm d\alpha}=Y-B.
\]
With $v_\theta=f_\theta$, direct substitution gives
\begin{equation}
\mathcal L_{\nu,\pi}(\theta)
=\E_{X,Y,B,\alpha}
 \left|v_\theta\!\left(X,(1-\alpha)B+\alpha Y,\alpha\right)
       -(Y-B)\right|^2
=:\mathcal L_{\mathrm{CFM}}(\theta).
\label{eq:quantile-squared-branches}
\end{equation}
This is the straight-path CFM loss \citep{lipman2023flow}.
The independent draws of $B$ and $\alpha$ specify training augmentation;
they impose no independence assumption on the demonstrated action
coordinates. The implementation uses reverse time
$t=1-\alpha$, so its path is $A_t=tB+(1-t)Y$ and its velocity target
is $B-Y$. The formulas above describe a scalar path. For action chunks,
use the full path $\mathbf A_\alpha=(1-\alpha)\mathbf B+\alpha\mathbf Y$
with Gaussian reference $\mathbf B\sim N(0,I)$; each velocity coordinate
conditions on this full path, and summing the coordinate losses gives the
squared Frobenius loss.

\textbf{What squared supervision learns.}\quad
Write $C=(X,A_\alpha,\alpha)$ and $\Delta=Y-B$. Under finite second moments,
\begin{equation}
\E[(f_\theta(C)-\Delta)^2\mid C]
=\bigl(f_\theta(C)-\E[\Delta\mid C]\bigr)^2
 +\operatorname{Var}(\Delta\mid C).
\label{eq:quantile-squared-conditional-risk}
\end{equation}
The second-moment condition is used only for this squared-risk
decomposition; the quantile result below requires only first moments.
Consequently, unrestricted regression learns $\E[Y\mid X]$, while CFM
learns $\E[Y-B\mid X,A_\alpha,\alpha]$. The former is a point action
estimate. The latter is a mean velocity field whose integration from
random reference noise can generate the conditional action distribution
\citep{lipman2023flow}. For both losses, the output derivative
$\partial_f(f-\Delta)^2=2(f-\Delta)$ grows linearly with the residual.
These observations motivate direct distribution supervision with bounded
feedback at the action-head outputs.

\textbf{Extending supervision to threshold events.}\quad
For direct action-distribution prediction, replace the numerical target
by the family of binary targets $Z_z=\mathbf1\{Y\le z\}$,
$z\in\R$. This is an explicit extension of the supervised representation;
it is not a requirement of CFM. Squared supervision of each event has
conditional optimum $\E[Z_z\mid X]=F_X(z)$. The full family of threshold
probabilities determines the scalar action distribution. Moreover, for
any two scalar actions $a,b$,
\begin{equation}
\int_{\R}
 \bigl(\mathbf1\{a\le z\}-\mathbf1\{b\le z\}\bigr)^2\,\mathrm dz
=|a-b|,
\label{eq:quantile-threshold-distance}
\end{equation}
because the indicators differ precisely between $a$ and $b$.
Lebesgue integration over thresholds therefore respects distance on the
action axis without selecting a finite set of cutoffs. It is a choice of
supervision, not the only possible weighting of threshold events.

\begin{proof}[Proof of Theorem~\ref{thm:quantile-objective}]
Let $q_\theta(X,u)$ be jointly measurable and nondecreasing in $u$.
Measurability makes the probabilities and expectations below well-defined
and is satisfied by the network operations used in our head. For
$U\sim\operatorname{Unif}(0,1)$ independent of $(X,Y)$, the induced
conditional CDF is
\[
G_\theta(z\mid X)
=\int_0^1\mathbf1\{q_\theta(X,u)\le z\}\,\mathrm du.
\]
The variable $U$ indexes quantile levels; its uniform law imposes no
distributional assumption on the action $Y$.
As in the theorem, assume $\E|Y|<\infty$ and
$\E_X\int_0^1|q_\theta(X,u)|\,\mathrm du<\infty$.
The integrated threshold score is the continuous ranked probability score
\citep{gneiting2007strictly}:
\begin{equation}
S(G_\theta;y\mid X)
=\int_{\R}
 \bigl(G_\theta(z\mid X)-\mathbf1\{y\le z\}\bigr)^2\,\mathrm dz,
\qquad
\mathcal L_{\mathrm{CDF}}(\theta)=\E_{X,Y}S(G_\theta;Y\mid X).
\label{eq:quantile-cdf-score}
\end{equation}
To see why first moments suffice, set $R=q_\theta(X,U)$. Jensen's
inequality, Tonelli's theorem, and \eqref{eq:quantile-threshold-distance} give
\[
0\le\mathcal L_{\mathrm{CDF}}(\theta)
\le\E|R-Y|\le\E|R|+\E|Y|<\infty.
\]
Tonelli applies to the nonnegative terms; the first moments make the
pairwise distances below finite and justify subsequent signed integral
exchanges and subtraction.
They are mild sufficient conditions: bounded actions satisfy the label
condition. For the finite head, a fixed continuous network on a compact
input domain has bounded selected outputs. Bounded image inputs and
fixed-length token sequences provide such an effective input domain.
For a general continuous quantile representation, integrability over both
context and quantile level remains an explicit requirement. Our action
normalization is affine and does not clip outliers, so normalization
alone does not establish bounded targets.

\textbf{Step 1: loss equivalence.}\quad
For almost every $X$, the conditional first moments are finite. Fix such
an $X$ and abbreviate its predictive CDF by $G$. Let $R,R'$ be independent
draws from $G$, also independent of the action label conditional on $X$.
The threshold identity \eqref{eq:quantile-threshold-distance} gives
\begin{align*}
\E|R-y|
&=\int_{\R}\left[G(z)+\mathbf1\{y\le z\}
                  -2G(z)\mathbf1\{y\le z\}\right]\,\mathrm dz,\\
\frac12\E|R-R'|
&=\int_{\R}G(z)(1-G(z))\,\mathrm dz.
\end{align*}
Subtracting yields
\begin{equation}
S(G;y)=\E|R-y|-\frac12\E|R-R'|.
\label{eq:quantile-pairwise-score}
\end{equation}
The pairwise term and its coefficient therefore follow from the integrated
threshold loss; they are not added to the regression or CFM objective.

To express this score through quantiles, write $q=q_\theta(X,\cdot)$ and $R=q(U)$ for
$U\sim\operatorname{Unif}(0,1)$. If $U'$ is an independent copy,
nondecreasingness and integrability give
\begin{align*}
\frac12\E|q(U)-q(U')|
&=\int_0^1\int_0^\tau
       \bigl(q(\tau)-q(v)\bigr)\,\mathrm dv\,\mathrm d\tau\\
&=\int_0^1(2\tau-1)q(\tau)\,\mathrm d\tau.
\end{align*}
Here nondecreasingness permits replacing
$|q(\tau)-q(v)|$ by $q(\tau)-q(v)$ when $v<\tau$; integrability permits
changing the order of integration. Flat segments and jumps are allowed,
so this step requires neither strict ordering nor a continuous density.
Using $\rho_\tau(e)=|e|/2+(\tau-1/2)e$ and
$\int_0^1(2\tau-1)y\,\mathrm d\tau=0$, we obtain
\begin{equation}
\begin{aligned}
S(G;y)
&=\int_0^1\left[|y-q(\tau)|-(2\tau-1)q(\tau)\right]\,\mathrm d\tau\\
&=2\int_0^1\rho_\tau\bigl(y-q(\tau)\bigr)\,\mathrm d\tau.
\end{aligned}
\label{eq:quantile-score-identity}
\end{equation}
Taking expectation over $(X,Y)$ proves
\eqref{eq:quantile-integrated-objective}, the first claim of the theorem.
This identity holds for every admissible predictor, whether or not its
model class can represent the true distribution.

\textbf{Step 2: population optimum.}\quad
Since $Z_z$ is binary, the conditional squared-risk decomposition is
\[
\E\!\left[(G_\theta(z\mid X)-Z_z)^2\mid X\right]
=\bigl(G_\theta(z\mid X)-F_X(z)\bigr)^2+F_X(z)(1-F_X(z)).
\]
For a conditional CDF $H$, write
$\mathcal L_{\mathrm{CDF}}(H)=\E_{X,Y}S(H;Y\mid X)$.
Integrating the decomposition gives
\begin{equation}
\mathcal L_{\mathrm{CDF}}(G_\theta)-\mathcal L_{\mathrm{CDF}}(F)
=\E_X\int_{\R}
 \bigl(G_\theta(z\mid X)-F_X(z)\bigr)^2\,\mathrm dz\ge0.
\label{eq:quantile-proper-risk}
\end{equation}
Equality implies $G_\theta(\cdot\mid X)=F_X$ for almost every $X$:
zero integral first gives equality for almost every threshold, and right
continuity extends it to every threshold. Thus the true conditional CDF
is the unique population optimum whenever it is representable. Its
nondecreasing quantile representation equals
$F_X^{-1}(\tau):=\inf\{z:F_X(z)\ge\tau\}$ almost everywhere, with respect
to the law of $X$ and Lebesgue measure on $(0,1)$. By Step~1, this is also
the optimum of the integrated quantile loss, proving the second claim.
The argument allows atoms and requires neither a density nor symmetry.
Representability is used only to attain the lower bound at $F_X$ within
the model class. Without it, \eqref{eq:quantile-proper-risk} still holds:
minimizing population loss seeks the best class approximation in expected
integrated squared CDF distance. This describes an infimum and does not
assert that a finite parameter attains it. Thus representability is an
ideal population benchmark, not a property guaranteed by a finite head.
\end{proof}

\textbf{Consequences for a finite quantile head.}\quad
For selected levels $0<\tau_0<\cdots<\tau_{K-1}<1$ and positive weights
$w_k$ summing to one, the finite training objective is
\begin{equation}
\ell_Q(X,Y)=\sum_{k=0}^{K-1}
w_k\rho_{\tau_k}\bigl(Y-q_k(X)\bigr).
\label{eq:quantile-finite-loss}
\end{equation}
Weight normalization only fixes the loss scale. Assume finite first
moments for $Y$ and the selected outputs $q_k(X)$.
Its conditional subgradient for a single quantile is
\[
\partial_q\E[\rho_\tau(Y-q)\mid X]
=[F_X(q^-)-\tau,\,F_X(q)-\tau].
\]
It contains zero exactly when $F_X(q^-)\le\tau\le F_X(q)$.
Thus free nondecreasing output functions attain the population minimum
at the selected true quantiles, uniquely when each quantile is unique.
A shared model attains this minimum if it can jointly represent those
quantile functions. This is weaker than representing the entire conditional
distribution. If it cannot, its approximation error is the infimum
excess finite pinball risk over its model class. This finite-grid argument
does not require the finite loss to equal the continuous integral.
For example, linear interpolation with constant tails extends an ordered
finite vector to an integrable nondecreasing representation. Theorem~\ref{thm:quantile-objective}
applies to that extension, but its integrated loss need not equal the
finite sum. In particular, our
$K=21$ levels on $[0.025,0.975]$ supervise those selected quantiles;
they do not determine the intervening quantiles or the omitted tails.
Applying this scalar result coordinate-wise identifies only the selected
quantiles of each action marginal; it does not identify their joint dependence.

\textbf{Ordering and bounded output feedback.}\quad
The median-and-gap construction in
Equations~\ref{eq:quantile-gaps}--\ref{eq:ordered-quantiles} enforces
ordering at the selected levels. Separate left and right gaps allow
asymmetric spacing, so the head adds no symmetry requirement.
Away from ties, the per-output derivative is
\[
\partial_{q_k}\ell_Q
=w_k\left[\mathbf1\{Y\le q_k\}-\tau_k\right],
\]
with subgradient interval $w_k[-\tau_k,1-\tau_k]$ at a tie. Hence
$|\partial_{q_k}\ell_Q|\le w_k\max(\tau_k,1-\tau_k)$ for every
subgradient, independently of residual magnitude. This is a bound on
loss feedback at each quantile output, not on gradients through an
arbitrary network parameterization.

\textbf{Repeated quantiles and strict gaps.}\quad
At a fixed input, positive gaps can represent any strictly ordered finite
quantile vector. Repeated true quantiles, which can occur for discrete
or deterministic actions, cannot be represented exactly by strictly
positive gaps. The scale $\delta_0$ in \eqref{eq:quantile-gaps} is not a
positive lower bound: softplus gaps can approach zero.

The resulting approximation can be quantified at the output level. Let
$K=2c+1$ with $c\ge1$, let $q_k^*(X)=F_X^{-1}(\tau_k)$ be the true selected quantiles,
and define $R_K(q)=\E\sum_k w_k\rho_{\tau_k}(Y-q_k(X))$.
The label first moment ensures that these interior true quantiles are
integrable. Our grid has $c=10$.
For any $\varepsilon>0$, set
$q_k^\varepsilon(X)=q_k^*(X)+\varepsilon(k-c)/c$.
Adjacent gaps are at least $\varepsilon/c$, the middle output is unchanged,
and each output moves by at most $\varepsilon$. Since pinball loss is
1-Lipschitz in its predicted value,
\[
0\le R_K(q^\varepsilon)-R_K(q^*)
\le\sum_k w_k\E|q_k^\varepsilon(X)-q_k^*(X)|
\le\varepsilon.
\]
Thus strict ordering alone need not create a positive infimum gap in
population risk. This output-level approximation does not establish
that the shared finite network can represent all such functions or that
training finds them; network approximation, estimation, and optimization
errors remain separate from Theorem~\ref{thm:quantile-objective}.

\textbf{Relation to masked training.}\quad
The random label mask in \eqref{eq:masked-pinball} preserves the finite
objective in expectation over the fresh mask. Conditional on a sample
with $N>0$ valid coordinates, the rule retains a uniform subset of fixed
size $R\ge1$. Each valid coordinate has inclusion probability $R/N$,
so its expected weight in the retained-coordinate average is $1/N$.
This gives the full valid-coordinate average without assuming independent
coordinate masks. When no coordinate is valid, both losses are zero.
This argument concerns random dropping. For a coordinate $j$, write its
effective weight as $\omega_j=V_j/\max(1,N)$. At inputs where
$\E[\omega_j\mid X]>0$, a sufficient condition for the original
conditional quantiles to remain optimal is that $\omega_j$ be determined
by $X$ or conditionally independent of $Y_j$ given $X$. In general,
the target is the conditional distribution reweighted by $\omega_j$.
These distinctions specify how the population result
applies to the implemented supervision; they do not establish a
finite-sample convergence or closed-loop performance guarantee.

\subsection{Proof of Theorem~\ref{thm:quantile-generalization}}
\label{sec:proof-center-generalization}
We compare direct scalar center updates at matched local mean correction
speed. All expectations below are conditional on a fixed input.

\textbf{Calibrated cumulative gaps and the center direction.}\quad
Let $F$ be the scalar conditional action CDF, with unique median $m^*$
and a density $f$ that is positive and twice continuously differentiable
near $m^*$. Write $f_0=f(m^*)>0$. Then $F(m^*)=1/2$, and $F$ has a
locally smooth inverse.
Fix an odd $K\geq3$ and strictly ordered locations
$t_1<\cdots<t_K$ with $t_{K+1-k}=-t_k$ and $t_k\in[-1,1]$.
Set $c=(K+1)/2$ and $\tau_k=1/2+\varepsilon t_k$.
For sufficiently small $\varepsilon>0$, the quantiles
$q_k^*=F^{-1}(\tau_k)$ lie in this neighborhood, are strictly ordered,
and satisfy $q_c^*=m^*$.

These quantiles are compatible with the implemented positive cumulative
gaps. In particular, their adjacent differences define
\[
g_j^{-*}=q_{c-j+1}^*-q_{c-j}^*>0,\qquad
g_j^{+*}=q_{c+j}^*-q_{c+j-1}^*>0,
\qquad 1\leq j\leq (K-1)/2.
\]
Each positive value is representable as
$\kappa\operatorname{softplus}(r)$ for any fixed $\kappa>0$.
There is no equality constraint between the two sides.
Writing $d_k^*=q_k^*-m^*$, hold these calibrated gaps fixed and consider
the scalar center direction
\begin{equation}
q_k(m)=m+d_k^*,\qquad
q_k(m^*+\delta)=q_k^*+\delta.
\label{eq:center-update-fixed-gaps}
\end{equation}
Thus every quantile has direct derivative one with respect to $m$.
For the equal-weight pinball loss and the median-only pinball loss,
the direct center gradients are, away from their breakpoints,
\begin{equation}
\begin{aligned}
G_Q(m,Y)
 &=\frac1K\sum_{k=1}^K
   \bigl[\mathbf{1}\{Y<q_k(m)\}-\tau_k\bigr],\\
G_M(m,Y)
 &=\mathbf{1}\{Y<m\}-\frac12.
\end{aligned}
\label{eq:center-update-direct-gradients}
\end{equation}
The density assumption makes the breakpoints probability-zero events
in the neighborhood used below. This identifies the gradients of the
actual cumulative-gap construction along the stated center direction.

\textbf{Exact gradient variance at calibration.}\quad
Define
$s_k=\mathbf{1}\{Y<q_k^*\}-\tau_k$.
Calibration gives $\E[s_k]=0$ and hence
$\E[G_Q(m^*,Y)]=\E[G_M(m^*,Y)]=0$.
The indicators are nested, so their product has expectation
$\min(\tau_k,\tau_\ell)$. Consequently,
\begin{equation}
\operatorname{Cov}(s_k,s_\ell)
=\min(\tau_k,\tau_\ell)-\tau_k\tau_\ell.
\label{eq:center-update-score-covariance}
\end{equation}
This is the usual composite-quantile score covariance
\citep{zou2008composite}; in particular, the proof does not assume that
the quantile scores from the same label are independent.

Let $B_h=\operatorname{Var}(G_h(m^*,Y))$ for $h\in\{Q,M\}$.
Since $\tau_c=1/2$, $B_M=1/4$. For the joint gradient, substituting the
selected levels in Equation~\ref{eq:center-update-score-covariance} gives
\[
\min(\tau_k,\tau_\ell)-\tau_k\tau_\ell
=\frac14-\frac{\varepsilon}{2}|t_k-t_\ell|
 -\varepsilon^2t_kt_\ell.
\]
Symmetry implies $\sum_k t_k=0$, so summing over all pairs yields the
exact identity
\begin{equation}
B_Q=\frac14-C_t\varepsilon,\qquad
C_t=\frac{1}{2K^2}\sum_{k,\ell=1}^K|t_k-t_\ell|>0.
\label{eq:center-update-variance}
\end{equation}
Strict positivity follows from the distinct levels. The variance
reduction is therefore of first order in $\varepsilon$.

\textbf{Local mean correction.}\quad
Under the fixed-gap perturbation in
Equation~\ref{eq:center-update-fixed-gaps},
\[
\E[G_Q(m^*+\delta,Y)]
=\frac1K\sum_{k=1}^K
  \bigl[F(q_k^*+\delta)-F(q_k^*)\bigr].
\]
It follows that
\begin{equation}
\begin{aligned}
a_Q
&:=\left.\frac{\partial}{\partial m}
       \E[G_Q(m,Y)]\right|_{m=m^*}
 =\frac1K\sum_{k=1}^K f(q_k^*),\\
a_M
&:=\left.\frac{\partial}{\partial m}
       \E[G_M(m,Y)]\right|_{m=m^*}
 =f_0.
\end{aligned}
\label{eq:center-update-slopes}
\end{equation}
Both are positive. Define $r(\tau)=f(F^{-1}(\tau))$ in a neighborhood
of $1/2$. The local positivity and $C^2$ smoothness of $f$ imply that
$r$ is $C^2$ there. Its Taylor expansion gives
\[
r(1/2+\varepsilon t_k)
=f_0+\varepsilon t_k r'(1/2)
 +\frac{\varepsilon^2t_k^2}{2}r''(1/2)
 +o(\varepsilon^2).
\]
Because $K$ and the $t_k$ are fixed, this remainder can be summed over
the finite set of levels. The first-order terms cancel, giving
\begin{equation}
a_Q
=f_0+\frac{\varepsilon^2 r''(1/2)}{2K}
       \sum_{k=1}^Kt_k^2+o(\varepsilon^2)
=f_0+O(\varepsilon^2).
\label{eq:center-update-slope-expansion}
\end{equation}
Thus the local mean correction slope changes only at second order.
This cancellation uses symmetry of the \emph{probability levels},
not symmetry of $f$ or equality of the left and right gaps.

\textbf{Noise after matching correction speeds.}\quad
The normalized noise coefficient is $V_h=B_h/a_h^2$ for $h\in\{Q,M\}$.
Equations~\ref{eq:center-update-variance} and
\ref{eq:center-update-slope-expansion} imply
\[
\frac{V_Q}{V_M}
=\frac{(1/4-C_t\varepsilon)/a_Q^2}{(1/4)/f_0^2}
=(1-4C_t\varepsilon)\frac{f_0^2}{a_Q^2}
=1-4C_t\varepsilon+O(\varepsilon^2).
\]
Since $C_t>0$, there is an $\varepsilon_0>0$ such that this ratio is
strictly less than one whenever $0<\varepsilon<\varepsilon_0$.
For fixed $K$, this inequality also holds uniformly over symmetric grid
shapes. Normalize $\max_k|t_k|=1$, so $\varepsilon$ is the largest
distance of a level from $1/2$. The two extreme rows of the pairwise
sum each contribute $K$, giving $C_t\ge1/K$; bounded $r''$ makes the
$O(\varepsilon^2)$ remainder uniform for $|t_k|\le1$.
Thus sufficiently close symmetric levels give $V_Q/V_M<1$ without
requiring fixed relative spacing.
We now translate this normalized noise comparison into the update
variance stated in the theorem.

For an explicit update interpretation, fix $0<\eta<1$ and set
$\delta_h^+=\delta-\eta G_h(m^*+\delta,Y)/a_h$.
Taylor expansion of the mean gradient in $\delta$ gives, for each
fixed sufficiently small $\varepsilon$,
\[
\E[\delta_h^+]=(1-\eta)\delta+O(\delta^2),
\qquad h\in\{Q,M\}.
\]
The two updates therefore have the same first-order mean correction
speed. At the true center, $\delta=0$, the center increments in the
theorem are $\Delta m_h=-\eta G_h(m^*,Y)/a_h=\delta_h^+$.
They have zero mean, so
\[
\operatorname{Var}(\Delta m_h)=\E[(\delta_h^+)^2]=\eta^2V_h,\qquad
\frac{\operatorname{Var}(\Delta m_Q)}{\operatorname{Var}(\Delta m_M)}
=\frac{V_Q}{V_M}<1.
\]
This proves Equation~\ref{eq:center-update-noise}.
The improvement is consequently not explained solely by a smaller
gradient magnitude or an unmatched learning rate. \hfill$\square$

\textbf{Scope of the conclusion.}\quad
The limit is a shrinking quantile-level spread for fixed finite $K$;
it is neither a sample-size limit nor a statement uniform in the number
of quantiles. Each sufficiently small positive
$\varepsilon$ gives strictly positive representable gaps.
No parametric distribution family or symmetry of the conditional
density is required. Calibration and the fixed-gap center direction
are substantive conditions: the result does not account for inaccurate
gaps or simultaneous updates to gaps and shared network parameters.
The normalization by $a_h$ defines the comparison of local correction
speeds; it is not a change to the optimizer used in our experiments.
Accordingly, the theorem establishes a local center-update mechanism,
not a guarantee of joint-training generalization or task success, and
does not by itself certify the wider quantile grid used in the
experiments.

\clearpage
\raggedbottom
\setlength{\parskip}{3pt}
\setlength{\textfloatsep}{12pt plus 2pt minus 2pt}
\setlength{\floatsep}{8pt plus 2pt minus 2pt}
\section{Experimental details and additional results}
\label{sec:appendix-simulation-results}
\label{sec:appendix-libero-pro}

\textbf{Simulation training.}\quad
We implement our method on $\pi_{0.5}$ with a frozen VLM, training
the action expert, 16 prompt tokens, and quantile heads.
The head uses the 21 levels, positive-gap parameterization, and
$10\%$ random coordinate-label mask specified in Section~\ref{sec:method}.
Simulation training uses AdamW with a global batch size of 128 across
two GPUs.
Prompt tokens use a peak learning rate of $10^{-4}$ and no weight
decay; other trainable parameters use $2.5\times10^{-5}$ and weight
decay $0.01$.
Both learning rates follow cosine decay with 200 warmup steps for LIBERO
and 666 for LIBERO-Plus.
Both models start from the same pretrained checkpoint and train for
one epoch on their respective datasets.


\textbf{Training and transfer settings.}\quad
On LIBERO, we evaluate policies trained on its standard demonstrations.
To assess their robustness, we directly evaluate these policies on
LIBERO-Plus and LIBERO-Pro without further training.
For LIBERO-Plus, we also evaluate policies fine-tuned on its training set,
reporting \emph{zero-shot transfer} and \emph{supervised fine-tuning}
results separately. Quantile Head uses deterministic median decoding
for Tables~\ref{tab:libero-pro-combined}--\ref{tab:libero-component-ablation};
alternative sampling rules are evaluated separately in
Section~\ref{sec:sampling-ablations} and the two-demonstration studies.
At each replanning step, the policy supplies a ten-step action sequence
to the controller, which executes it before obtaining a new observation
and replanning. Evaluation does not update model parameters. An episode
succeeds if the benchmark's task-success condition is met within its
rollout budget; an episode that terminates without success is a failure.

\textbf{LIBERO evaluation and aggregation.}\quad
We evaluate all ten tasks in each of Spatial, Object, Goal, and Long
(LIBERO-10). Our results in Table~\ref{tab:libero-pro-combined} use models
trained with three different random seeds, with each model evaluated for
1,000 episodes per suite.
Within each trained model, a suite success rate averages
its ten task success rates, and Avg. averages the four suite rates with
equal weight. We report the arithmetic mean of the corresponding
success rates over the three training seeds.
The matched head and component comparisons in
Table~\ref{tab:libero-component-ablation} use the same four suites and
task/suite averaging, with 100 evaluation episodes per task and setting.

\textbf{LIBERO-Plus protocol.}\quad
Both zero-shot and fine-tuned evaluations cover all 10,030 benchmark
instances, with one rollout per instance and trained model, using the
benchmark-provided initial states. Spatial, Object, Goal, and Long contain
2,402, 2,518, 2,591, and 2,519 instances, respectively. Their rollout
limits are 280, 280, 300, and 520 environment steps.
We assign instances to the seven official perturbation categories.
Within each model, each category rate pools successful rollouts over all
its instances across the four suites; Total pools all 10,030 instances.
Thus, Total is instance-weighted, rather than an unweighted mean of
the seven category rates or four suite rates.
For Table~\ref{tab:libero-plus-comparison}, we train models with three
different random seeds and report the arithmetic mean of their
corresponding success rates.

\textbf{LIBERO-Pro protocol.}\quad
We evaluate the three seed-trained LIBERO policies on object, position,
semantic, and task perturbations. For each trained model, we evaluate
each task for 100 episodes.
We use the benchmark-provided
initial states; rollout limits are 280 steps for Spatial and Object,
300 for Goal, and 800 for Long.
Within each model, each perturbation rate averages the four suite rates,
and Avg. averages the four perturbation rates with equal weight.
Reported rates then average the corresponding values over the three
training seeds.
Original and Environment are
excluded from Avg. The Original column in
Table~\ref{tab:libero-pro-by-suite} is the unperturbed control from the
separate Pro evaluation, rather than a reuse of the standard LIBERO result.

\textbf{External comparisons.}\quad
Cited baselines in Tables~\ref{tab:libero-pro-combined}
and~\ref{tab:libero-plus-comparison} retain their source-reported
training, evaluation, and aggregation protocols. Their pretraining,
trainable modules, observation histories, and per-suite versus joint
training can differ. These tables compare reported system performance;
Table~\ref{tab:libero-component-ablation} provides the matched comparisons
of action heads and training components. Superscripts in the comparison
tables identify the source of each reused result.

\textbf{Ablation Settings.}\quad
Each setting in Table~\ref{tab:libero-component-ablation} is evaluated for
100 episodes per task on the four LIBERO suites.
Quantile and direct-regression variants use deterministic center decoding.
Avg. reports the mean success rate across the four LIBERO suites.
Average episode time in (a) is measured over all evaluation episodes,
including both successful and failed trials. The relative reduction in
average episode time uses $L_1$ as the reference.
This episode-level quantity includes the effects of rollout length and
termination, and is distinct from the latency of one policy forward pass.
Unspecified components in Table~\ref{tab:libero-component-ablation} follow the full model.
The $L_2$ and $L_1$ rows use direct regression heads with
$e^2$ and $|e|$ losses, respectively;
Median-Only in (b) retains the quantile head and supervises only its median.
Detached Anchor detaches $m$ only in noncentral quantiles, whose losses
still update shared features. This variant additionally multiplies the
gradient with respect to the median output $m$ by $K=21$.
Symmetric Quantiles share left/right offsets.

The component ablations in (c) retain the ordered median-and-positive-gap head.
Pinball Loss supervises $21$ quantiles; disabling it uses only $0.5|e|$
center supervision, normalized over one output. Label Mask randomly drops
$10\%$ of valid coordinate labels, sharing the mask across quantiles.
Prompts are $16$ trainable context tokens. Checkmarks enable components;
blanks disable them. Rows cover all eight combinations of the three listed
components: all off, three single-component, three two-component, and all on.
Without joint supervision, gap branches receive no supervised gradients.
Removing only Pinball Loss matches Median-Only in (b). Task instructions
and validity/padding masks remain active.

\textbf{Flow-matching inference.}\quad
The matched Flow Matching baseline in
Table~\ref{tab:libero-component-ablation}(a) starts from Gaussian noise
and uses ten Euler integration steps from $t=1$ to $t=0$ per replanning
call. These are numerical integration steps; the controller executes
ten environment actions between replanning calls.

\begin{figure}[!t]
  \centering
  \includegraphics[width=0.8\linewidth]{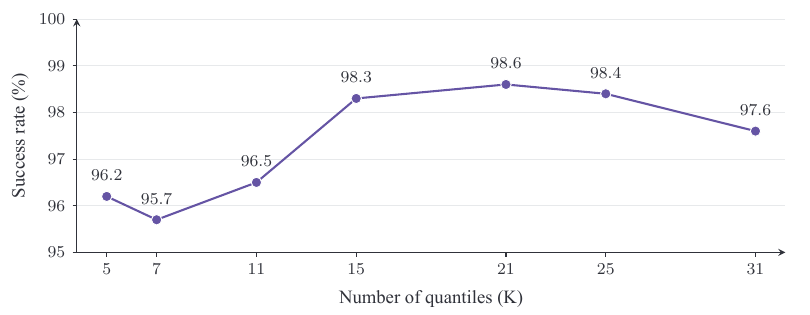}
  \caption{\textbf{Number of quantiles on LIBERO-Long.}
  Success rate as the number of predicted quantiles $K$ varies.
  Rates average three training seeds, with 1,000 evaluation episodes per seed and setting.
  The default $K=21$ attains the highest observed rate of 98.6\%.}
  \label{fig:quantile-count-ablation}
\end{figure}

\textbf{Number of Quantiles.}\quad
We compare seven quantile counts, from $K=5$ to $K=31$, on LIBERO-Long.
For each setting, we train models with three different random seeds,
evaluate each model over 1,000 episodes, and report the arithmetic mean
of the three success rates.
Figure~\ref{fig:quantile-count-ablation} shows higher success rates for
$K=15$--$25$ (98.3--98.6\%) than for $K=5$--$11$ (95.7--96.5\%),
followed by a decrease to 97.6\% at $K=31$.
Our default $K=21$ achieves the best observed success rate while using
fewer quantiles than $K=25$; increasing $K$ beyond 21 provides no further
gain in this exploratory study.
\begin{figure}[!t]
  \centering
  \includegraphics[width=0.8\linewidth]{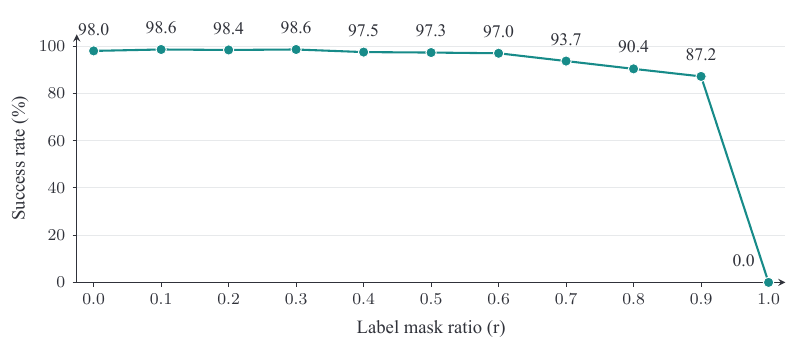}
  \caption{\textbf{Label mask ratio on LIBERO-Long.}
  Success rate as the fraction of masked action labels $r$ increases.
  Rates average three training seeds, with 1,000 evaluation episodes per seed and setting.
  The sweep includes the fully masked setting $r=1.0$.}
  \label{fig:label-mask-ratio-ablation}
\end{figure}

\textbf{Label Mask Ratio.}\quad
We vary $r$ from 0 to 1.0 in steps of 0.1 on LIBERO-Long.
For each setting, we train models with three different random seeds,
evaluate each model over 1,000 episodes, and report the arithmetic mean
of the three success rates.
In Figure~\ref{fig:label-mask-ratio-ablation}, the default $r=0.1$ and
$r=0.3$ both reach 98.6\% success, compared with 98.0\% without masking.
Success declines at higher masking ratios, reaching 87.2\% at $r=0.9$
and 0.0\% at $r=1.0$.
For the fully masked $r=1.0$ setting, we disable the requirement to
retain at least one label in Section~\ref{sec:label-mask} and drop all
valid action labels.

\textbf{Sampling Analysis Settings.}\quad
We study action decoding and temporal sampling structure on LIBERO-10.
Each configuration in Figure~\ref{fig:inference-sampling-ablation} is
evaluated for 100 episodes per task across all ten tasks.
The eleven configurations for $[0.2,0.8]$ and $[0.3,0.7]$, including
the median baseline, follow this budget. We use $T_s=0.35$.
At each replanning call, the controller receives and executes a ten-step
action sequence; the gripper uses its median. Sampled decoders transform
standard-normal latent variables into ranks by
$U=\ell+(1-2\ell)\Phi(T_s Z)$, where $\Phi$ is the standard normal CDF,
using windows $[\ell,1-\ell]=[0.2,0.8]$ and $[0.3,0.7]$.
Let $Q(u)$ interpolate the predicted quantile knots piecewise linearly
for each step and coordinate.
Quantile decoders use $Q(U)$; the Gaussian surrogate below uses the same
ranks and matches the predicted median and $10$--$90\%$ width. For each
predicted step and action coordinate, it decodes
\begin{equation}
 a_G(U)=m+\sigma_G\Phi^{-1}(U),\qquad m=Q(0.5),\qquad
 \sigma_G=\frac{Q(0.9)-Q(0.1)}{2\Phi^{-1}(0.9)}.
 \label{eq:gaussian-surrogate-decoder}
\end{equation}
It retains the same rank window, temperature, temporal rank process,
and median-decoded gripper. Thus the comparison changes the marginal
inverse CDF while matching its center and central interval width.
The median result is shared across panels.
Both panels additionally evaluate sampling windows
$[0.1,0.9]$ and $[0.4,0.6]$.
The two additional Gaussian-surrogate configurations in (a) each use
100 episodes per task, giving 1,000 episodes per configuration.
Both panels reuse the median and all four chunk-shared
quantile results. Chunk-shared ranks vary independently across motion
coordinates; $\rho$ controls the Gaussian latent AR process.

\begin{table}[!t]
  \centering
  \caption{Quantile Head success rates (\%) by suite and perturbation type on LIBERO-Pro.}
  \label{tab:libero-pro-by-suite}
  \vspace{4pt}
  \scalebox{0.8}{%
    \begin{minipage}{\linewidth}
      \centering
  \small
  \renewcommand{\arraystretch}{1.06}
  \setlength{\tabcolsep}{4pt}
  \begin{tabular*}{\linewidth}{@{\extracolsep{\fill}}lcccccc@{}}
    \toprule
    Suite & Original & Object & Position & Semantic & Task & Avg. \\
    \midrule
    Spatial & 99.0 & 100.0 & 49.0 & 98.0 & 47.0 & 73.5 \\
    Object & 100.0 & 95.0 & 26.0 & 99.0 & 10.0 & 57.5 \\
    Goal & 95.0 & 81.0 & 34.0 & 97.0 & 27.0 & 59.8 \\
    Long & 98.0 & 67.0 & 16.0 & 97.0 & 20.0 & 50.0 \\
    Mean & 98.0 & 85.8 & 31.3 & 97.8 & 26.0 & 60.2 \\
    \bottomrule
  \end{tabular*}
    \end{minipage}%
  }
\end{table}

\begin{table}[!t]
  \centering
  \caption{Quantile Head success rates (\%) on LIBERO-Plus by suite and perturbation.\newline
  Overall uses instance-weighted aggregation across the four suites.}
  \label{tab:libero-plus-by-suite}
  \vspace{4pt}
  \scalebox{0.8}{%
    \begin{minipage}{\linewidth}
      \centering
      \small
      \renewcommand{\arraystretch}{1.06}
      \setlength{\tabcolsep}{3pt}
      \begin{tabular*}{\linewidth}{@{\extracolsep{\fill}}lcccccccc@{}}
        \multicolumn{9}{@{}l}{\textbf{(a) Zero-shot transfer}} \\
        \toprule
        Suite & Camera & Robot & Lang. & Light & Bkg. & Noise & Layout & Total \\
        \midrule
        Spatial & 75.8 & 89.4 & 97.2 & 100.0 & 99.6 & 98.6 & 98.4 & 93.7 \\
        Object & 87.1 & 73.6 & 93.5 & 100.0 & 100.0 & 99.1 & 92.3 & 91.5 \\
        Goal & 78.2 & 79.0 & 68.5 & 91.8 & 95.4 & 94.7 & 73.2 & 81.7 \\
        Long & 46.8 & 81.9 & 92.7 & 92.0 & 94.5 & 87.8 & 88.8 & 82.1 \\
        Overall & 71.6 & 80.7 & 87.6 & 96.1 & 97.2 & 94.8 & 87.8 & 87.1 \\
        \bottomrule
      \end{tabular*}
      \vspace{7pt}
      \begin{tabular*}{\linewidth}{@{\extracolsep{\fill}}lcccccccc@{}}
        \multicolumn{9}{@{}l}{\textbf{(b) Supervised fine-tuning}} \\
        \toprule
        Suite & Camera & Robot & Lang. & Light & Bkg. & Noise & Layout & Total \\
        \midrule
        Spatial & 97.9 & 88.3 & 95.9 & 100.0 & 100.0 & 96.6 & 96.9 & 96.3 \\
        Object & 97.5 & 68.8 & 91.0 & 100.0 & 100.0 & 99.5 & 91.1 & 91.9 \\
        Goal & 79.9 & 77.3 & 65.6 & 97.5 & 92.2 & 94.5 & 72.0 & 81.3 \\
        Long & 80.4 & 74.8 & 94.0 & 91.6 & 97.6 & 86.6 & 91.7 & 87.3 \\
        Overall & 88.6 & 77.0 & 86.2 & 97.4 & 97.3 & 94.1 & 87.3 & 89.1 \\
        \bottomrule
      \end{tabular*}
    \end{minipage}%
  }
\end{table}

\begin{figure}[t]
  \centering
  \includegraphics[width=\linewidth]{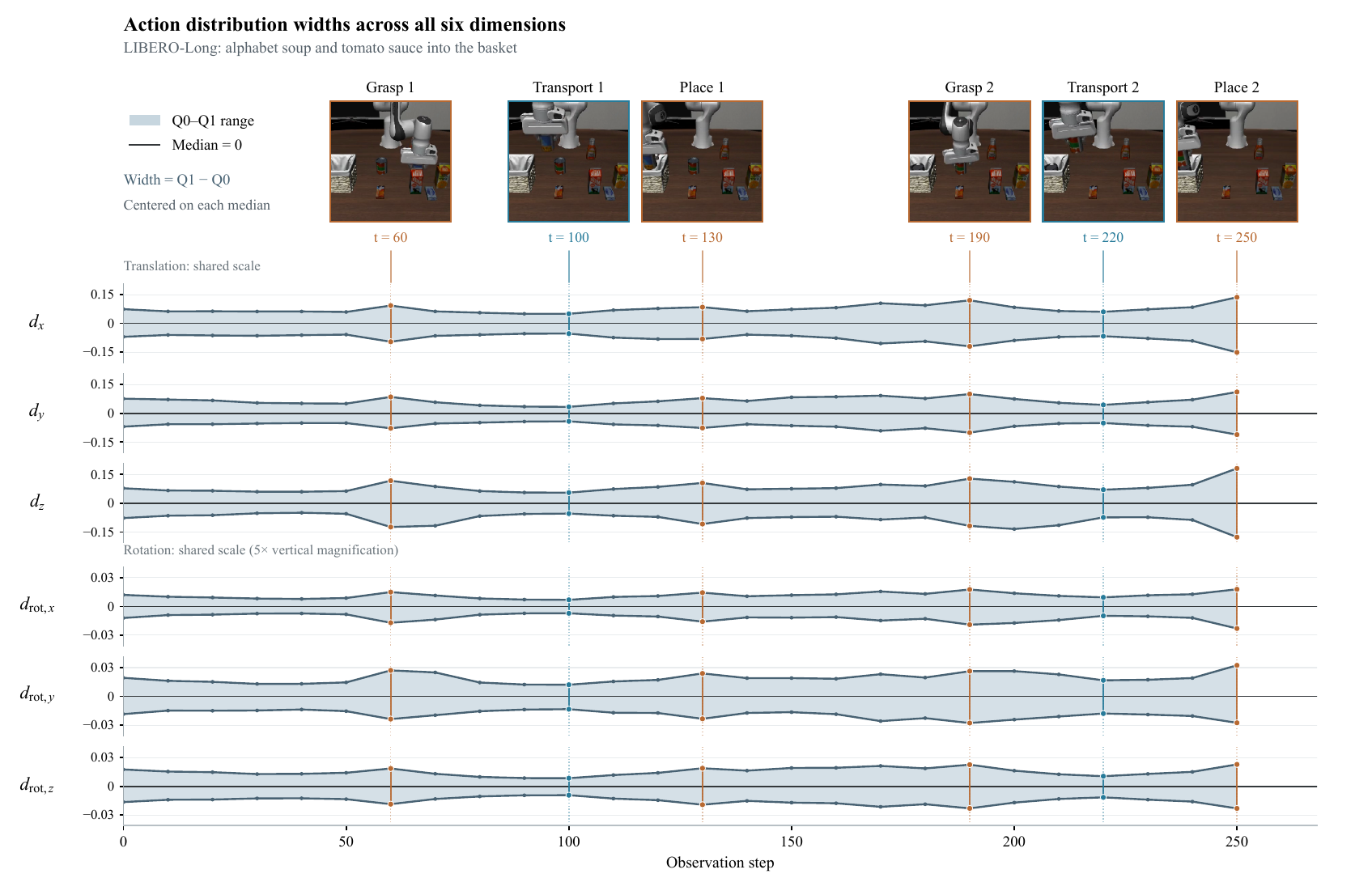}
  \caption{\textbf{Predicted action spread across six motion dimensions.}
  Median-centered ranges in environment-command units from the rollout in
  Figure~\ref{fig:quantile-mean-range}; rotation panels use fivefold vertical magnification.}
  \label{fig:quantile-ranges-all-dimensions}
\end{figure}

\subsection{Per-dimension action distributions}
\label{sec:appendix-action-distributions}

Let $\widetilde Q_{\tau,d}(t)$ be the normalized quantile for coordinate $d$
at observation $t$ in Figure~\ref{fig:quantile-mean-range}. Averaging its
median-centered endpoints over six coordinates, excluding the gripper, gives:
\begin{align*}
  L(t) &= \frac{1}{6}\sum_{d=1}^{6}
  \left(\widetilde Q_{0,d}(t)-\widetilde Q_{0.5,d}(t)\right),
  & U(t) &= \frac{1}{6}\sum_{d=1}^{6}
  \left(\widetilde Q_{1,d}(t)-\widetilde Q_{0.5,d}(t)\right), \\
  W(t) &= U(t)-L(t)
  = \frac{1}{6}\sum_{d=1}^{6}
  \left(\widetilde Q_{1,d}(t)-\widetilde Q_{0,d}(t)\right).
\end{align*}
This is a descriptive mean of marginal ranges, not the quantile range of
an averaged action. All outputs use the first prediction step ($h=0$);
$Q_0$ and $Q_1$ are obtained by linear extrapolation of the outer quantile knots.

Figure~\ref{fig:quantile-ranges-all-dimensions} shows the individual ranges
in environment-command units. The three translation panels share one
vertical scale; the three rotation panels share another, magnified fivefold.
Both figures use 26 replanning observations from one successful LIBERO-Long
episode that places alphabet soup and tomato sauce in a basket. Video frames
are matched to observations by their recorded indices. ``Place'' marks
lowering into the basket and does not imply that release is complete.

\begin{figure}[!t]
  \centering
  \includegraphics[width=\linewidth]{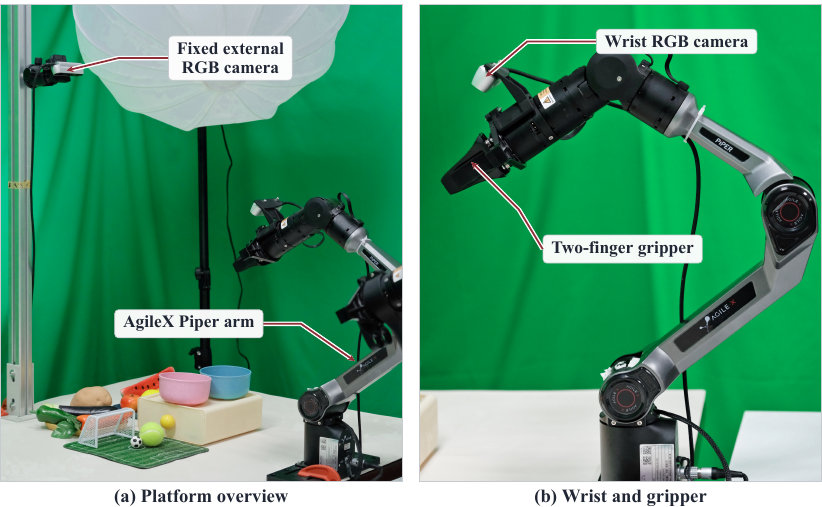}
  \caption{\textbf{Real-robot platform.}
  The experimental setup and arm assembly, with the AgileX Piper arm,
  two-finger gripper, wrist RGB camera, and fixed external RGB camera labeled.}
  \label{fig:real-robot-platform}
\end{figure}

\subsection{Real-robot experiments}
\label{sec:appendix-real-robot}

\textbf{Platform.}\quad
The platform has an AgileX Piper arm, a two-finger gripper, a wrist RGB
camera, and a fixed external RGB camera
(Figure~\ref{fig:real-robot-platform}).

\textbf{Tasks.}\quad
We evaluate two tasks: placing an apple on a yellow plate and removing a
cuboid from a blue plate.
Figure~\ref{fig:real-robot-tasks} illustrates their action sequences using
frames from the recorded videos.

\begin{figure}[!t]
  \centering
  \resizebox{\linewidth}{!}{
\begingroup%
\definecolor{rrMotion}{HTML}{305FA0}%
\definecolor{rrTaskInk}{HTML}{283443}%
\definecolor{rrTaskBorder}{HTML}{CBD1D7}%
\begin{tikzpicture}[x=1cm,y=1cm,line cap=round,line join=round,
  every node/.style={font=\rmfamily\fontsize{8}{9}\selectfont,text=rrTaskInk,inner sep=1.5pt},
  rr motion/.style={draw=rrMotion,line width=1.2pt,
    preaction={draw=white,line width=2.5pt,-},
    -{Latex[length=2.1mm,width=1.6mm]}},
  rr ring/.style={draw=rrMotion,line width=.85pt,
    preaction={draw=white,line width=1.9pt}}]
\node[anchor=west,font=\rmfamily\bfseries\fontsize{9}{10.5}\selectfont] at (0,7.39273) {(a) Apple: table $\rightarrow$ yellow plate};
\begin{scope}[shift={(0.00000,3.87636)}]
\begin{scope}
\clip (0,0) rectangle (4.45000,3.23636);
\node[anchor=south west,inner sep=0pt] at (-2.99364,-0.32364)
  {\includegraphics[width=9.19127cm]{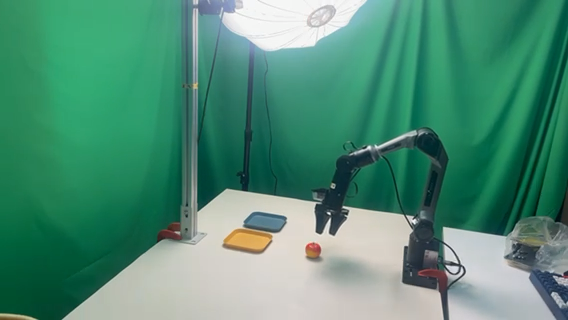}};
\end{scope}
\draw[rrTaskBorder,line width=.45pt] (0,0) rectangle (4.45000,3.23636);
\draw[rr ring] (2.07127,0.80909) circle (0.19418);
\node[anchor=north west,fill=white,rounded corners=1.5pt,inner sep=2pt] at (.10,3.13636)
  {\textcolor{rrMotion}{\textbf{1}}\quad Approach};
\node[anchor=north east,fill=white,rounded corners=1.5pt,inner sep=2pt] at (4.35000,3.13636) {3.60\,s};
\end{scope}
\begin{scope}[shift={(4.75000,3.87636)}]
\begin{scope}
\clip (0,0) rectangle (4.45000,3.23636);
\node[anchor=south west,inner sep=0pt] at (-2.99364,-0.32364)
  {\includegraphics[width=9.19127cm]{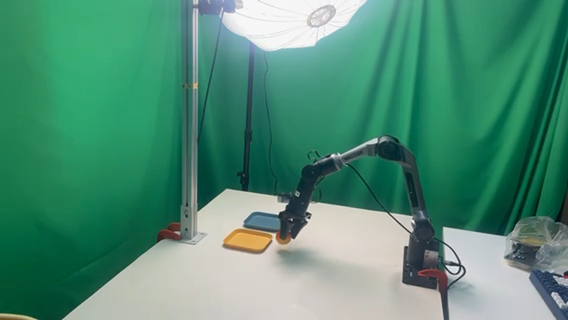}};
\end{scope}
\draw[rrTaskBorder,line width=.45pt] (0,0) rectangle (4.45000,3.23636);
\draw[rr motion] (2.10364,0.48545) .. controls (1.78000,0.14564) and (1.14891,0.21036) .. (0.98709,0.72818);
\node[anchor=north west,fill=white,rounded corners=1.5pt,inner sep=2pt] at (.10,3.13636)
  {\textcolor{rrMotion}{\textbf{2}}\quad Transfer};
\node[anchor=north east,fill=white,rounded corners=1.5pt,inner sep=2pt] at (4.35000,3.13636) {5.20\,s};
\end{scope}
\begin{scope}[shift={(9.50000,3.87636)}]
\begin{scope}
\clip (0,0) rectangle (4.45000,3.23636);
\node[anchor=south west,inner sep=0pt] at (-2.99364,-0.32364)
  {\includegraphics[width=9.19127cm]{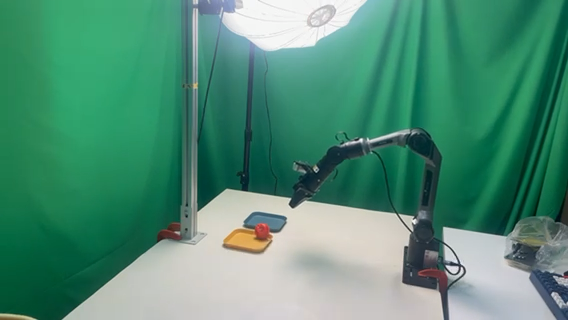}};
\end{scope}
\draw[rrTaskBorder,line width=.45pt] (0,0) rectangle (4.45000,3.23636);
\draw[rr ring] (1.24600,1.13273) circle (0.19418);
\node[anchor=north west,fill=white,rounded corners=1.5pt,inner sep=2pt] at (.10,3.13636)
  {\textcolor{rrMotion}{\textbf{3}}\quad Complete};
\node[anchor=north east,fill=white,rounded corners=1.5pt,inner sep=2pt] at (4.35000,3.13636) {8.10\,s};
\end{scope}
\node[anchor=west,font=\rmfamily\bfseries\fontsize{9}{10.5}\selectfont] at (0,3.51636) {(b) Cuboid: blue plate $\rightarrow$ table};
\begin{scope}[shift={(0.00000,0.00000)}]
\begin{scope}
\clip (0,0) rectangle (4.45000,3.23636);
\node[anchor=south west,inner sep=0pt] at (-2.99364,-0.32364)
  {\includegraphics[width=9.19127cm]{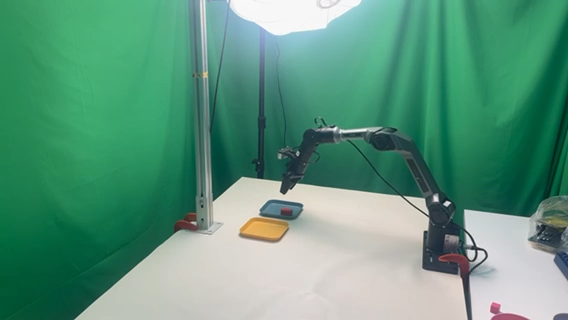}};
\end{scope}
\draw[rrTaskBorder,line width=.45pt] (0,0) rectangle (4.45000,3.23636);
\draw[rr ring] (1.63436,1.45636) circle (0.14564);
\node[anchor=north west,fill=white,rounded corners=1.5pt,inner sep=2pt] at (.10,3.13636)
  {\textcolor{rrMotion}{\textbf{1}}\quad Approach};
\node[anchor=north east,fill=white,rounded corners=1.5pt,inner sep=2pt] at (4.35000,3.13636) {3.27\,s};
\end{scope}
\begin{scope}[shift={(4.75000,0.00000)}]
\begin{scope}
\clip (0,0) rectangle (4.45000,3.23636);
\node[anchor=south west,inner sep=0pt] at (-2.99364,-0.32364)
  {\includegraphics[width=9.19127cm]{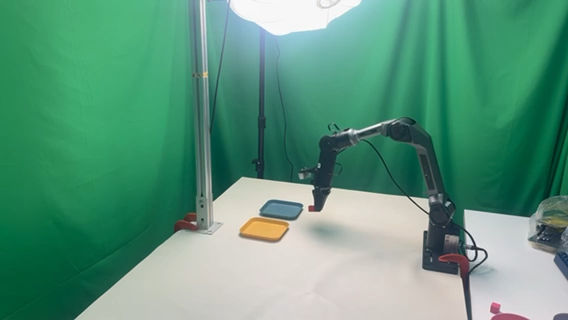}};
\end{scope}
\draw[rrTaskBorder,line width=.45pt] (0,0) rectangle (4.45000,3.23636);
\draw[rr motion] (1.50491,1.11655) .. controls (1.53727,0.71200) and (1.92564,0.48545) .. (2.28164,0.71200);
\node[anchor=north west,fill=white,rounded corners=1.5pt,inner sep=2pt] at (.10,3.13636)
  {\textcolor{rrMotion}{\textbf{2}}\quad Transfer};
\node[anchor=north east,fill=white,rounded corners=1.5pt,inner sep=2pt] at (4.35000,3.13636) {5.20\,s};
\end{scope}
\begin{scope}[shift={(9.50000,0.00000)}]
\begin{scope}
\clip (0,0) rectangle (4.45000,3.23636);
\node[anchor=south west,inner sep=0pt] at (-2.99364,-0.32364)
  {\includegraphics[width=9.19127cm]{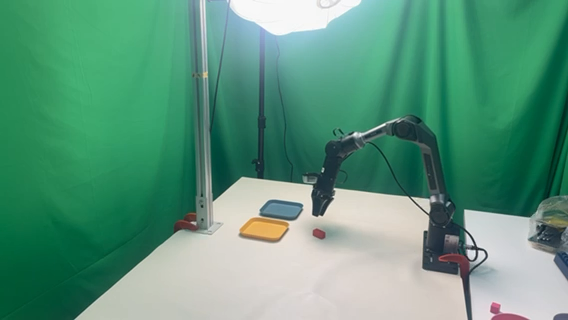}};
\end{scope}
\draw[rrTaskBorder,line width=.45pt] (0,0) rectangle (4.45000,3.23636);
\draw[rr ring] (2.18455,1.06800) circle (0.14564);
\node[anchor=north west,fill=white,rounded corners=1.5pt,inner sep=2pt] at (.10,3.13636)
  {\textcolor{rrMotion}{\textbf{3}}\quad Complete};
\node[anchor=north east,fill=white,rounded corners=1.5pt,inner sep=2pt] at (4.35000,3.13636) {6.00\,s};
\end{scope}
\end{tikzpicture}%
\endgroup%
}
  \caption{\textbf{Two real-robot tasks.}
  Frames from the task videos show apple placement and cuboid removal.
  Numbered keyframes indicate temporal order, with source-video times in seconds.
  Blue circles mark the manipulated objects, and arrows schematically indicate
  the direction of motion.}
  \label{fig:real-robot-tasks}
\end{figure}

\textbf{Training.}\quad
Each task has 100 demonstration trajectories. We train using two NVIDIA
A8000 GPUs for 6,000 optimizer updates per training run. The compared methods are Quantile Head and
$\pi_{0.5}$. Both methods use the same demonstration data, optimization
settings, and training budget. The $\pi_{0.5}$ baseline is fully
fine-tuned, including its VLM. Quantile Head keeps the VLM frozen and
trains the action expert, prompts, and output heads, as described in
Section~\ref{sec:action-features}. Quantile Head uses deterministic
median decoding at deployment.
At each replanning call, the policy supplies a ten-step action sequence
to the controller, which executes it before obtaining a new observation
and replanning.

\textbf{Evaluation conditions.}\quad
For each task, both methods are evaluated over the same ranges of initial
robot and scene configurations, using the same task-specific success
criteria and timeout limits. A trial is successful when the task goal is
achieved within its time limit; a trial that reaches the timeout without
meeting the success criterion is counted as a failure.

\textbf{Evaluation budget and aggregation.}\quad
For each method and task, we evaluate each trained model with a budget
of 100 trials.
We compute each model's task success rate as the percentage of successful
trials.
Table~\ref{tab:real-robot-results} reports per-task success rates; its
Average row gives the unweighted mean over the two tasks. The reported
$\Delta$ is the difference between the two methods in percentage points.
The video frames in Figure~\ref{fig:real-robot-tasks} illustrate task
execution and are not additional evaluation trials.

\subsection{Controlled two-demonstration experiment}
\label{sec:two-demo-distributions}

\textbf{Task and demonstration construction.}\quad
We use task 0 of LIBERO-Object with the instruction
``pick up the alphabet soup and place it in the basket.''
A scripted geometric controller generates two successful, complete
demonstrations, each containing 315 actions. Starting from the same
simulator state, the demonstrations take opposite 11\,cm lateral detours,
perpendicular to the initial approach direction, before aligning with
the object and completing grasping, transport, and placement.
Both use initial-state index 0, environment seed 20260922, and ten
settling steps. The initial external-camera image, wrist-camera image,
robot state, and simulator state are verified to be identical.
The instruction contains no left/right cue, so the shared initial
observation has two different demonstrated actions.
Each observation contains two $224\times224$ RGB images and an
eight-dimensional robot state; each action contains six motion commands
and one gripper command.
Figure~\ref{fig:two-demo-data-description} shows the recorded approach
paths, the first ten action commands, and representative observations
from these two training trajectories, with top-down renderings of the
saved states at maximum route separation.

\begin{figure}[!t]
  \centering
\begingroup
\definecolor{demoLeft}{HTML}{0072B2}
\definecolor{demoRight}{HTML}{D55E00}
\captionsetup[subfigure]{font=footnotesize,justification=raggedright,
  singlelinecheck=false,skip=3pt}
\setlength{\fboxsep}{0pt}
\setlength{\fboxrule}{0.6pt}
\centering

\begin{subfigure}[t]{0.48\linewidth}
  \centering
  \includegraphics[width=\linewidth]{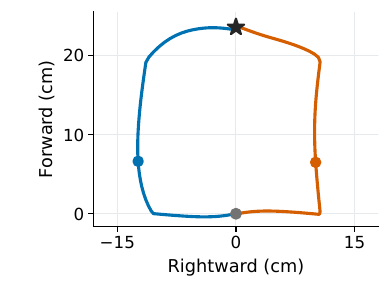}
  \caption{Measured approach paths.}
  \label{fig:two-demo-data-paths}
\end{subfigure}\hfill
\begin{subfigure}[t]{0.48\linewidth}
  \centering
  \includegraphics[width=\linewidth]{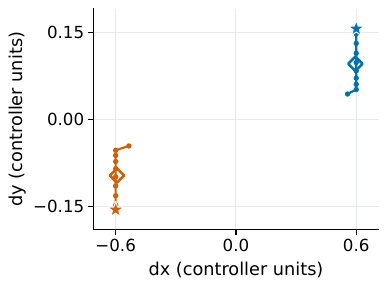}
  \caption{First ten action commands.}
  \label{fig:two-demo-data-actions}
\end{subfigure}

\medskip
\begin{subfigure}[t]{0.235\linewidth}
  \centering
  \fcolorbox{demoLeft}{white}{\includegraphics[width=\dimexpr\linewidth-2\fboxrule\relax]{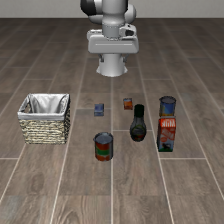}}
  \caption{\textcolor{demoLeft}{Left}: shared start.\newline $t=0$.}
  \label{fig:two-demo-data-left-start}
\end{subfigure}\hfill
\begin{subfigure}[t]{0.235\linewidth}
  \centering
  \fcolorbox{demoLeft}{white}{%
  \begin{tikzpicture}
    \node[anchor=south west,inner sep=0] (view) at (0,0)
      {\includegraphics[width=\dimexpr\linewidth-2\fboxrule\relax]{figures/two_demo_data/overhead/left.png}};
    \begin{scope}[x={(view.south east)},y={(view.north west)}]
      \clip (0,0) rectangle (1,1);
      \csname arxivOverlay@two_demo_data@left\endcsname
    \end{scope}
  \end{tikzpicture}%
}
  \caption{\textcolor{demoLeft}{Left}: top view.\newline $t=29$; $12.37$\,cm left.}
  \label{fig:two-demo-data-left-separated}
\end{subfigure}\hfill
\begin{subfigure}[t]{0.235\linewidth}
  \centering
  \fcolorbox{demoLeft}{white}{\includegraphics[width=\dimexpr\linewidth-2\fboxrule\relax]{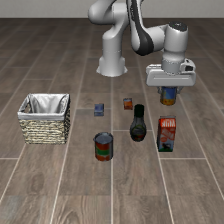}}
  \caption{\textcolor{demoLeft}{Left}: object lifted.\newline $t=120$.}
  \label{fig:two-demo-data-left-lifted}
\end{subfigure}\hfill
\begin{subfigure}[t]{0.235\linewidth}
  \centering
  \fcolorbox{demoLeft}{white}{\includegraphics[width=\dimexpr\linewidth-2\fboxrule\relax]{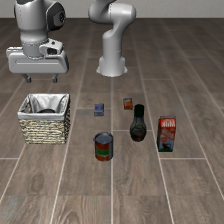}}
  \caption{\textcolor{demoLeft}{Left}: final success.\newline $t=315$.}
  \label{fig:two-demo-data-left-final}
\end{subfigure}

\medskip
\begin{subfigure}[t]{0.235\linewidth}
  \centering
  \fcolorbox{demoRight}{white}{\includegraphics[width=\dimexpr\linewidth-2\fboxrule\relax]{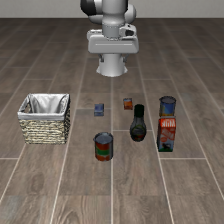}}
  \caption{\textcolor{demoRight}{Right}: shared start.\newline $t=0$.}
  \label{fig:two-demo-data-right-start}
\end{subfigure}\hfill
\begin{subfigure}[t]{0.235\linewidth}
  \centering
  \fcolorbox{demoRight}{white}{%
  \begin{tikzpicture}
    \node[anchor=south west,inner sep=0] (view) at (0,0)
      {\includegraphics[width=\dimexpr\linewidth-2\fboxrule\relax]{figures/two_demo_data/overhead/right.png}};
    \begin{scope}[x={(view.south east)},y={(view.north west)}]
      \clip (0,0) rectangle (1,1);
      \csname arxivOverlay@two_demo_data@right\endcsname
    \end{scope}
  \end{tikzpicture}%
}
  \caption{\textcolor{demoRight}{Right}: top view.\newline $t=29$; $10.07$\,cm right.}
  \label{fig:two-demo-data-right-separated}
\end{subfigure}\hfill
\begin{subfigure}[t]{0.235\linewidth}
  \centering
  \fcolorbox{demoRight}{white}{\includegraphics[width=\dimexpr\linewidth-2\fboxrule\relax]{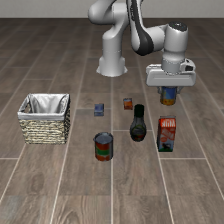}}
  \caption{\textcolor{demoRight}{Right}: object lifted.\newline $t=120$.}
  \label{fig:two-demo-data-right-lifted}
\end{subfigure}\hfill
\begin{subfigure}[t]{0.235\linewidth}
  \centering
  \fcolorbox{demoRight}{white}{\includegraphics[width=\dimexpr\linewidth-2\fboxrule\relax]{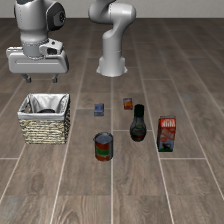}}
  \caption{\textcolor{demoRight}{Right}: final success.\newline $t=315$.}
  \label{fig:two-demo-data-right-final}
\end{subfigure}

\caption{\textbf{Two training demonstrations with the same initial observation.}
Blue and orange denote left and right detours under the shared instruction
``pick up the alphabet soup and place it in the basket.''
(a) Measured end-effector paths for $t=0,\ldots,65$, resolved into rightward
and forward displacement relative to the initial approach direction.
The gray circle marks the shared start, the black star the initial target
position, and colored dots the largest simultaneous XY separation
(22.45\,cm at $t=29$).
(b) Raw controller commands: stars mark the first actions and hollow
diamonds the ten-step means.
(d), (h) Top-down renderings of the recorded $t=29$ states, with screen
right aligned to the rightward axis in (a). Dashed lines mark the shared
start--target centerline; colored rings mark the projected end effector
and arrows show its lateral offset in the restored state.
All other frames are original training-camera
observations. Here $t$ counts executed actions.
Both 315-step demonstrations end in task success.}
\label{fig:two-demo-data-description}
\endgroup
\end{figure}

\textbf{Training samples and shared initialization.}\quad
Each training example pairs a pre-action observation with the next ten
actions from the same demonstration, giving 630 available start indices.
We mask steps beyond the end of
a trajectory and unused action dimensions. Each batch contains 16
examples: two copies of the initial observation from each demonstration
and six uniformly sampled start indices from each demonstration.
Thus, 25\% of the batch slots are reserved for the ambiguous start;
the remaining 75\% samples both trajectories equally, with replacement.
All four variants use the same batch sequence, instruction, and native
state/action preprocessing from the same pretrained $\pi_{0.5}$ checkpoint.
They also share the initialization of the action expert, the freshly
initialized linear output projection, and 16 trainable prompt tokens.
The VLM is frozen; each variant trains its action expert and other
action-side modules, including the prompt tokens and output head.
Quantile Head additionally trains its
gap projection, so the trainable parameter counts are not identical.

\textbf{Objectives and optimization.}\quad
Quantile Head uses ordered quantiles at 21 equally spaced levels from
0.025 to 0.975 and averages their pinball losses.
The $L_1$ and $L_2$ heads use absolute and squared action errors,
respectively. For Flow Matching, we draw standard-normal noise
$\epsilon$, set $x_t=(1-t)a+t\epsilon$, and supervise the velocity
$\epsilon-a$ with squared error, where
$t=0.001+0.999b$ and $b\sim\operatorname{Beta}(1.5,1)$.
All losses average over valid coordinates within each example and then
over the batch. This controlled experiment uses padding masks only,
without random coordinate-label dropping.
We train each variant for 6,000 updates
on one NVIDIA RTX 5090 GPU in float32, with AMP and TF32 disabled.
AdamW uses $(\beta_1,\beta_2)=(0.9,0.95)$, $\epsilon=10^{-8}$,
and gradient clipping at norm 1.
The peak learning rate is $10^{-4}$ for prompts and
$2.5\times10^{-5}$ for other trainable parameters, with weight decay
0 and 0.01, respectively. Both rates use 200 linear warmup updates
followed by cosine decay toward 10\% of their peak values.
Frozen VLM prefixes are cached during training; cached and complete
forward passes are checked for output and gradient agreement.
The frozen parameters are also verified to remain unchanged.

\textbf{Fixed-observation distribution sampling.}\quad
For one illustrated training run, we compare 1,000 action samples per
head at the shared initial observation. Each sample contains ten prediction steps;
Figure~\ref{fig:four-head-six-dimensions} shows the first step only.
These 1,000 draws are distinct from the 100 closed-loop trials per decoder
visualized in Figure~\ref{fig:four-head-dx-distribution}.
Quantile Head draws one $z\sim\mathcal{N}(0,1)$ per sample and uses
$\tau=\Phi(z)$ for all six motion coordinates and all ten prediction
steps, with the gripper at its median. Here, $\Phi$ denotes the standard
normal CDF. Actions are decoded from the exported quantile function
using piecewise-linear interpolation and linear tail extrapolation
beyond its $[0.025,0.975]$ knots.
This fixed-observation decoding does not require a new network forward
pass for each noise draw. Flow Matching uses ten integration steps;
its samples and the deterministic $L_1$/$L_2$ outputs use the same
observation and instruction. No noise is added after decoding.
These draws characterize conditional output distributions, rather than
variation across training runs or closed-loop episodes.
The shared quantile rank specifies a coupling across coordinates;
the marginal plots alone do not establish learned joint dependence.

\begin{figure}[!t]
  \centering
  \includegraphics[width=\linewidth]{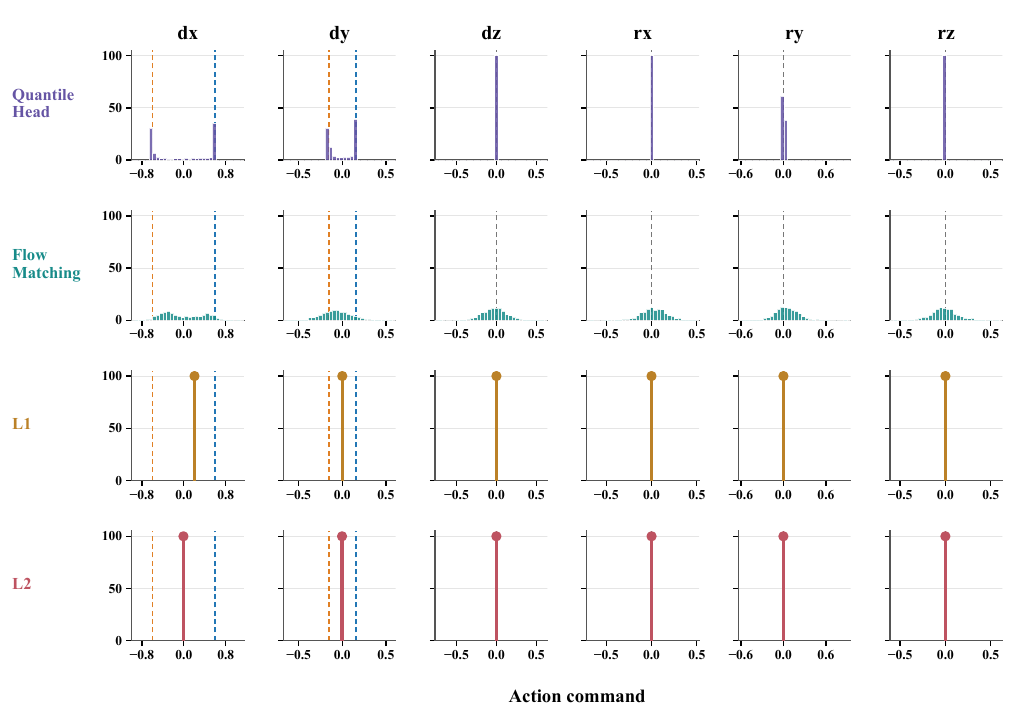}
  \caption{\textbf{Action distributions learned across six motion dimensions.}
  Rows compare the four heads and columns show first-step action commands
  from the no-obstacle setting in Figure~\ref{fig:four-head-dx-distribution}.
  Every panel uses 1,000 samples from one illustrated training run; bins and axes are shared across heads
  within each column. The vertical axis is probability mass per bin
  (\%); stems denote point masses. Blue and orange dashed lines mark
  the two demonstrations, with a gray line where their values coincide.
  Only $d_x$ and $d_y$ have distinct demonstration targets at this step;
  the other four targets are zero. A narrow sampled distribution can
  occupy a single bin and need not be a point mass.}
  \label{fig:four-head-six-dimensions}
\end{figure}

\textbf{Closed-loop sampling settings.}\quad
Using the same trained checkpoints, we initially evaluate seven settings:
three stochastic Quantile decoders, Quantile median, Flow Matching,
$L_1$, and $L_2$.
For each stochastic Quantile decoder, the latent for motion coordinate
$d$ in replanning chunk $k$ follows
\begin{equation}
  z_{0,d}\sim\mathcal{N}(0,1),\qquad
  z_{k,d}=0.9z_{k-1,d}+\sqrt{1-0.9^2}\,\epsilon_{k,d},
  \qquad \epsilon_{k,d}\sim\mathcal{N}(0,1).
  \label{eq:two-demo-temporal-noise}
\end{equation}
Innovations are independent across coordinates and episodes.
The native decoder uses $\tau_{k,d}=0.1+0.8\Phi(0.35z_{k,d})$;
the independent-rank decoder uses $\tau_{k,d}=\Phi(z_{k,d})$;
the shared-rank decoder uses $\tau_{k,d}=\Phi(z_{k,1})$
for all six motion coordinates.
Each rank is held constant over the ten prediction steps in its chunk,
and the gripper always uses rank 0.5.
The three stochastic decoders therefore share the same temporal
correlation; ``independent'' refers to ranks across coordinates.
All use the quantile interpolation and tail rule described above.
Median decoding sets every rank to 0.5. Flow~Matching draws fresh
Gaussian noise at each replan and uses ten Euler integration steps.
The $L_1$ and $L_2$ heads return deterministic commands.

\begin{table}[!t]
  \centering
  \caption{\textbf{Closed-loop success rates in the two-demonstration experiment.}
  Evaluation uses the same training initial state.
  These rates measure completion at this fixed state.}
  \label{tab:two-demo-closed-loop}
  \small
  \begin{tabular}{lr}
    \toprule
    Setting & Success rate (\%) \\
    \midrule
    Quantile: native ranks & 94.0 \\
    Quantile: independent ranks & 87.0 \\
    Quantile: shared ranks & 92.0 \\
    Quantile: median & \textbf{100.0} \\
    Flow Matching (10 integration steps) & 3.0 \\
    $L_1$ & \textbf{100.0} \\
    $L_2$ & \textbf{100.0} \\
    \bottomrule
  \end{tabular}
\end{table}

\begin{figure}[!t]
  \centering
  \captionsetup[subfigure]{justification=raggedright,singlelinecheck=false}
  \begin{subfigure}[t]{0.31\linewidth}
    \centering
    \includegraphics[width=\linewidth]{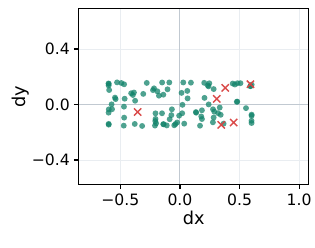}
    \caption{Quantile: native ranks.\newline 94\% success.}
    \label{fig:two-demo-xy-native}
  \end{subfigure}\hfill
  \begin{subfigure}[t]{0.31\linewidth}
    \centering
    \includegraphics[width=\linewidth]{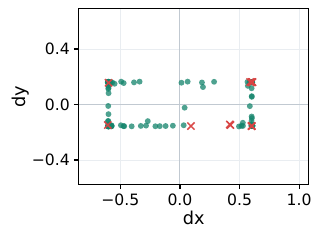}
    \caption{Quantile: independent ranks.\newline 87\% success.}
    \label{fig:two-demo-xy-independent}
  \end{subfigure}\hfill
  \begin{subfigure}[t]{0.31\linewidth}
    \centering
    \includegraphics[width=\linewidth]{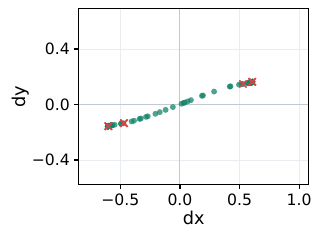}
    \caption{Quantile: shared ranks.\newline 92\% success.}
    \label{fig:two-demo-xy-shared}
  \end{subfigure}

  \medskip
  \begin{subfigure}[t]{0.31\linewidth}
    \centering
    \includegraphics[width=\linewidth]{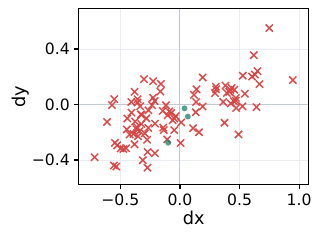}
    \caption{Flow Matching.\newline 3\% success.}
    \label{fig:two-demo-xy-fm}
  \end{subfigure}\hfill
  \begin{subfigure}[t]{0.31\linewidth}
    \centering
    \includegraphics[width=\linewidth]{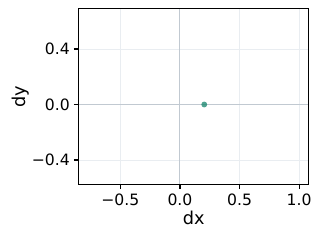}
    \caption{$L_1$.\newline 100\% success.}
    \label{fig:two-demo-xy-l1}
  \end{subfigure}\hfill
  \begin{subfigure}[t]{0.31\linewidth}
    \centering
    \includegraphics[width=\linewidth]{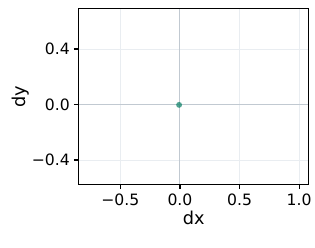}
    \caption{$L_2$.\newline 100\% success.}
    \label{fig:two-demo-xy-l2}
  \end{subfigure}

  \medskip
  \begin{subfigure}[t]{0.31\linewidth}
    \centering
    \includegraphics[width=\linewidth]{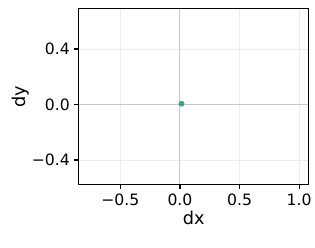}
    \caption{Quantile: median.\newline 100\% success.}
    \label{fig:two-demo-xy-median}
  \end{subfigure}
  \caption{\textbf{Initial action commands and closed-loop outcomes.}
  Points show one training run's 100 episodes per decoder; subcaptions
  report success rates.
  Each point represents one episode's first executed action. Green circles indicate eventual success; red crosses indicate failure.
  All panels use the same axes and, within this experiment, the same
  initial observation. Axis limits also match
  Figure~\ref{fig:two-demo-obstacle-xy} for comparison with the obstacle setting.
  The 100 points in each of (e), (f), and (g) coincide exactly; no
  coordinate jitter is added. Colors label whole episodes, rather than
  the correctness of individual action coordinates.}
  \label{fig:two-demo-closed-loop-xy}
\end{figure}

\textbf{Closed-loop evaluation protocol.}\quad
We use an evaluation budget of 100 complete simulator episodes per
setting. Every episode resets to the same training initial state and
environment seed used to collect the demonstrations, with the same
images, robot state, and instruction at the start.
Random generators and temporal latents are maintained separately for
each episode. The policy predicts and executes ten actions before
observing again. Evaluation stops on environment termination or
truncation, with a maximum of 400 executed actions.
Success means that the environment's success flag becomes true during
the episode. Every recorded success also has a successful final state.
Evaluation performs no weight updates.
Inference batches eight separate environments; the first two predicted
chunks and decoded actions are checked against single-example inference
with the same inputs and noise.
The illustrated run records executed actions, states, predictions,
noise, success flags, and video for checking the plotted trajectories.

\textbf{Action coordinates and interpretation.}\quad
Figure~\ref{fig:two-demo-closed-loop-xy} places each episode's first
executed action from one illustrated training run at its $(d_x,d_y)$ coordinates and colors it by the
final episode outcome. These coordinates are environment action
commands, not measured Cartesian displacements in meters.
We use the first action because its observation is shared across all
episodes; later actions may be conditioned on different states.
The successful and failed Quantile episodes occupy overlapping regions,
so these coordinates do not define a success/failure threshold:
other action dimensions and later sampled actions also affect completion.
As shown in Table~\ref{tab:two-demo-closed-loop}, the three sampled
Quantile decoders achieve \textbf{87--94\% success}, while Flow Matching
achieves \textbf{3\%}. Median, $L_1$, and $L_2$ each achieve \textbf{100\%}
with identical repeated action sequences within each trained model.
Repeated executions at this fixed state are not independent task instances.
Moreover, high task success alone does not establish coverage of both
demonstrated approach routes.

\textbf{Supplemental density-weighted sampling.}\quad
We additionally evaluate the same no-obstacle Quantile models with the
six-motion-coordinate density-weighted decoder defined
in Equations~(\ref{eq:obstacle-density-score})--(\ref{eq:obstacle-density-probability}).
It squares the mean marginal density over six motion coordinates and ten
prediction steps, samples one shared rank, and fixes the gripper at its median.
The initial state, instruction, 400-action limit, ten-action execution
chunks, AR coefficient 0.9, and evaluation budget match the original
shared-rank evaluation. No weights are updated.
The success rate is \textbf{95\%}, compared with \textbf{92\%} for uniform shared ranks and
\textbf{100\%} for median decoding.
Figure~\ref{fig:four-head-dx-distribution} illustrates 100 first-action
points and their episode outcomes from one training run.

\subsection{Two demonstrations with a central obstacle}
\label{sec:two-demo-obstacle}

Appendix~\ref{sec:two-demo-distributions} permits deterministic policies
to complete the task despite the two distinct demonstrated routes.
We add a physical obstruction to examine completion in a setting where
the prescribed direct approach fails.

\begin{figure}[!t]
  \centering
\begingroup
\definecolor{demoLeft}{HTML}{0072B2}
\definecolor{demoRight}{HTML}{D55E00}
\captionsetup[subfigure]{font=footnotesize,justification=raggedright,
  singlelinecheck=false,skip=3pt}
\setlength{\fboxsep}{0pt}
\setlength{\fboxrule}{0.6pt}
\centering

\begin{subfigure}[t]{0.48\linewidth}
  \centering
  \includegraphics[width=\linewidth]{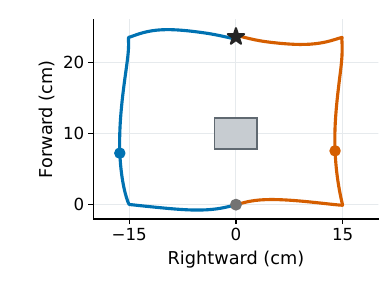}
  \caption{Measured approach paths.}
  \label{fig:two-demo-obstacle-data-paths}
\end{subfigure}\hfill
\begin{subfigure}[t]{0.48\linewidth}
  \centering
  \includegraphics[width=\linewidth]{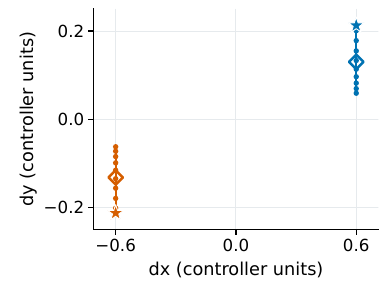}
  \caption{First ten action commands.}
  \label{fig:two-demo-obstacle-data-actions}
\end{subfigure}

\medskip
\begin{subfigure}[t]{0.235\linewidth}
  \centering
  \fcolorbox{demoLeft}{white}{\includegraphics[width=\dimexpr\linewidth-2\fboxrule\relax]{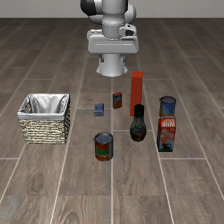}}
  \caption{\textcolor{demoLeft}{Left}: shared start.\newline $t=0$.}
  \label{fig:two-demo-obstacle-data-left-start}
\end{subfigure}\hfill
\begin{subfigure}[t]{0.235\linewidth}
  \centering
  \fcolorbox{demoLeft}{white}{%
  \begin{tikzpicture}
    \node[anchor=south west,inner sep=0] (view) at (0,0)
      {\includegraphics[width=\dimexpr\linewidth-2\fboxrule\relax]{figures/two_demo_obstacle_data/overhead/left.png}};
    \begin{scope}[x={(view.south east)},y={(view.north west)}]
      \clip (0,0) rectangle (1,1);
      \csname arxivOverlay@two_demo_obstacle_data@left\endcsname
    \end{scope}
  \end{tikzpicture}%
}
  \caption{\textcolor{demoLeft}{Left}: top view.\newline $t=40$; $16.32$\,cm left.}
  \label{fig:two-demo-obstacle-data-left-separated}
\end{subfigure}\hfill
\begin{subfigure}[t]{0.235\linewidth}
  \centering
  \fcolorbox{demoLeft}{white}{\includegraphics[width=\dimexpr\linewidth-2\fboxrule\relax]{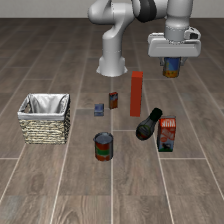}}
  \caption{\textcolor{demoLeft}{Left}: object lifted.\newline $t=180$.}
  \label{fig:two-demo-obstacle-data-left-lifted}
\end{subfigure}\hfill
\begin{subfigure}[t]{0.235\linewidth}
  \centering
  \fcolorbox{demoLeft}{white}{\includegraphics[width=\dimexpr\linewidth-2\fboxrule\relax]{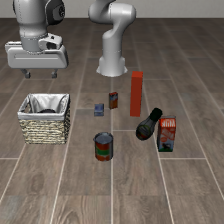}}
  \caption{\textcolor{demoLeft}{Left}: final success.\newline $t=430$.}
  \label{fig:two-demo-obstacle-data-left-final}
\end{subfigure}

\medskip
\begin{subfigure}[t]{0.235\linewidth}
  \centering
  \fcolorbox{demoRight}{white}{\includegraphics[width=\dimexpr\linewidth-2\fboxrule\relax]{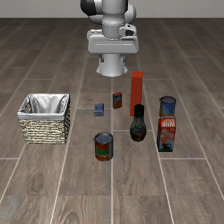}}
  \caption{\textcolor{demoRight}{Right}: shared start.\newline $t=0$.}
  \label{fig:two-demo-obstacle-data-right-start}
\end{subfigure}\hfill
\begin{subfigure}[t]{0.235\linewidth}
  \centering
  \fcolorbox{demoRight}{white}{%
  \begin{tikzpicture}
    \node[anchor=south west,inner sep=0] (view) at (0,0)
      {\includegraphics[width=\dimexpr\linewidth-2\fboxrule\relax]{figures/two_demo_obstacle_data/overhead/right.png}};
    \begin{scope}[x={(view.south east)},y={(view.north west)}]
      \clip (0,0) rectangle (1,1);
      \csname arxivOverlay@two_demo_obstacle_data@right\endcsname
    \end{scope}
  \end{tikzpicture}%
}
  \caption{\textcolor{demoRight}{Right}: top view.\newline $t=40$; $13.92$\,cm right.}
  \label{fig:two-demo-obstacle-data-right-separated}
\end{subfigure}\hfill
\begin{subfigure}[t]{0.235\linewidth}
  \centering
  \fcolorbox{demoRight}{white}{\includegraphics[width=\dimexpr\linewidth-2\fboxrule\relax]{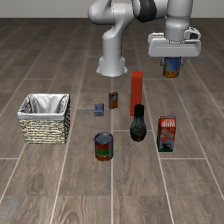}}
  \caption{\textcolor{demoRight}{Right}: object lifted.\newline $t=180$.}
  \label{fig:two-demo-obstacle-data-right-lifted}
\end{subfigure}\hfill
\begin{subfigure}[t]{0.235\linewidth}
  \centering
  \fcolorbox{demoRight}{white}{\includegraphics[width=\dimexpr\linewidth-2\fboxrule\relax]{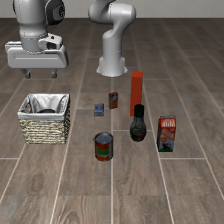}}
  \caption{\textcolor{demoRight}{Right}: final success.\newline $t=430$.}
  \label{fig:two-demo-obstacle-data-right-final}
\end{subfigure}

\caption{\textbf{Two successful training demonstrations around a central obstacle.}
Blue and orange denote left and right detours from the same initial
observation under the shared instruction ``pick up the alphabet soup and
place it in the basket.''
(a) Measured end-effector approach paths for $t=0,\ldots,105$; the shaded
rectangle is the footprint of the $6\times4.4\times19$\,cm obstacle.
The gray circle marks the shared start, the black star the initial target
position, and colored dots the largest simultaneous XY separation
(30.24\,cm at $t=40$).
(b) Raw controller commands: stars mark the first actions and hollow
diamonds the ten-step means.
(d), (h) True overhead renderings of the saved $t=40$ states, with screen
right aligned to the rightward axis in (a). Dashed centerlines and colored
end-effector rings and offset arrows distinguish the two routes; the red
dashed outline marks the obstacle's projected top face, including occluded
edges. All other frames are original training-camera observations.
Here $t$ counts executed actions. Both 430-action demonstrations finish
successfully without recorded obstacle contact and are the only training
trajectories for this experiment.}
\label{fig:two-demo-obstacle-data-description}
\endgroup
\end{figure}

\textbf{Task and demonstration construction.}\quad
We extend the task in Appendix~\ref{sec:two-demo-distributions} with a
fixed, visible, collidable box on the approach to the alphabet soup.
The box measures $6\times4.4\times19$\,cm and is centered 10\,cm ahead
of the initial end-effector position along the tabletop approach direction.
The task instruction remains ``pick up the alphabet soup and place it
in the basket,'' with no left/right cue. Initial-state index 0 and
environment seed 20260922 are retained, and both demonstrations are
collected anew in this modified scene. Their initial external and wrist
images, robot states, and simulator states are identical to each other.
The original and modified scenes have different rendered observations
and full simulator states, so this equality applies within each experiment.

The scripted controller generates one complete left-detour demonstration
and one complete right-detour demonstration. Each contains 430 actions,
uses a nominal 15\,cm lateral route offset, and completes the task without
recorded obstacle contact. In comparison, the demonstrations in
Appendix~\ref{sec:two-demo-distributions} contain 315 actions each and use
11\,cm offsets in a scene without the added box. The approach height is
reduced from 18 to 12\,cm, and transport after grasping also follows a
lateral detour around the box. A separate straight-approach
control fails to complete the obstacle task and records contact on 179
of its 475 executed steps. This failed control is excluded from training.
It verifies that this particular direct approach is obstructed;
the environment retains the original object-in-basket success criterion.
Obstacle contact alone does not terminate an episode or label it a failure.

Figure~\ref{fig:two-demo-obstacle-data-description} documents the two
training samples in the same format as
Figure~\ref{fig:two-demo-data-description}. The measured paths and aligned
overhead views show the two approaches passing on opposite sides of the
central box; both demonstrations then lift the object and complete placement.

\textbf{Training samples and optimization.}\quad
We train all four variants anew on only the two obstacle demonstrations,
giving 860 available pre-action start indices. The observation and action
formats, ten-action training targets, padding masks, balanced batch of 16,
and forced sampling of the shared start follow
Appendix~\ref{sec:two-demo-distributions}. All four heads start from the
same pretrained $\pi_{0.5}$ checkpoint and share the action-side and
16-token prompt initialization and training batch sequence within this
experiment. The VLM remains frozen; the action expert, other action-side
modules, output head, and prompts are trained for 6,000 updates.
Objectives, optimizer settings, learning-rate schedules,
and float32 precision also follow the earlier experiment. The changed
demonstrations produce a new training cache and newly fitted checkpoints.
Cached/full forward and gradient agreement and unchanged frozen weights
are verified for every head.

\textbf{Fixed-observation action distributions.}\quad
Figure~\ref{fig:obstacle-four-head-six-dimensions} extends the
four-head comparison in Figure~\ref{fig:four-head-six-dimensions} to
the obstacle demonstrations. Using the checkpoints from one illustrated training run,
we draw 1,000 action chunks per head at their common training-start
observation and instruction. Each chunk contains ten actions; the
figure shows the first action's six motion coordinates.
Quantile Head draws $z\sim\mathcal{N}(0,1)$ independently for each
sample and uses $\tau=\Phi(z)$ across all six coordinates and all ten
prediction steps, with the gripper fixed at its median. We decode the
exported quantile function with the same piecewise-linear interpolation
and tail extrapolation as in Appendix~\ref{sec:two-demo-distributions}.
Flow Matching uses fresh Gaussian noise and ten integration steps;
$L_1$ and $L_2$ are evaluated repeatedly at the same observation.
The frozen VLM prefix is cached, and no parameters are updated.

Bins and axes are shared across heads within each dimension. Blue and
orange reference lines denote the two demonstrations; gray lines mark
coincident targets. At this first step, the targets differ in $d_x$ and
$d_y$, share $d_z=-0.6$, and have zero rotational commands. Thus, the
vertical approach command differs from the zero $d_z$ target in
Figure~\ref{fig:four-head-six-dimensions}. These plots characterize
conditional output distributions at one observation. Task-completion
accuracy is evaluated separately below.

Quantile samples have peaks near both demonstrated $d_x$ and $d_y$
values, with some mass remaining between the peaks. In these coordinates,
$L_1$ and $L_2$ each produce a single value between the two targets,
while Flow Matching spreads its mass over wider action ranges.

\begin{figure}[!t]
  \centering
  \begingroup
  \setlength{\tabcolsep}{0pt}
  \newcommand{\obstacleDistributionHead}[2]{%
    \raisebox{-0.5\height}{\parbox{0.10\linewidth}{\raggedright\footnotesize\color[HTML]{#1}\bfseries #2}}}
  \begin{tabular}{@{}c@{}cccccc@{}}
    & $\boldsymbol{d_x}$ & $\boldsymbol{d_y}$ & $\boldsymbol{d_z}$ & $\boldsymbol{r_x}$ & $\boldsymbol{r_y}$ & $\boldsymbol{r_z}$ \\
    \obstacleDistributionHead{6554A4}{Quantile\\Head}
      & %
    \raisebox{-0.5\height}{%
      \begin{subfigure}[c]{0.147\linewidth}
        \includegraphics[width=\linewidth]{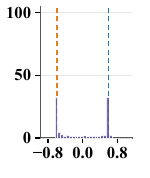}%
      \end{subfigure}}
      & %
    \raisebox{-0.5\height}{%
      \begin{subfigure}[c]{0.147\linewidth}
        \includegraphics[width=\linewidth]{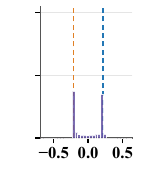}%
      \end{subfigure}}
      & %
    \raisebox{-0.5\height}{%
      \begin{subfigure}[c]{0.147\linewidth}
        \includegraphics[width=\linewidth]{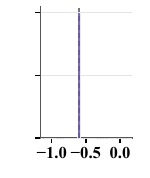}%
      \end{subfigure}}
      & %
    \raisebox{-0.5\height}{%
      \begin{subfigure}[c]{0.147\linewidth}
        \includegraphics[width=\linewidth]{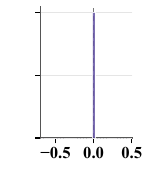}%
      \end{subfigure}}
      & %
    \raisebox{-0.5\height}{%
      \begin{subfigure}[c]{0.147\linewidth}
        \includegraphics[width=\linewidth]{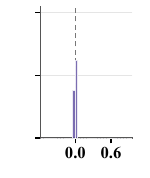}%
      \end{subfigure}}
      & %
    \raisebox{-0.5\height}{%
      \begin{subfigure}[c]{0.147\linewidth}
        \includegraphics[width=\linewidth]{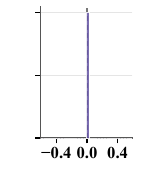}%
      \end{subfigure}} \\[2pt]
    \obstacleDistributionHead{178B89}{Flow\\Matching}
      & %
    \raisebox{-0.5\height}{%
      \begin{subfigure}[c]{0.147\linewidth}
        \includegraphics[width=\linewidth]{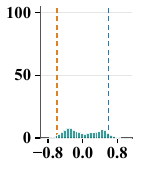}%
      \end{subfigure}}
      & %
    \raisebox{-0.5\height}{%
      \begin{subfigure}[c]{0.147\linewidth}
        \includegraphics[width=\linewidth]{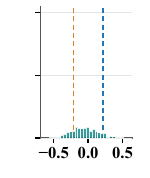}%
      \end{subfigure}}
      & %
    \raisebox{-0.5\height}{%
      \begin{subfigure}[c]{0.147\linewidth}
        \includegraphics[width=\linewidth]{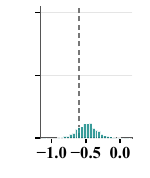}%
      \end{subfigure}}
      & %
    \raisebox{-0.5\height}{%
      \begin{subfigure}[c]{0.147\linewidth}
        \includegraphics[width=\linewidth]{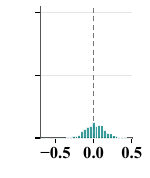}%
      \end{subfigure}}
      & %
    \raisebox{-0.5\height}{%
      \begin{subfigure}[c]{0.147\linewidth}
        \includegraphics[width=\linewidth]{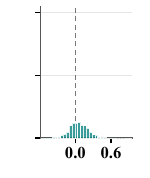}%
      \end{subfigure}}
      & %
    \raisebox{-0.5\height}{%
      \begin{subfigure}[c]{0.147\linewidth}
        \includegraphics[width=\linewidth]{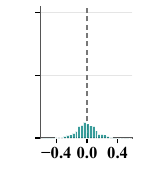}%
      \end{subfigure}} \\[2pt]
    \obstacleDistributionHead{BB8127}{$\boldsymbol{L_1}$}
      & %
    \raisebox{-0.5\height}{%
      \begin{subfigure}[c]{0.147\linewidth}
        \includegraphics[width=\linewidth]{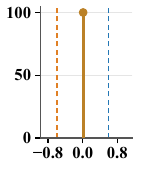}%
      \end{subfigure}}
      & %
    \raisebox{-0.5\height}{%
      \begin{subfigure}[c]{0.147\linewidth}
        \includegraphics[width=\linewidth]{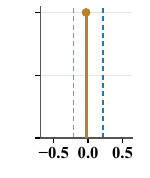}%
      \end{subfigure}}
      & %
    \raisebox{-0.5\height}{%
      \begin{subfigure}[c]{0.147\linewidth}
        \includegraphics[width=\linewidth]{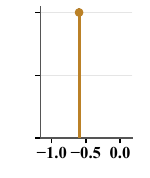}%
      \end{subfigure}}
      & %
    \raisebox{-0.5\height}{%
      \begin{subfigure}[c]{0.147\linewidth}
        \includegraphics[width=\linewidth]{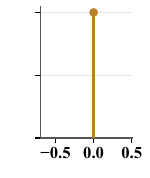}%
      \end{subfigure}}
      & %
    \raisebox{-0.5\height}{%
      \begin{subfigure}[c]{0.147\linewidth}
        \includegraphics[width=\linewidth]{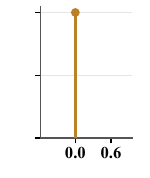}%
      \end{subfigure}}
      & %
    \raisebox{-0.5\height}{%
      \begin{subfigure}[c]{0.147\linewidth}
        \includegraphics[width=\linewidth]{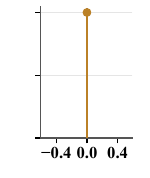}%
      \end{subfigure}} \\[2pt]
    \obstacleDistributionHead{BE5361}{$\boldsymbol{L_2}$}
      & %
    \raisebox{-0.5\height}{%
      \begin{subfigure}[c]{0.147\linewidth}
        \includegraphics[width=\linewidth]{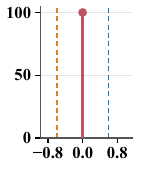}%
      \end{subfigure}}
      & %
    \raisebox{-0.5\height}{%
      \begin{subfigure}[c]{0.147\linewidth}
        \includegraphics[width=\linewidth]{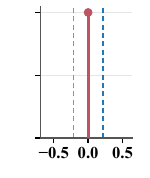}%
      \end{subfigure}}
      & %
    \raisebox{-0.5\height}{%
      \begin{subfigure}[c]{0.147\linewidth}
        \includegraphics[width=\linewidth]{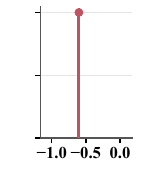}%
      \end{subfigure}}
      & %
    \raisebox{-0.5\height}{%
      \begin{subfigure}[c]{0.147\linewidth}
        \includegraphics[width=\linewidth]{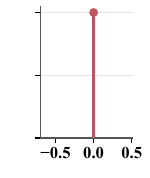}%
      \end{subfigure}}
      & %
    \raisebox{-0.5\height}{%
      \begin{subfigure}[c]{0.147\linewidth}
        \includegraphics[width=\linewidth]{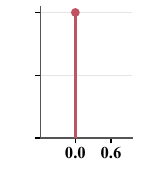}%
      \end{subfigure}}
      & %
    \raisebox{-0.5\height}{%
      \begin{subfigure}[c]{0.147\linewidth}
        \includegraphics[width=\linewidth]{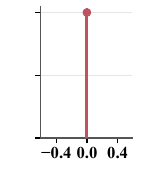}%
      \end{subfigure}}
  \end{tabular}
  \endgroup
  \caption{\textbf{Conditional action distributions with a central obstacle.}
  Rows compare four heads; columns show six motion coordinates of the first predicted action,
  using 1,000 samples per head from one illustrated training run at a shared initial observation.
  Each column shares 32 bins and axes. The vertical axis gives probability mass per bin (\%)
  on a 0--105\% scale. Stems denote exactly constant outputs; narrow distributions remain histograms.
  Blue and orange dashed lines mark the two demonstrations, with gray where references coincide:
  $d_z=-0.6$ and all rotation targets zero.
  Quantile Head uses one shared rank $\tau=\Phi(z)$ across motion coordinates;
  Flow Matching uses ten integration steps.
  These marginal distributions do not establish learned joint dependence or measure closed-loop success.}
  \label{fig:obstacle-four-head-six-dimensions}
\end{figure}

\textbf{Closed-loop sampling and evaluation.}\quad
We compare the same seven decoders as in
Appendix~\ref{sec:two-demo-distributions}. Stochastic Quantile latents use
the AR coefficient 0.9 in \eqref{eq:two-demo-temporal-noise}.
Native ranks use $0.1+0.8\Phi(0.35z_d)$, independent ranks use
$\Phi(z_d)$, and shared ranks use $\Phi(z_1)$ across all six motion
coordinates. The gripper uses its median, and each rank is held over the
ten actions in a chunk. Quantile median, $L_1$, and $L_2$ are deterministic;
Flow Matching uses fresh Gaussian noise and ten Euler integration steps.
Every episode starts in the obstacle training state and predicts and
executes ten actions before replanning. The action budget increases from
400 to 600 to accommodate the longer demonstrations. Each decoder has
an evaluation budget of 100 episodes. In the illustrated run, median actions
and simulator trajectories are identical across repeated executions.
The illustrated trajectories are checked against their recorded success
flags, executed actions, resets, noise, and obstacle contacts.

\textbf{Task success rates.}\quad
Table~\ref{tab:two-demo-obstacle-success} reports task-completion accuracy,
defined as the percentage of completed episodes in which the original
environment success flag becomes true. All recorded successes also have
successful final states. Among these seven matched decoders in the obstacle
setting, shared-rank Quantile sampling attains the highest observed success rate, \textbf{6.0\%},
followed by independent ranks at \textbf{4.0\%} and native ranks
at \textbf{1.0\%}. Flow Matching, $L_1$, $L_2$, and Quantile median each
achieve \textbf{0.0\%}. Thus, the deterministic decoders that achieved
100\% in the earlier setting fail in this modified setting, while
stochastic Quantile decoding retains a small number of successful episodes.
The highest rate within this seven-decoder comparison is only 6\%; it does not establish
reliable obstacle avoidance. Because the scene, demonstrations, fitted
checkpoints, and action budget all change, the cross-experiment difference
is a comparison of two controlled settings rather than an isolated
estimate of the obstacle's effect. Each setting uses one training initial
state; repeated deterministic episodes are not
independent task instances.

\begin{table}[!t]
  \centering
  \caption{\textbf{Task success rates in the original and obstacle settings.}
  Each head is trained for 6,000 updates on its setting's two demonstrations. These rates
  measure completion at a fixed training initial state.}
  \label{tab:two-demo-obstacle-success}
  \small
  \setlength{\tabcolsep}{5pt}
  \begin{tabular}{lrr}
    \toprule
    & \multicolumn{2}{c}{Success rate (\%)} \\
    \cmidrule(lr){2-3}
    Setting & Earlier (\ref{sec:two-demo-distributions}) & Obstacle \\
    \midrule
    Quantile: native ranks & 94.0 & 1.0 \\
    Quantile: independent ranks & 87.0 & 4.0 \\
    Quantile: shared ranks & 92.0 & \textbf{6.0} \\
    Quantile: median & \textbf{100.0} & 0.0 \\
    Flow Matching (10 steps) & 3.0 & 0.0 \\
    $L_1$ & \textbf{100.0} & 0.0 \\
    $L_2$ & \textbf{100.0} & 0.0 \\
    \bottomrule
  \end{tabular}
\end{table}

\begin{figure}[!t]
  \centering
  \captionsetup[subfigure]{justification=raggedright,singlelinecheck=false}
  \begin{subfigure}[t]{0.31\linewidth}
    \centering
    \includegraphics[width=\linewidth]{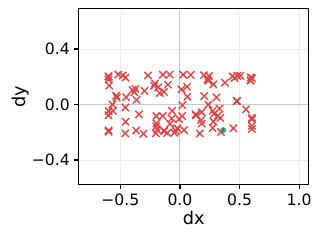}
    \caption{Quantile: native ranks.\newline 1\% success.}
    \label{fig:two-demo-obstacle-xy-native}
  \end{subfigure}\hfill
  \begin{subfigure}[t]{0.31\linewidth}
    \centering
    \includegraphics[width=\linewidth]{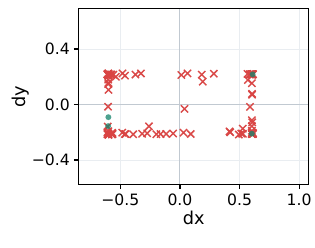}
    \caption{Quantile: independent ranks.\newline 4\% success.}
    \label{fig:two-demo-obstacle-xy-independent}
  \end{subfigure}\hfill
  \begin{subfigure}[t]{0.31\linewidth}
    \centering
    \includegraphics[width=\linewidth]{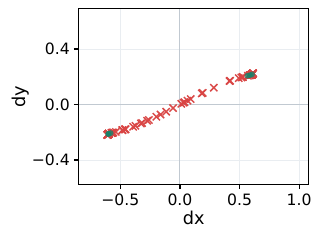}
    \caption{Quantile: shared ranks.\newline 6\% success.}
    \label{fig:two-demo-obstacle-xy-shared}
  \end{subfigure}

  \medskip
  \begin{subfigure}[t]{0.31\linewidth}
    \centering
    \includegraphics[width=\linewidth]{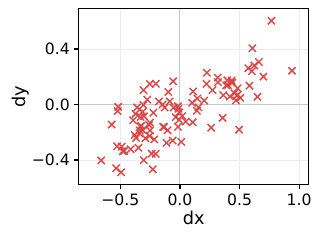}
    \caption{Flow Matching.\newline 0\% success.}
    \label{fig:two-demo-obstacle-xy-fm}
  \end{subfigure}\hfill
  \begin{subfigure}[t]{0.31\linewidth}
    \centering
    \includegraphics[width=\linewidth]{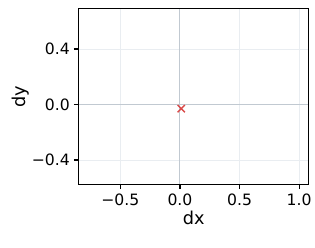}
    \caption{$L_1$.\newline 0\% success.}
    \label{fig:two-demo-obstacle-xy-l1}
  \end{subfigure}\hfill
  \begin{subfigure}[t]{0.31\linewidth}
    \centering
    \includegraphics[width=\linewidth]{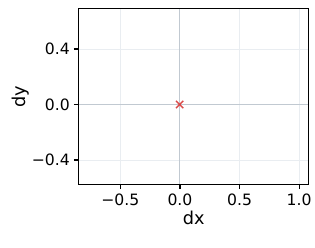}
    \caption{$L_2$.\newline 0\% success.}
    \label{fig:two-demo-obstacle-xy-l2}
  \end{subfigure}

  \medskip
  \begin{subfigure}[t]{0.31\linewidth}
    \centering
    \includegraphics[width=\linewidth]{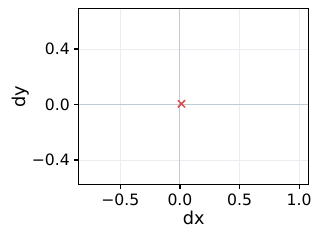}
    \caption{Quantile: median.\newline 0\% success.}
    \label{fig:two-demo-obstacle-xy-median}
  \end{subfigure}
  \caption{\textbf{Initial action commands and outcomes with a central obstacle.}
  Points illustrate one training run; subcaptions report success rates.
  Each point is one episode's first action:
  green circles indicate eventual success and red crosses indicate failure.
  All settings share the same obstacle-task initial observation.
  Axes match Figure~\ref{fig:two-demo-closed-loop-xy} for visual comparison;
  coordinates are environment action commands, not physical displacements.
  The 100 $L_1$ and $L_2$ points and the 100 median points each coincide exactly
  within their respective panels. No coordinate jitter is added.
  Colors describe whole episodes, rather than the correctness of an individual action.}
  \label{fig:two-demo-obstacle-xy}
\end{figure}

\textbf{Action coordinates and interpretation.}\quad
Figure~\ref{fig:two-demo-obstacle-xy} relates the first executed
$(d_x,d_y)$ command to each complete episode's outcome in one illustrated
training run, using the same
axes and colors as Figure~\ref{fig:two-demo-closed-loop-xy} for the earlier
experiment. These are action commands, not measured end-effector positions.
All episodes within each experiment share their initial observation.
The obstacle setting places the deterministic median, $L_1$, and $L_2$
outputs near the origin, with all their episodes failing. Shared-rank
sampling traces a narrow curve; successes occur near its two ends,
alongside failures. Independent ranks span a broader set of coordinate
combinations. This association does not make a first-action coordinate
a sufficient condition for success: other coordinates and later actions
also differ. Inspection of the illustrated run's successful trajectories confirms
that the end-effector first crosses the obstacle-center plane on the left
in three shared-rank episodes and on the right in three. This describes
the end-effector's first crossing of the plane; it does not imply that
every sampled route succeeds or remains free of contact throughout the episode.

\textbf{Effect of the shared sampling interval.}\quad
We next vary only the sampling interval of each trained Quantile
model, retaining the obstacle scene, task instruction, initial state,
and 600-action evaluation budget above. For a lower endpoint
$\ell\in\{0,0.1,0.2,0.3,0.4\}$, the shared motion rank in chunk $k$ is
\begin{equation}
  \tau_k=\ell+(1-2\ell)\Phi(z_{k,1}),
  \qquad \tau_{k,d}=\tau_k\quad(d=1,\ldots,6),
  \label{eq:obstacle-interval-sampling}
\end{equation}
where $z_{k,1}$ follows the Gaussian AR process in
\eqref{eq:two-demo-temporal-noise}. The rank is shared across all ten
actions in each chunk, and the gripper remains at rank $0.5$.
This affine CDF mapping gives a marginally uniform rank within
$[\ell,1-\ell]$, with temporal correlation across chunks; it uses neither
density weighting nor the native decoder's temperature factor $0.35$.
In particular, the shared $[0.1,0.9]$ setting differs from the native
decoder in both its temperature and its cross-coordinate rank sharing.
The sixth setting fixes every coordinate at its median, $\tau=0.5$.

Each setting uses an evaluation budget of 100 simulator episodes. Corresponding
stochastic evaluations use matched underlying Gaussian sequences. The
full-range shared and median results reuse the evaluations in
Table~\ref{tab:two-demo-obstacle-success}. No model is retrained when
changing the sampling interval.

\textbf{Success rates and action distributions.}\quad
For $[0,1]$, $[0.1,0.9]$, $[0.2,0.8]$, $[0.3,0.7]$, $[0.4,0.6]$,
and fixed $0.5$, respectively, success rates are
\textbf{6\%, 4\%, 15\%, 7\%, 3\%, and 0\%}.
The highest success rate is \textbf{15\%} for $[0.2,0.8]$,
9 percentage points above full-range shared sampling. Success does not
increase monotonically as the interval narrows.
Figure~\ref{fig:obstacle-quantile-intervals} plots the first executed
$(d_x,d_y)$ command from each episode in one illustrated training run
and labels it by the full episode's outcome. The narrowest stochastic window reduces coverage of
the demonstrated endpoints; median decoding produces one central
command, with all illustrated episodes failing. Successful episodes coexist
with nearby failed first commands, so these coordinates alone do not
determine task completion. These success rates describe performance
at one training initial state; they do not establish a universally optimal
interval or reliable obstacle avoidance.

\begin{figure}[!tp]
  \centering
  \definecolor{intervalSuccess}{HTML}{087F6C}
  \definecolor{intervalFailure}{HTML}{D84A49}
  \definecolor{intervalExpert}{HTML}{4269A3}
  \captionsetup[subfigure]{font=small,justification=centering,singlelinecheck=false,skip=3pt}
  {\small
  \tikz[baseline=-0.6ex]{\draw[intervalSuccess,line width=1pt] (0,0) circle (0.07);}\; Success
  \hspace{12pt}\tikz[baseline=-0.6ex]{\draw[intervalFailure,line width=1pt] (-.06,-.06)--(.06,.06) (-.06,.06)--(.06,-.06);}\; Failure
  \hspace{12pt}\tikz[baseline=-0.6ex]{\fill[intervalExpert] (0,.085)--(.085,0)--(0,-.085)--(-.085,0)--cycle;}\; Training demonstration}
  \par\vspace{5pt}
  \begin{subfigure}[t]{0.48\linewidth}
    \centering
    \includegraphics[width=\linewidth]{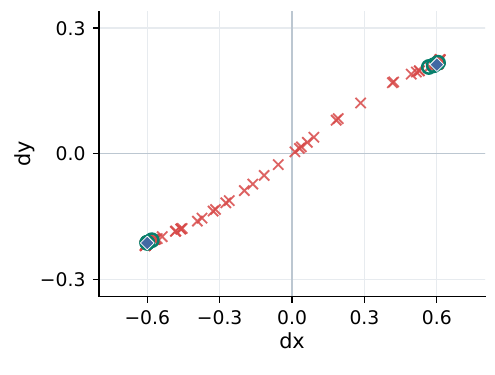}
    \caption{$\tau\in[0,1]$.\newline 6\% success.}
    \label{fig:obstacle-interval-full}
  \end{subfigure}\hfill
  \begin{subfigure}[t]{0.48\linewidth}
    \centering
    \includegraphics[width=\linewidth]{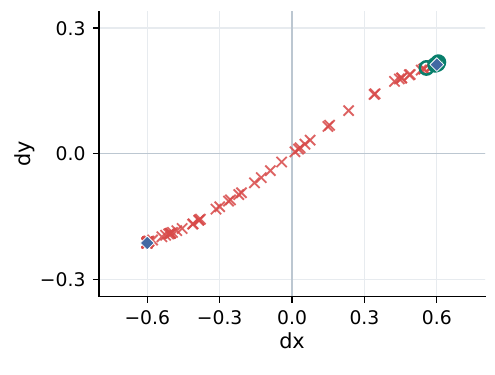}
    \caption{$\tau\in[0.1,0.9]$.\newline 4\% success.}
    \label{fig:obstacle-interval-01-09}
  \end{subfigure}
  \par\medskip
  \begin{subfigure}[t]{0.48\linewidth}
    \centering
    \includegraphics[width=\linewidth]{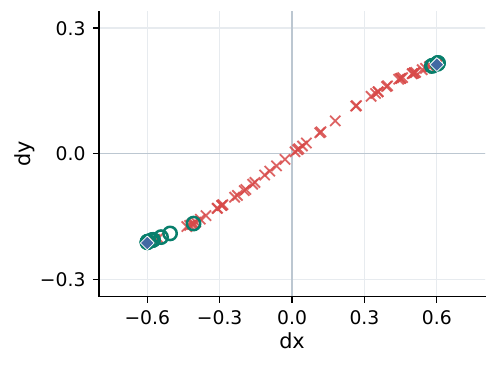}
    \caption{$\tau\in[0.2,0.8]$.\newline 15\% success.}
    \label{fig:obstacle-interval-02-08}
  \end{subfigure}\hfill
  \begin{subfigure}[t]{0.48\linewidth}
    \centering
    \includegraphics[width=\linewidth]{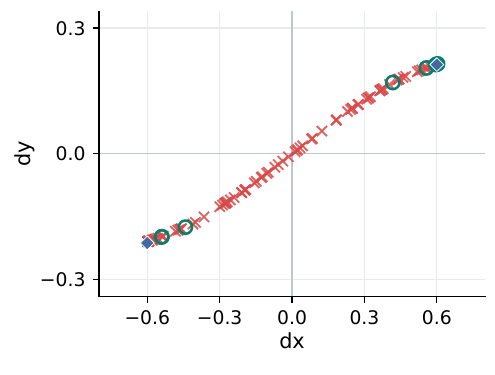}
    \caption{$\tau\in[0.3,0.7]$.\newline 7\% success.}
    \label{fig:obstacle-interval-03-07}
  \end{subfigure}
  \par\medskip
  \begin{subfigure}[t]{0.48\linewidth}
    \centering
    \includegraphics[width=\linewidth]{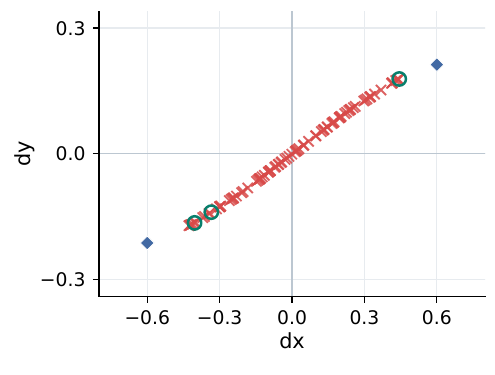}
    \caption{$\tau\in[0.4,0.6]$.\newline 3\% success.}
    \label{fig:obstacle-interval-04-06}
  \end{subfigure}\hfill
  \begin{subfigure}[t]{0.48\linewidth}
    \centering
    \includegraphics[width=\linewidth]{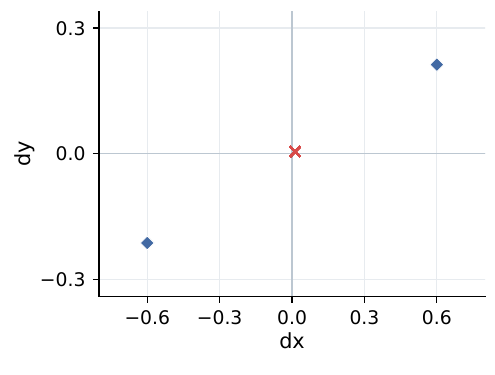}
    \caption{$\tau=0.5$ (median).\newline 0\% success.\newline 100 identical points overlap.}
    \label{fig:obstacle-interval-median}
  \end{subfigure}
  \caption{\textbf{Shared quantile sampling intervals and obstacle-task outcomes.}
  Each panel illustrates 100 episodes from one training run at the same
  initial observation; subcaptions report success rates.
  Each point is the first executed action command:
  green circles indicate eventual episode success, red crosses indicate
  failure, and blue diamonds mark the two training demonstrations' first
  actions. The six motion coordinates share one rank; the gripper uses
  its median. All panels use identical axes in environment action-command
  units, with no jitter or removal of coincident points. The deterministic
  median repeats are not independent task instances.}
  \label{fig:obstacle-quantile-intervals}
\end{figure}

\begin{table}[!t]
  \centering
  \caption{\textbf{Density-based decoding in the obstacle setting.}
  All variants use the same trained Quantile models.}
  \label{tab:obstacle-density-variants}
  \small
  \begin{tabular}{lllr}
    \toprule
    Selection & Scored coordinates & Gripper rank & Success (\%) \\
    \midrule
    Random, squared density & Six motion & Median & 13 \\
    Random, linear density & All seven & Shared sampled rank & 14 \\
    Maximum mean density & All seven & Shared selected rank & 0 \\
    \bottomrule
  \end{tabular}
\end{table}

\begin{table}[!t]
  \centering
  \caption{\textbf{Two-demonstration task success rates (\%).}
  Evaluation uses each setting's fixed training initial state.
  Uniform shared ranks cover $[0,1]$;
  density weighting uses six motion coordinates with exponent two.
  Deterministic repeats are not
  independent task instances.}
  \label{tab:two-demo-main-success}
  \small
  \begin{tabular}{lrr}
    \toprule
    Decoder & No obstacle & Obstacle \\
    \midrule
    Quantile: uniform shared ranks & 92 & 6 \\
    Quantile: density-weighted shared ranks & 95 & \textbf{13} \\
    Quantile: median & \textbf{100} & 0 \\
    Flow Matching (10 integration steps) & 3 & 0 \\
    $L_1$ & \textbf{100} & 0 \\
    $L_2$ & \textbf{100} & 0 \\
    \bottomrule
  \end{tabular}
\end{table}

\textbf{Density-weighted shared-rank sampling.}\quad
\label{sec:obstacle-density-sampling}
We further evaluate a density-weighted decoder using the same obstacle
Quantile models, without retraining. For each model at a new observation,
let $Q_{h,d,j}$ be its predicted quantile at learned level $\tau_j$, in the
model's normalized action coordinates. For each interval between learned
levels, we compute
\begin{equation}
  s_j=\frac{1}{60}\sum_{h=0}^{9}\sum_{d=1}^{6}
  \frac{\tau_{j+1}-\tau_j}
       {\max(Q_{h,d,j+1}-Q_{h,d,j},10^{-6})}.
  \label{eq:obstacle-density-score}
\end{equation}
The 21 learned levels span $[0.025,0.975]$. Adding the two outer
intervals gives 22 intervals covering $[0,1]$; each outer interval
copies the adjacent learned interval's density, consistently with
linear tail extrapolation. For interval $I_j$ of width $\Delta_j$,
the normalized sampling weight is
\begin{equation}
  p_j=\frac{\Delta_j s_j^2}{\sum_k\Delta_k s_k^2}.
  \label{eq:obstacle-density-probability}
\end{equation}
We draw a continuous rank by inverting the resulting piecewise-linear
CDF with $U_k=\Phi(z_{k,1})$, retaining the AR coefficient 0.9 in
\eqref{eq:two-demo-temporal-noise}. The sampled rank is shared across
all six motion coordinates and all ten steps of the chunk; the gripper
is excluded from scoring and decoded at rank 0.5. Scores are recomputed
after each new observation. This is a state-dependent proposal based
on average marginal densities, not an estimate of joint action likelihood.
Soft weighting does not exclude low-density intervals.

The trained models, initial state, instruction, ten-action execution
chunks, 600-action limit, and evaluation budget match the full-range
shared-rank baseline. The success rate is \textbf{13\%}, versus
\textbf{6\%} for uniform shared ranks,
a 7-percentage-point difference at this fixed initial state. The low
absolute success rate does not establish reliable obstacle avoidance.
Applying the same decoding rule to the no-obstacle models gives 95\%
success; see Appendix~\ref{sec:two-demo-distributions} for that
supplemental evaluation, which retains its original 400-action limit.

For completeness, an earlier random decoder averages density over all
seven coordinates, including the gripper, uses linear weights
$p_j\propto\Delta_j s_j$, and applies the selected rank to all seven
coordinates. It achieves \textbf{14\%} success. A separate
deterministic diagnostic selects the midpoint of the learned interval
with the highest seven-coordinate mean density, again using that rank
for all coordinates; its success rate is \textbf{0\%}. These variants
are summarized in Table~\ref{tab:obstacle-density-variants}.
The six- and seven-coordinate random variants differ in the scoring
dimensions, weighting exponent, and gripper decoding, so their
one-percentage-point difference does not isolate any one of these changes.

Table~\ref{tab:two-demo-main-success} summarizes the two settings across
action heads and the three Quantile decoding strategies shown in
Figure~\ref{fig:four-head-dx-distribution}.

\begin{thebibliography}{35}
\providecommand{\natexlab}[1]{#1}
\providecommand{\url}[1]{\texttt{#1}}
\expandafter\ifx\csname urlstyle\endcsname\relax
  \providecommand{\doi}[1]{doi: #1}\else
  \providecommand{\doi}{doi: \begingroup \urlstyle{rm}\Url}\fi

\bibitem[Black et~al.(2024)Black, Brown, Driess, Esmail, Equi, Finn, Fusai,
  Groom, Hausman, Ichter, et~al.]{black2024pi_0}
Kevin Black, Noah Brown, Danny Driess, Adnan Esmail, Michael Equi, Chelsea
  Finn, Niccolo Fusai, Lachy Groom, Karol Hausman, Brian Ichter, et~al.
\newblock $\pi_0$: A vision-language-action flow model for general robot
  control.
\newblock \emph{arXiv preprint arXiv:2410.24164}, 2024.

\bibitem[Black et~al.(2025)Black, Brown, Darpinian, et~al.]{pmlr-v305-black25a}
Kevin Black, Noah Brown, James Darpinian, et~al.
\newblock $\pi_{0.5}$: a vision-language-action model with open-world
  generalization.
\newblock In Joseph Lim, Shuran Song, and Hae-Won Park (eds.),
  \emph{Proceedings of The 9th Conference on Robot Learning}, volume 305 of
  \emph{Proceedings of Machine Learning Research}, pp.\  17--40. PMLR, 27--30
  Sep 2025.

\bibitem[Brohan et~al.(2023)Brohan, Brown, Carbajal, Chebotar, Chen,
  Choromanski, Ding, Driess, Dubey, Finn, et~al.]{brohan2023rt}
Anthony Brohan, Noah Brown, Justice Carbajal, Yevgen Chebotar, Xi~Chen,
  Krzysztof Choromanski, Tianli Ding, Danny Driess, Avinava Dubey, Chelsea
  Finn, et~al.
\newblock Rt-2: Vision-language-action models transfer web knowledge to robotic
  control.
\newblock \emph{arXiv preprint arXiv:2307.15818}, 2023.

\bibitem[Chen et~al.(2026{\natexlab{a}})Chen, Lyu, and Beksi]{chen2026reconvla}
Lingling Chen, Zongyao Lyu, and William~J. Beksi.
\newblock {ReconVLA}: An uncertainty-guided and failure-aware
  vision-language-action framework for robotic control.
\newblock \emph{arXiv preprint arXiv:2604.16677}, 2026{\natexlab{a}}.

\bibitem[Chen et~al.(2026{\natexlab{b}})Chen, Zhang, Gong, and
  Qiu]{chen2026simple}
Yitong Chen, Shiduo Zhang, Jingjing Gong, and Xipeng Qiu.
\newblock Let it be simple: One-step action generation for
  vision-language-action models.
\newblock \emph{arXiv preprint arXiv:2606.05737}, 2026{\natexlab{b}}.

\bibitem[Chi et~al.(2025)Chi, Xu, Feng, Cousineau, Du, Burchfiel, Tedrake, and
  Song]{chi2025diffusion}
Cheng Chi, Zhenjia Xu, Siyuan Feng, Eric Cousineau, Yilun Du, Benjamin
  Burchfiel, Russ Tedrake, and Shuran Song.
\newblock Diffusion policy: Visuomotor policy learning via action diffusion.
\newblock \emph{The International Journal of Robotics Research}, 44\penalty0
  (10-11):\penalty0 1684--1704, 2025.

\bibitem[Fei et~al.(2026)Fei, Wang, Shi, Dai, Cai, Qian, Ji, He, Zhang, Fei,
  et~al.]{fei2026libero}
Senyu Fei, Siyin Wang, Junhao Shi, Zihao Dai, Jikun Cai, Pengfang Qian, Li~Ji,
  Xinzhe He, Shiduo Zhang, Zhaoye Fei, et~al.
\newblock Libero-plus: A progressive robustness benchmark for
  visual-language-action models.
\newblock In \emph{Proceedings of the IEEE/CVF Conference on Computer Vision
  and Pattern Recognition}, pp.\  38574--38583, 2026.

\bibitem[Gneiting \& Raftery(2007)Gneiting and Raftery]{gneiting2007strictly}
Tilmann Gneiting and Adrian~E Raftery.
\newblock Strictly proper scoring rules, prediction, and estimation.
\newblock \emph{Journal of the American statistical Association}, 102\penalty0
  (477):\penalty0 359--378, 2007.

\bibitem[Han et~al.(2026)Han, Jeon, Jung, Zurbr{\"u}gg, An, Portela, Hutter,
  Pollefeys, Kim, and Hong]{han2026geometric}
Jisang Han, Seonghu Jeon, Jaewoo Jung, Ren{\'e} Zurbr{\"u}gg, Honggyu An,
  Tifanny Portela, Marco Hutter, Marc Pollefeys, Seungryong Kim, and Sunghwan
  Hong.
\newblock Geometric action model for robot policy learning.
\newblock \emph{arXiv preprint arXiv:2606.17046}, 2026.

\bibitem[Islam et~al.(2026)Islam, Peddapalli, Lee, and Ahn]{islam2026sureflow}
Md~Tanvir Islam, Sai~Navaneet Peddapalli, Sangmoon Lee, and Sangtae Ahn.
\newblock {SUREFlow}: State-space uncertainty-aware {RE}sidual flow matching
  for robust robot manipulation.
\newblock \emph{arXiv preprint arXiv:2607.10504}, 2026.

\bibitem[Kim et~al.(2024)Kim, Pertsch, Karamcheti, Xiao, Balakrishna, Nair,
  Rafailov, Foster, Sanketi, Vuong, Kollar, Burchfiel, Tedrake, Sadigh, Levine,
  Liang, and Finn]{kim2024openvla}
Moo~Jin Kim, Karl Pertsch, Siddharth Karamcheti, Ted Xiao, Ashwin Balakrishna,
  Suraj Nair, Rafael Rafailov, Ethan~P Foster, Pannag~R Sanketi, Quan Vuong,
  Thomas Kollar, Benjamin Burchfiel, Russ Tedrake, Dorsa Sadigh, Sergey Levine,
  Percy Liang, and Chelsea Finn.
\newblock Open{VLA}: An open-source vision-language-action model.
\newblock In \emph{8th Annual Conference on Robot Learning}, 2024.

\bibitem[Kim et~al.(2025)Kim, Finn, and Liang]{kim2025fine}
Moo~Jin Kim, Chelsea Finn, and Percy Liang.
\newblock Fine-tuning vision-language-action models: Optimizing speed and
  success.
\newblock \emph{arXiv preprint arXiv:2502.19645}, 2025.

\bibitem[Koenker \& Bassett~Jr(1978)Koenker and
  Bassett~Jr]{koenker1978regression}
Roger Koenker and Gilbert Bassett~Jr.
\newblock Regression quantiles.
\newblock \emph{Econometrica: journal of the Econometric Society}, pp.\
  33--50, 1978.

\bibitem[Lee et~al.(2025)Lee, Duan, Fang, Deng, Liu, Li, Fang, Zhang, Wang,
  Lee, et~al.]{lee2025molmoact}
Jason Lee, Jiafei Duan, Haoquan Fang, Yuquan Deng, Shuo Liu, Boyang Li, Bohan
  Fang, Jieyu Zhang, Yi~Ru Wang, Sangho Lee, et~al.
\newblock Molmoact: Action reasoning models that can reason in space.
\newblock \emph{arXiv preprint arXiv:2508.07917}, 2025.

\bibitem[Lipman et~al.(2023)Lipman, Chen, Ben-Hamu, Nickel, and
  Le]{lipman2023flow}
Yaron Lipman, Ricky T.~Q. Chen, Heli Ben-Hamu, Maximilian Nickel, and Matthew
  Le.
\newblock Flow matching for generative modeling.
\newblock In \emph{The Eleventh International Conference on Learning
  Representations}, 2023.

\bibitem[Liu et~al.(2023)Liu, Zhu, Gao, Feng, Liu, Zhu, and
  Stone]{liu2023libero}
Bo~Liu, Yifeng Zhu, Chongkai Gao, Yihao Feng, Qiang Liu, Yuke Zhu, and Peter
  Stone.
\newblock Libero: Benchmarking knowledge transfer for lifelong robot learning.
\newblock \emph{Advances in Neural Information Processing Systems},
  36:\penalty0 44776--44791, 2023.

\bibitem[Liu et~al.(2026)Liu, Zhao, Yang, Chen, Zhou, Chen, Ren, Peng, Jin,
  Wang, et~al.]{liu2026ge}
Renhang Liu, Wenzhi Zhao, Zhuo Yang, Liliang Chen, Pengfei Zhou, Shengcong
  Chen, Guanghui Ren, Youlun Peng, Rongjun Jin, Nan Wang, et~al.
\newblock Ge-act 2.0: Pretraining and scaling a world-action model for robotic
  manipulation.
\newblock \emph{arXiv preprint arXiv:2609.05588}, 2026.

\bibitem[Luan et~al.(2026)Luan, Li, Zhao, Zhang, Wu, and Ma]{luan2026snapflow}
Wuyang Luan, Junhui Li, Weiguang Zhao, Wenjian Zhang, Tieru Wu, and Rui Ma.
\newblock Snapflow: One-step action generation for flow-matching vlas via
  progressive self-distillation.
\newblock \emph{arXiv preprint arXiv:2604.05656}, 2026.

\bibitem[Luo et~al.(2026)Luo, Chen, Liang, Wang, and Li]{luo2026simvla}
Yuankai Luo, Woping Chen, Tong Liang, Baiqiao Wang, and Zhenguo Li.
\newblock Simvla: A simple vla baseline for robotic manipulation.
\newblock \emph{arXiv preprint arXiv:2602.18224}, 2026.

\bibitem[Ma et~al.(2026)Ma, Cai, Xu, Li, Yang, Tian, Cao, Zhu, Qiu, Yang,
  et~al.]{ma2026internvla}
Haoxiang Ma, Junhao Cai, Xiaoxu Xu, Hao Li, Yuyin Yang, Yang Tian, Jiafei Cao,
  Hongrui Zhu, Zherui Qiu, Yuqiang Yang, et~al.
\newblock {InternVLA-A1.5}: Unifying understanding, latent foresight, and
  action for compositional generalization.
\newblock \emph{arXiv preprint arXiv:2607.04988}, 2026.

\bibitem[Narayan et~al.(2024)Narayan, Wang, Canini, and
  Gupta]{narayan2024expected}
Taman Narayan, Serena~Lutong Wang, Kevin~Robert Canini, and Maya Gupta.
\newblock Expected pinball loss for quantile regression and inverse cdf
  estimation.
\newblock \emph{Transactions on Machine Learning Research}, 2024.

\bibitem[Pertsch et~al.(2025)Pertsch, Stachowicz, Ichter, Driess, Nair, Vuong,
  Mees, Finn, and Levine]{pertsch2025fast}
Karl Pertsch, Kyle Stachowicz, Brian Ichter, Danny Driess, Suraj Nair, Quan
  Vuong, Oier Mees, Chelsea Finn, and Sergey Levine.
\newblock Fast: Efficient action tokenization for vision-language-action
  models.
\newblock \emph{arXiv preprint arXiv:2501.09747}, 2025.

\bibitem[Richter \& Wattenhofer(2019)Richter and
  Wattenhofer]{richter2019learning}
Oliver Richter and Roger Wattenhofer.
\newblock Learning policies through quantile regression.
\newblock \emph{arXiv preprint arXiv:1906.11941}, 2019.

\bibitem[Shi et~al.(2026{\natexlab{a}})Shi, Li, Xie, Wang, Zhou, Wang, Zhang,
  Luo, and Huang]{shi2026memoryvla++}
Hao Shi, Weiye Li, Bin Xie, Yulin Wang, Renping Zhou, Tiancai Wang, Xiangyu
  Zhang, Ping Luo, and Gao Huang.
\newblock {MemoryVLA++}: Temporal modeling via memory and imagination in
  vision-language-action models.
\newblock \emph{arXiv preprint arXiv:2606.09827}, 2026{\natexlab{a}}.

\bibitem[Shi et~al.(2026{\natexlab{b}})Shi, Xie, Liu, Sun, Liu, Wang, Zhou,
  Fan, Zhang, and Huang]{shi2026memoryvla}
Hao Shi, Bin Xie, Yingfei Liu, Lin Sun, Fengrong Liu, Tiancai Wang, Erjin Zhou,
  Haoqiang Fan, Xiangyu Zhang, and Gao Huang.
\newblock Memoryvla: Perceptual-cognitive memory in vision-language-action
  models for robotic manipulation.
\newblock In \emph{International Conference on Learning Representations},
  volume 2026, pp.\  18567--18602, 2026{\natexlab{b}}.

\bibitem[Wang et~al.(2026)Wang, Ding, Li, Cui, Ge, Tong, Song, Zhao, Zhao, Hou,
  et~al.]{wang2026vla}
Yihao Wang, Pengxiang Ding, Lingxiao Li, Can Cui, Zirui Ge, Xinyang Tong,
  Wenxuan Song, Han Zhao, Wei Zhao, Pengxu Hou, et~al.
\newblock Vla-adapter: An effective paradigm for tiny-scale
  vision-language-action model.
\newblock In \emph{Proceedings of the AAAI conference on artificial
  intelligence}, volume~40, pp.\  18638--18646, 2026.

\bibitem[Wu et~al.(2026{\natexlab{a}})Wu, Fan, Liao, Jiang, Yang, Luo, Wu,
  Zheng, and Loy]{wu2026vlanext}
Xiao-Ming Wu, Bin Fan, Kang Liao, Jian-Jian Jiang, Runze Yang, Yihang Luo,
  Zhonghua Wu, Wei-Shi Zheng, and Chen~Change Loy.
\newblock Vlanext: Recipes for building strong vla models.
\newblock \emph{arXiv preprint arXiv:2602.18532}, 2026{\natexlab{a}}.

\bibitem[Wu et~al.(2026{\natexlab{b}})Wu, Matsushima, and
  Ota]{wu2026continuous}
Yueh-Hua Wu, Tatsuya Matsushima, and Kei Ota.
\newblock Continuous reasoning for vision-language-action.
\newblock \emph{arXiv preprint arXiv:2606.00229}, 2026{\natexlab{b}}.

\bibitem[Yang et~al.(2026)Yang, Liu, Kou, Chen, Hu, Zhou, Zhao, Wei, Xia, Li,
  et~al.]{yang2026world}
Yi~Yang, Zhihong Liu, Siqi Kou, Yiyang Chen, Yanzhe Hu, Jianbo Zhou, Boyuan
  Zhao, Zhijie Wei, Xiao Xia, Xueqi Li, et~al.
\newblock World-language-action model for unified world modeling, language
  reasoning, and action synthesis.
\newblock \emph{arXiv preprint arXiv:2606.05979}, 2026.

\bibitem[Yu et~al.(2026)Yu, Tao, Leimgruber, Esterl, Stiasny, Bunn, Wen, Guo,
  and Cremer]{yu2026orderfusion}
Runyao Yu, Yuchen Tao, Fabian Leimgruber, Tara Esterl, Jochen Stiasny, Derek~W
  Bunn, Qingsong Wen, Hongye Guo, and Jochen~L Cremer.
\newblock Orderfusion: Encoding orderbook for end-to-end probabilistic intraday
  electricity price forecasting.
\newblock \emph{Advanced Engineering Informatics}, 76:\penalty0 105131, 2026.

\bibitem[Zhang et~al.(2026)Zhang, Wu, Duan, and Han]{zhang2026cofactvla}
Yan Zhang, Yinan Wu, Haoran Duan, and Jungong Han.
\newblock {CofactVLA}: Deconfounding vision-language-action models via
  counterfactual intervention.
\newblock \emph{arXiv preprint arXiv:2608.04396}, 2026.

\bibitem[Zheng et~al.(2026)Zheng, Li, Wang, Liu, Kang, Feng, Zheng, Zou, Chen,
  Zeng, et~al.]{zheng2026x}
Jinliang Zheng, Jianxiong Li, Zhihao Wang, Dongxiu Liu, Xirui Kang, Yuchun
  Feng, Yinan Zheng, Jiayin Zou, Yilun Chen, Jia Zeng, et~al.
\newblock X-vla: Soft-prompted transformer as scalable cross-embodiment
  vision-language-action model.
\newblock In \emph{International Conference on Learning Representations},
  volume 2026, pp.\  60580--60606, 2026.

\bibitem[Zhong et~al.(2026)Zhong, Liu, Wei, Xiong, Liu, and Ren]{zhong2026acot}
Linqing Zhong, Yi~Liu, Yifei Wei, Ziyu Xiong, Si~Liu, and Guanghui Ren.
\newblock Acot-vla: Action chain-of-thought for vision-language-action models.
\newblock In \emph{Proceedings of the IEEE/CVF Conference on Computer Vision
  and Pattern Recognition}, pp.\  8152--8162, 2026.

\bibitem[Zhou et~al.(2025)Zhou, Xu, Tie, Chen, Zhang, Chu, Zhou, and
  Sun]{zhou2025libero}
Xueyang Zhou, Yangming Xu, Guiyao Tie, Yongchao Chen, Guowen Zhang, Duanfeng
  Chu, Pan Zhou, and Lichao Sun.
\newblock Libero-pro: Towards robust and fair evaluation of
  vision-language-action models beyond memorization.
\newblock \emph{arXiv preprint arXiv:2510.03827}, 2025.

\bibitem[Zou \& Yuan(2008)Zou and Yuan]{zou2008composite}
Hui Zou and Ming Yuan.
\newblock Composite quantile regression and the oracle model selection theory.
\newblock \emph{The Annals of Statistics}, 36\penalty0 (3):\penalty0
  1108--1126, 2008.

\end{thebibliography}
\end{document}